\documentclass{article}
\usepackage{float}
\usepackage{graphicx}
\usepackage{amssymb, amsmath, amsthm,mathtools}
\usepackage{epstopdf}

\usepackage{bbm}
\usepackage[margin=1in]{geometry}
\usepackage{layout, color}
\usepackage{enumerate}
\usepackage{enumitem}
\usepackage{authblk}
\usepackage{algorithm}
\usepackage{algpseudocode}
\usepackage{setspace}
\usepackage{hyperref}
\usepackage[compress]{natbib}

\usepackage{url}
\usepackage{subfig}
\usepackage{mdwlist}
\usepackage{multirow}
\usepackage{amsthm}
\usepackage{color}
\usepackage[dvipsnames]{xcolor}

\usepackage[title]{appendix}
\usepackage{wrapfig}

\newcommand{\R}{\mathbb{R}}
\newcommand{\PX}{\mathcal{P}_{\mathbf{X}}}
\newcommand{\PZ}{\mathcal{P}_{\mathbf{Z}}}
\DeclareMathOperator{\diam}{diam}

\newtheorem{theorem}{Theorem}[section]

\newtheorem{assump}{Assumption}
\newtheorem{corollary}[theorem]{Corollary}
\newtheorem{lemma}[theorem]{Lemma}
\newtheorem{definition}[theorem]{Definition}

\theoremstyle{definition}

\newtheorem{conjecture*}{Conjecture}
\theoremstyle{plain}

\newenvironment{keywords}{%
    \small\textbf{Keywords:}\space
}{%
    \par
}

\title{Recovering the Euclidean Geometry of High-Dimensional Data Under Heterogeneous Noise}
\author{ Keyi Li${^{1,*}}$~~~~Yuval Kluger${^{1,2}}$~~~~Boris Landa${^{1,3,*}}$\\
\small{${^1}$Department of Applied \& Computational Mathematics, Yale University}\\
\small{${^2}$Department of Pathology, Yale University}\\
\small{${^3}$Department of Electrical and Computer Engineering, Yale University}\\
\small{${^*}$Corresponding author. Email: keyi.y.li@yale.edu, boris.landa@yale.edu}
}

\date{}
\begin{document}
\maketitle

\begin{abstract}
    Pairwise Euclidean distances provide a simple geometric representation of high-dimensional data and underlie many data analysis methods, including clustering, visualization, and manifold learning. However, in real-world applications, these distances can be severely corrupted by heterogeneous noise, where different data points are subject to different, and potentially large, noise levels. Such noise distorts the distances inhomogeneously, leading to a misrepresentation of the underlying data geometry. In this work, we ask whether the underlying clean Euclidean geometry can nevertheless be recovered from the corrupted pairwise distances. We answer this question affirmatively, showing that, perhaps surprisingly, such recovery is possible across a broad class of high-dimensional settings with underlying low-dimensional geometry, even when the noise magnitudes vary substantially across observations. We develop a principled approach for recovering the clean pairwise distances, without prior knowledge of either the clean data structure or the noise distribution, and establish probabilistic recovery guarantees. In particular, the recovery error converges to zero at polynomial rates as the feature dimension and dataset size increase simultaneously. Experiments on synthetic datasets demonstrate accurate recovery of the Euclidean geometry in challenging regimes and improved robustness of downstream methods that rely on pairwise distances. When applied to real single-cell RNA sequencing data, our method reliably recovers local neighborhood structure, enabling accurate nearest-neighbor identification used in many downstream analyses.
\end{abstract}

\begin{keywords}
Euclidean distance, geometry, heterogeneous noise, inhomogeneous corruption, kernel
\end{keywords}

\section{Introduction}

Pairwise Euclidean distances provide a simple and widely used representation of the geometry of high-dimensional data. They play a central role in many data analysis methods, including dimensionality reduction and manifold learning~\cite{MDS,tsne,laplacian_eigenmap,diffusion_map}, community detection and clustering~\cite{spectral_clustering,spectral_clustering3,spectral_clustering4,community_detection}, image denoising~\cite{image_denoising,image_denoising2,image_denoising4}, and signal processing over graph domains~\cite{graph_learning,graph_learning2,graph_learning3}.
In many of these methods, the pairwise distances are used to construct a graph representation of the data, where vertices correspond to observations and edge weights encode their geometric similarity. The simplest example is the $k$-nearest neighbor (KNN) graph, in which each data point is connected to its $k$ closest neighbors. A commonly used weighted alternative is the Gaussian affinity matrix, whose entries decay with the pairwise Euclidean distances and therefore provide a more graded description of similarity. In either case, the usefulness of the resulting representation depends critically on whether the observed distances faithfully reflect the geometry of the underlying clean data.

\subsection{The influence of noise}\label{sec:influence_noise}

Modern real-world datasets are often both high-dimensional and noisy. Such noise can arise from diverse sources, including measurement errors~\cite{remote_sensing},
inherent stochasticity~\cite{quantum_measurement},
and limitations of data acquisition techniques~\cite{scRNAseq_intro}.
Moreover, the strength and distribution of the noise may vary substantially across observations.

To understand how such noise affects Euclidean geometry, we consider the additive model\begin{equation}
    \label{eq:noise_model}
    \mathbf{y}_i=\mathbf{x}_i+\boldsymbol{\eta}_i,
\end{equation}
where $\mathbf{Y}=\{\mathbf{y}_i\}_{i=1}^n\subset\mathbb{R}^d$ denotes the observed corrupted dataset, $\mathbf{X}=\{\mathbf{x}_i\}_{i=1}^n\subset\mathbb{R}^d$ is the underlying clean dataset, and $\{\boldsymbol{\eta}_i\}_{i=1}^n\subset\mathbb{R}^d$ are independent zero-mean noise vectors. The observed squared Euclidean distance can be decomposed as
\begin{equation}
    \label{eq:squared_distance_cost}
    \begin{aligned}
        \|\mathbf{y}_i-\mathbf{y}_j\|_2^2
        &=\|(\mathbf{x}_i-\mathbf{x}_j)
            +(\boldsymbol{\eta}_i-\boldsymbol{\eta}_j)\|_2^2\\
        &=\|\mathbf{x}_i-\mathbf{x}_j\|_2^2+r_i+r_j+\epsilon_{ij},
    \end{aligned}
\end{equation}
where $r_i=\|\boldsymbol{\eta}_i\|_2^2$ is the squared magnitude of the noise affecting observation $i$, and $\epsilon_{ij}
    =
    2\langle\mathbf{x}_i-\mathbf{x}_j,
    \boldsymbol{\eta}_i-\boldsymbol{\eta}_j\rangle
    -2\langle\boldsymbol{\eta}_i,\boldsymbol{\eta}_j\rangle$.
Note that $\mathbb{E}[\epsilon_{ij}] = 0$ for any $i \neq j$, while $\mathbb{E}[r_i] \geq 0$ and can be large. Thus, noise influences the pairwise distances through two mechanisms: an additive bias term $r_i+r_j$ determined by the noise magnitudes of the observations, and a fluctuation term $\epsilon_{ij}$ arising from interactions between the signal and noise and between two noise vectors.

Several works analyzed Euclidean distances and Gaussian kernels under high-dimensional homogeneous noise, where the noise magnitudes $\{r_i\}_{i=1}^n$ concentrate around a common value $c$ and the noise is sufficiently incoherent or delocalized across coordinates~\cite{homoskedastic_theory1,homoskedastic_theory}. It was shown that in such homogeneous high-dimensional settings, the observed distances satisfy $\|\mathbf{y}_i-\mathbf{y}_j\|_2^2
    \approx
    \|\mathbf{x}_i-\mathbf{x}_j\|_2^2+2c$ for $i\neq j$. That is, all off-diagonal squared distances are biased by approximately the same additive constant. For Gaussian affinity matrices, this translates into a common multiplicative scaling of the off-diagonal entries, which can be easily removed through standard normalizations such as row-stochastic or symmetric normalization~\cite{homoskedastic_theory}.

The situation is considerably more challenging when the noise magnitudes vary across observations. In this heterogeneous regime, the terms $r_i+r_j$ in~\eqref{eq:squared_distance_cost} differ from pair to pair and may be comparable to, or even larger than, the corresponding clean squared distances. The observed distance between two points may therefore be determined more by their individual noise levels than by their true geometric separation. This makes the observed distances unreliable as direct measures of the underlying geometry: a small observed distance need not indicate genuine geometric proximity, it may simply result from either one or both observations having relatively small noise magnitudes. Thus, the effect of heterogeneous noise can also produce the undesirable phenomenon of \textit{hubness}~\cite{hubness}, where some points in high-dimensional space are repeatedly selected as near neighbors of many other points.

Such heterogeneous corruption arises naturally in many real datasets, including single-cell RNA sequencing~\cite{scRNAseq_intro}, photon imaging~\cite{photon_imaging}, network traffic analysis~\cite{network}, and astronomical data~\cite{astro}. Heterogeneity can also arise when datasets collected under different experimental conditions, at different times, or from different sources are combined for joint analysis~\cite{human_atlas}. These examples motivate the question of whether the underlying geometry can be recovered without knowing in advance how strongly each observation has been corrupted.

To illustrate the geometric consequences of heterogeneous noise, we consider a simple example in which clean data are sampled from a circle embedded in an ambient space of dimension $d=10^4$ and then corrupted by a mixture of two noise patterns. Most points are affected by noise whose magnitude varies smoothly and substantially along the circle, while the remaining points are corrupted by outlier-type noise of considerably larger magnitude. The latter noise pattern naturally gives rise to outliers: by \eqref{eq:squared_distance_cost}, a large noise magnitude significantly inflates all distances involving that point, making it appear far from every other point. See Appendix~\ref{sec:intro_ex} for details and Figure~\ref{fig:intro_ex}(a) for a two-dimensional example.

As shown in Figure~\ref{fig:intro_ex}(d), the corrupted squared distances are widely dispersed around their clean counterparts. In particular, small corrupted distances no longer reliably identify nearby points in the clean geometry. To illustrate the consequences for a downstream geometric analysis, we apply t-SNE~\cite{tsne}, which first transforms pairwise distances into affinities $\{P_{ij}\}$ describing the neighborhood of each observation and then seeks a low-dimensional embedding that preserves these affinities. To visualize the effect of the noise, we order the points according to their angular coordinate on the clean circle. Under this ordering, affinities $P_{ij}$ constructed from the clean distances should be large near the diagonal $i=j$, with corner blocks reflecting the periodicity of the circle. By contrast, the affinities obtained from the corrupted distances place substantial mass far from the diagonal, assigning large affinity to pairs that are distant in the clean geometry; see Figure~\ref{fig:intro_ex}(e). The same figure also reveals the hubness phenomenon: a small subset of mildly corrupted points receives large affinity from many other observations, appearing as rows and columns with spread out affinity values. Since t-SNE attempts to preserve these distorted neighborhood relationships, its embedding of the corrupted data fails to recover the underlying circle and instead produces several disconnected blobs (Figure~\ref{fig:intro_ex}(b)).

\begin{figure}[!ht]
    \centering
    \includegraphics[width=\textwidth]{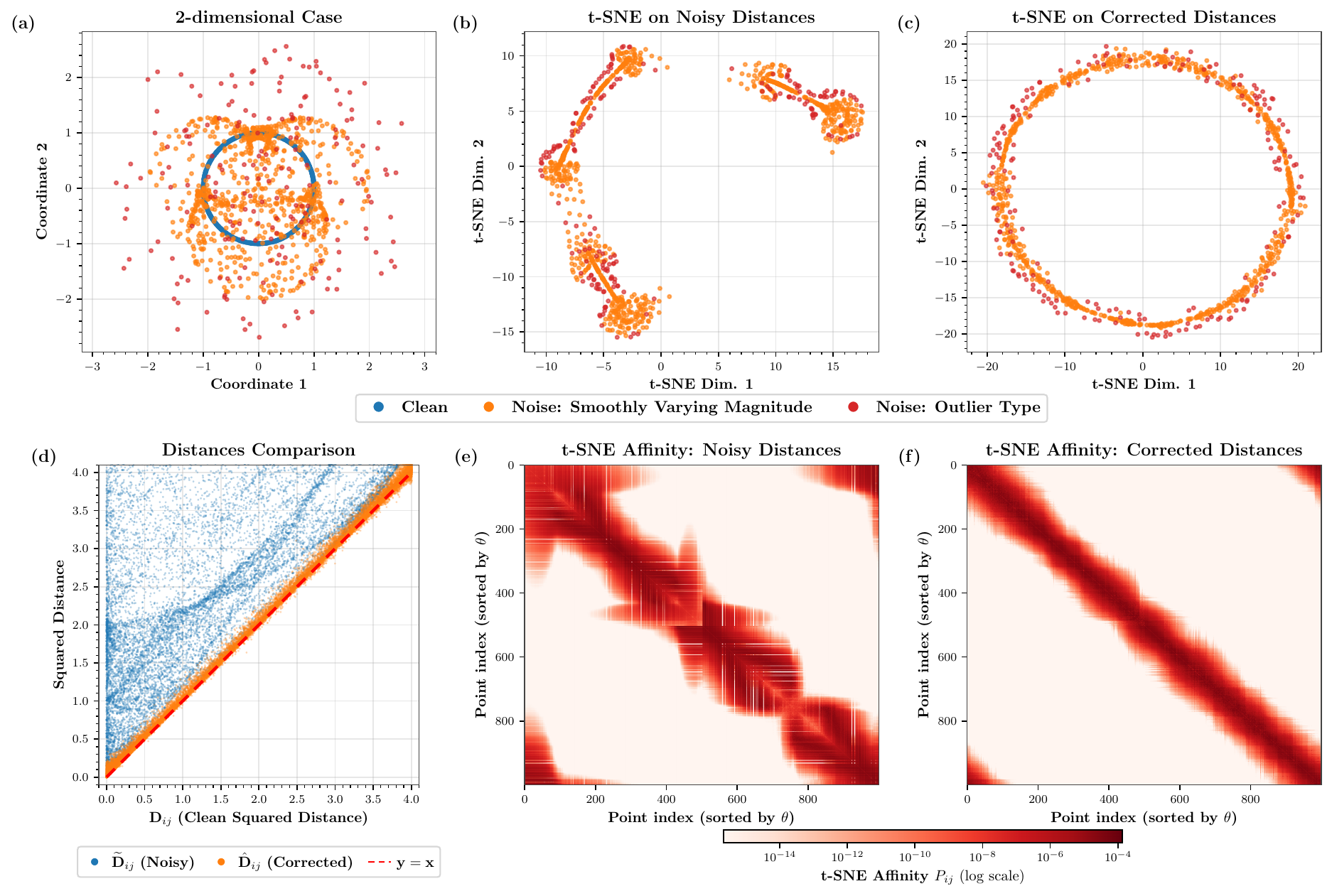}
    \caption{t-SNE embedding using the corrupted distances fails to recover the clean circular geometry, while the corrected distances restore it. (a) Illustration of simulated data for $n=10^3$ and $d=2$. (b)--(c) t-SNE embeddings (perplexity $=50$) computed from the corrupted and corrected distance matrices. (d) Comparison of the clean distances $\mathbf{D}$ from~\eqref{eq:dist_mat}, the corrupted distances $\widetilde{\mathbf{D}}$ from~\eqref{eq:corrupted_dist_mat}, and the corrected distances $\hat{\mathbf{D}}$ from~\eqref{eq:correction}. (e)--(f) The corresponding t-SNE affinities $P_{ij}$ computed from the corrupted and corrected distance matrices, with points ordered by their angular parameter on the circle.}
    \label{fig:intro_ex}
\end{figure}

\subsection{Our contributions}

In this work, we ask the following fundamental question: Can the clean Euclidean distances $\|\mathbf{x}_i-\mathbf{x}_j\|$ be recovered from the observed distances $\|\mathbf{y}_i-\mathbf{y}_j\|$, even when the noise magnitudes $r_i$ in~\eqref{eq:squared_distance_cost} vary substantially across the data and severely distort the observed geometry? This question has not been resolved to the best of our knowledge. We answer this question affirmatively by uncovering a mechanism that enables accurate recovery across a broad class of high-dimensional settings with underlying low-dimensional geometry.

This mechanism relies on key insights involving
the structure of the corruption in~\eqref{eq:squared_distance_cost}, as well as the effects of high dimensionality and large sample size. We briefly explain these ideas informally, while a more detailed treatment is provided in Section~\ref{sec:method derivation}.
First, when the noise vectors $\boldsymbol{\eta}_i$ are high-dimensional and
sufficiently incoherent or delocalized across coordinates, the fluctuation term $\epsilon_{ij}$ in~\eqref{eq:squared_distance_cost} becomes
small. Consequently, the
observed squared distances are approximately related to the clean
ones through the additive distortion
$\|\mathbf{y}_i-\mathbf{y}_j\|^2 \approx
\|\mathbf{x}_i-\mathbf{x}_j\|^2 + r_i + r_j$. Ordinary nearest-neighbor selection based on these corrupted distances is highly sensitive
to variations in the term $r_i +r_j$, and cannot reliably identify points that are close in the clean geometry.
In contrast, we show that the linear sum assignment problem~\cite{siam_assignment_book} is invariant to this particular
type of additive distortion. This property allows optimal assignments based on the corrupted distances
to identify pairs of points that are close in the underlying
clean geometry, even when individual pairwise distances are
substantially distorted.

The second key insight concerns the geometry of the clean data.
When the clean points $\{\mathbf{x}_i\}$ admit an intrinsically
low-dimensional geometric structure (e.g., a low-dimensional manifold), sufficiently large datasets
contain many points that are close in the clean geometry.
Consequently, the clean squared distances
$\|\mathbf{x}_i-\mathbf{x}_j\|^2$ between points matched via
optimal assignment are typically small when $n$ is large,
so~\eqref{eq:squared_distance_cost} yields
$\|\mathbf{y}_i-\mathbf{y}_j\|^2 \approx r_i+r_j$ for these matched pairs.
By combining multiple such assignments, we obtain a system of
approximate linear equations in the unknown noise magnitudes
$\{r_i\}$, from which these magnitudes can be estimated.
Subtracting these magnitudes from the corrupted
squared distances then allows us to approximately reconstruct the underlying
clean distances.
Section~\ref{sec:method derivation} develops these ideas in detail
and explains how the effects of high dimensionality and large
sample size are combined to achieve this recovery.
Figure~\ref{fig:approx} illustrates the key insights underlying our approach.

Overall, we obtain a principled
procedure for recovering the clean Euclidean distances under high-dimensional heterogeneous noise; see
Algorithm~\ref{alg:debias} in Section~\ref{sec:method derivation}. The method
can operate directly on the corrupted pairwise distances, is free of any hyperparameters, and requires no
knowledge of the noise distribution or explicit model for the clean data
beyond the geometric conditions needed for our theoretical guarantees.
When applied to the motivating example in Figure~\ref{fig:intro_ex}, our approach produces corrected distances that concentrate tightly around the clean ones (Figure~\ref{fig:intro_ex}(d)). Hence, our approach can naturally handle both types of heterogeneous noise in this example: the one smoothly varying along the circle and the unstructured outlier-type noise. With these corrected distances, t-SNE accurately recovers both the affinity structure (Figure~\ref{fig:intro_ex}(f)) and the circular geometry (Figure~\ref{fig:intro_ex}(c)) of the clean data. Runtime considerations and scalability of our method are discussed in Appendix~\ref{sec:runtime}.

In Section~\ref{sec:theoretical guarantee}, we establish theoretical guarantees for our approach. We consider
independent, non-identically distributed sub-Gaussian noise, whose magnitude may vary substantially across observations
and can be comparable to or even exceed the scale of the clean data.
Under suitable conditions on the sub-Gaussian norm of the noise
and the underlying clean geometry
(Assumptions~\ref{assump:noise} and~\ref{assump:generalized_cluster}),
we derive probabilistic error bounds for the recovered pairwise distances and noise magnitudes.
Specifically, we show that their normalized $\ell_1$ errors
converge to zero as the feature dimension $d$ and sample size $n$
increase (Theorem~\ref{thm:estimation_error_bound}), with rates
that depend on the feature dimension $d$ and the sub-Gaussian
norm of the noise, and decay polynomially with the sample size
$n$ as $n^{-\alpha}$, where $\alpha>0$ depends on the underlying
clean data geometry. Importantly, for clean data with low-dimensional
intrinsic structure, this exponent is determined by the relevant intrinsic
dimension $m$ rather than the ambient dimension $d$.

Our geometric assumptions on the clean data cover
datasets obtained through bi-Lipschitz embeddings of arbitrary point sets from fixed bounded domains
in $\mathbb{R}^m$ into
$\mathbb{R}^d$, with $m\ll d$
(Theorem~\ref{thm:partition}), as well as
compact $m$-dimensional manifolds (Corollary~\ref{col:manifold})
and unions of such manifolds, without requiring any distributional
assumptions on the clean data. Furthermore, we show that the
recovery guarantees extend beyond the pairwise distances. In particular, we derive an operator-norm error bound
for Gaussian affinity matrices constructed from the corrected
distances, showing that they approximate their clean counterparts
and providing control over their spectral properties
(Corollary~\ref{cor:kernel_operator_norm}).
Our analysis also establishes robustness properties of the linear
sum assignment problem itself. Specifically, we derive bounds on optimal assignment under corrupted distances and forbidden patterns (Lemma~\ref{lem:Geometric LSAP Cost}), showing that the resulting assignments retain
small total costs when evaluated using the clean geometry.

We complement our theoretical analysis with numerical experiments
that examine both the predicted asymptotic behavior of our approach
and its practical benefits. First, in
Section~\ref{sec:demonstrating_example}, we numerically validate
the theoretical results by demonstrating that the optimal-assignment
costs decay with sample size at rates determined by the intrinsic
dimension of the underlying clean geometry, and that the recovery errors for the
noise magnitudes and pairwise distances decrease with $n$ and $d$
at rates consistent with our bounds;
see Figure~\ref{fig:combined_theory}.

    We then investigate the effectiveness of our approach in more challenging settings and its benefits for downstream geometric analysis. In Section~\ref{sec:uniform_density}, we consider synthetic data corrupted by heterogeneous noise whose magnitudes depend on the clean geometry. We demonstrate accurate recovery of the noise magnitudes and pairwise distances, and show how these results enable reliable signal-to-noise ratio assessment and substantially improve the robustness of Gaussian affinity matrices and KNN graphs; see Figure~\ref{fig:unit_circle}. In Section~\ref{sec:self_tune_eig}, we consider the additional challenge of non-uniform sampling density and demonstrate that self-tuning kernels constructed from the corrected distances recover the spectral structure of the clean geometry, overcoming limitations of conventional kernel normalization methods; see Figure~\ref{fig:two_circle}. In Section~\ref{sec:scRNAseq}, we apply our approach to single-cell RNA sequencing data.
    We show that the recovered noise magnitudes agree with predictions via the prototypical Poisson model~\cite{scrnaseq_poisson_model}, while KNN graphs constructed from the corrected distances exhibit improved cell-type neighborhood consistency compared with those from the corrupted distances and related methods; see Figure~\ref{fig:pbmc_neighborhood}. An additional real-world example, in which our approach improves downstream clustering performance relative to competing methods, is provided in Appendix~\ref{sec:real2}.

    \subsection{Related work}\label{sec:related_work}
        Methods for addressing the impact of heterogeneous noise on Euclidean distances have primarily focused on learning robust graphs from noisy data~\cite{boris1,boris2,snekhorn,b_matching}, without explicitly correcting the pairwise distances. Specifically,~\cite{boris1} demonstrated that under mild conditions, Gaussian kernel matrices constructed from datasets corrupted by heterogeneous noise converge in probability to their clean counterparts after doubly stochastic normalization. This normalization is achieved through the Sinkhorn--Knopp algorithm~\cite{sinkhorn}, which iteratively scales the rows and columns of the kernel matrix until it simultaneously achieves unit sums for each row and column. Building on the foundation of~\cite{boris1},~\cite{boris2} leveraged doubly stochastic normalization to develop robust inference tools, including estimators for pairwise Euclidean distances and noise magnitudes. Despite these advances, methods based on doubly stochastic normalization face three significant limitations: (1) they encounter numerical instability and high computational cost when small bandwidth parameters are used for kernel matrices, even though small bandwidths are theoretically desirable for the resulting affinity matrices to approximate important operators under appropriate conditions~\cite{boris3,operator1}; (2) they typically require nontrivial parameter tuning; and (3) they perform poorly in settings with non-uniform data density, as the Gaussian kernels with universal bandwidth parameters inherently fail to adapt to local density variations across data points.

        Several works have sought to simultaneously achieve robustness to heterogeneous noise and adaptivity to data density~\cite{snekhorn,b_matching}.
        In~\cite{snekhorn}, the authors proposed SNEkhorn, a method that constructs an affinity matrix from data by solving an entropy-regularized optimal transport problem that exhibits robustness to heterogeneous noise while adapting to the local density of the data. This framework imposes explicit constraints, requiring the resulting affinity matrix to be row-stochastic and symmetric while enforcing identical entropy across rows to achieve adaptivity to data density. A dual descent optimization framework was proposed to solve the constrained optimization problem. However, this approach is computationally expensive (often requiring GPU acceleration) and demands careful selection of regularization parameters to guarantee the convergence of the optimization procedure. Another notable approach is $b$-matching~\cite{b_matching}, which learns a symmetric, $b$-regular adjacency matrix $\mathbf{A}$ by solving a generalized linear sum assignment problem.
        Given a corrupted squared distance matrix $\widetilde{\mathbf{D}}$ as defined in~\eqref{eq:corrupted_dist_mat}, $b$-matching seeks the optimal feasible adjacency matrix $\mathbf{A}$ that minimizes the objective function $\sum_{ij} \widetilde{D}_{ij} A_{ij}$. This optimization problem can be solved using loopy belief propagation~\cite{loopy_belief_prop}.
        Yet practical applications typically require hyperparameter tuning, necessitating repeated optimizations that significantly increase the computational burden.

        While these aforementioned methods have made important contributions to enhancing robustness against heterogeneous noise, our proposed approach offers distinct advantages in three key aspects: (1) we directly estimate clean pairwise distances rather than constructing robust graphs, enabling seamless application to distance-based methods beyond those that rely on graphs; (2) our approach is computationally tractable, numerically stable, and requires no hyperparameter tuning; and (3) we establish theoretical guarantees in the form of probabilistic error bounds for the recovery errors of noise magnitudes and pairwise Euclidean distances.

        Our work is also related to the field of metric nearness~\cite{metric_repair_review}, which addresses distance corruption under a fundamentally different setting. Specifically, metric nearness tackles the problem of optimally restoring metric properties to a given dissimilarity matrix (e.g., distance matrix) that violates metric axioms due to noise. Prominent approaches in this domain include Triangle-Fixing algorithms~\cite{metric_repair_triangle}, which iteratively correct triangle inequality violations, and norm-based optimization methods~\cite{metric_repair_norm}, which identify the closest valid metric matrix by minimizing the $\ell_{p}$ norm distance between the original and the repaired matrices subject to metric constraints. In contrast, our work considers a setting where the corrupted squared distance matrix $\widetilde{\mathbf{D}}$ in~\eqref{eq:corrupted_dist_mat} is computed directly from the noisy dataset $\mathbf{Y}$. Hence, the square roots of its off-diagonal entries automatically satisfy metric properties by construction.

        An alternative approach to mitigate the effects of heterogeneous noise is to first denoise the data and then compute distances from the denoised observations.
        Notable approaches include Wavelet-based methods~\cite{wavelet_denoise}, which transform data into the wavelet domain to separate noise from meaningful signals through thresholding; Principal Component Analysis~\cite{pca}, which assumes low-rank signal structures and discards components with smaller variances that likely represent noise; and deep learning approaches such as autoencoders~\cite{autoencoder}, which compress corrupted input into latent representations before reconstructing a clean version. However, denoising methods typically require domain-specific knowledge about the underlying structure of the data, such as noise characteristics, low-rank properties, or sparsity patterns, which limits their applications in less-studied domains or when such structures do not naturally hold.

\section{Geometry recovery via optimal assignment in the
high-dimensional, large-sample regime}\label{sec:method derivation}
        {In this section, we explain why the clean distances can be recovered and how our approach (Algorithm~\ref{alg:debias}) achieves this. The approach rests on two asymptotic phenomena, summarized in Figure~\ref{fig:approx}: in high dimensions, the cross terms $\epsilon_{ij}$ vanish, and with large sample sizes, the clean distances between near neighbors vanish. We present the derivation in this asymptotic setting and defer rigorous probabilistic guarantees to Section~\ref{sec:theoretical guarantee}. Since our approach is equivariant under a common rescaling of the data, we assume without loss of generality that $\max_{i\in[n]}\|\mathbf{x}_i\|_2\leq 1$.}

        \begin{figure}[!ht]
            \centering
        \includegraphics[width= \textwidth]{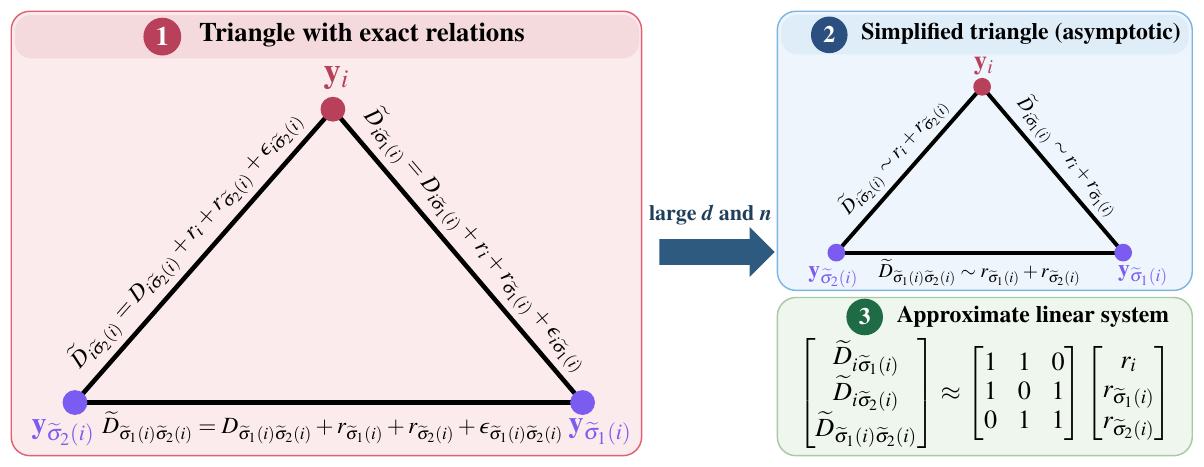}
            \caption{Key phenomena underlying our approach in the regime of large sample size ($n$) and high dimensionality ($d$). For the identified true near neighbors of $\mathbf{y}_i$, denoted by $\mathbf{y}_{\widetilde{\sigma}_1(i)}$ and $\mathbf{y}_{\widetilde{\sigma}_2(i)}$: (1) In high dimensions, $\epsilon_{i\widetilde{\sigma}_1(i)}, \epsilon_{i\widetilde{\sigma}_2(i)}, \epsilon_{\widetilde{\sigma}_1(i)\widetilde{\sigma}_2(i)} \sim 0$; (2) With large sample size, $D_{i\widetilde{\sigma}_1(i)}, D_{i\widetilde{\sigma}_2(i)}, D_{\widetilde{\sigma}_1(i)\widetilde{\sigma}_2(i)} \sim 0$.}
            \label{fig:approx}
        \end{figure}

        We define $\mathbf{D} = \{D_{ij}\}$ as the matrix of squared pairwise Euclidean distances in the clean data $\mathbf{X}$ after setting its main diagonal to infinity, i.e.,
                    \begin{equation}\label{eq:dist_mat}
            D_{ij} =
            \begin{cases}
                \|\mathbf{x}_i - \mathbf{x}_j\|^2_2, & i \neq j, \\
                +\infty, & i = j,
            \end{cases}
            \end{equation}
        for any $i,j \in [n]$. Similarly, $\widetilde{\mathbf{D}}$ mirrors this construction using the noisy distances; see~\eqref{eq:corrupted_dist_mat} in Algorithm~\ref{alg:debias}. The diagonal entries of $\mathbf{D}$ and $\widetilde{\mathbf{D}}$ are set to $+\infty$ to exclude trivial self-comparisons and to serve important methodological purposes that will be detailed later.

        We consider a high-dimensional setting similar to those examined in~\cite{boris2,homoskedastic_theory,homoskedastic_theory1}, where the feature dimension $d$ is sufficiently large and the noise vectors $\{\boldsymbol{\eta}_i\}_{i=1}^n$ in~\eqref{eq:noise_model} do not concentrate excessively in any particular direction. This condition is formalized as Assumption~\ref{assump:noise} in Section~\ref{sec:theoretical guarantee}. Importantly, in this setting, the term $\epsilon_{ij}$ in~\eqref{eq:squared_distance_cost} concentrates around zero with high probability (see Lemma~\ref{lem:epsilon_bound_in_prob}). This concentration property leads to the following asymptotic relationship: as $d\rightarrow\infty$, for any off-diagonal entries of
        $\widetilde{\mathbf{D}}$ in~\eqref{eq:corrupted_dist_mat}:
            \begin{equation}\label{eq:approx_ep}
                \widetilde{D}_{ij} \sim D_{ij} + {r}_i + {r}_j,
            \end{equation}
            where $r_i$ and $r_j$ are the noise magnitudes in~\eqref{eq:squared_distance_cost}, and $D_{ij}$ is the clean distance in~\eqref{eq:dist_mat}. The relationship in~\eqref{eq:approx_ep} enables recovering $D_{ij}$ by estimating and subtracting $r_i$ and $r_j$ from $\widetilde{D}_{ij}$.

        To estimate the noise magnitudes $\{{r}_i\}_{i=1}^n$ in~\eqref{eq:squared_distance_cost}, we leverage the geometric properties of the clean dataset $\mathbf{X}$. Specifically, when $\mathbf{X}$ is sufficiently large and satisfies suitable geometric conditions (see Assumption~\ref{assump:generalized_cluster}), the distances between any clean data point $\mathbf{x}_i \in \mathbf{X}$ and its near neighbors vanish to 0 as the sample size $n$ increases. This property implies that if we can identify pairs of data points in $\mathbf{Y}$ whose clean counterparts are geometrically close in $\mathbf{X}$, the clean distances between these data pairs approach 0 asymptotically. In particular, suppose we know a bijective function $\widetilde{\sigma}: [n] \rightarrow [n]$ that assigns to each $\mathbf{y}_i \in \mathbf{Y}$ a point $\mathbf{y}_{\widetilde{\sigma}(i)} \in \mathbf{Y}$ (with $\widetilde{\sigma}(i) \neq i$, $\forall i \in [n]$) such that their clean counterparts $\mathbf{x}_i$ and $\mathbf{x}_{\widetilde{\sigma}(i)}$ are close (i.e., $D_{i \widetilde{\sigma}(i)} \sim 0$ as $n\rightarrow \infty$). Then, as $d$ and $n$ tend to infinity, the corrupted distance $\widetilde{D}_{i\widetilde{\sigma}(i)}$ satisfies: \begin{equation}\label{eq:approx_C}
            \widetilde{D}_{i\widetilde{\sigma}(i)}\sim r_i + r_{\widetilde{\sigma}(i)}, \qquad \forall i \in [n].
        \end{equation}
        We extend this insight to
        two true near neighbors $\mathbf{y}_{\widetilde{\sigma}_1(i)}$ and $\mathbf{y}_{\widetilde{\sigma}_2(i)}$ of $\mathbf{y}_i$ (i.e., $\mathbf{x}_{\widetilde{\sigma}_1(i)}$ and $\mathbf{x}_{\widetilde{\sigma}_2(i)}$ are geometrically close to $\mathbf{x}_i$). Given that $D_{i\widetilde{\sigma}_1(i)}, D_{i\widetilde{\sigma}_2(i)}\sim 0$, the triangle inequality gives $D_{\widetilde{\sigma}_1(i)\widetilde{\sigma}_2(i)} \sim 0$, yielding the following approximate $3\times3$ linear system for large $d$ and $n$:
        \begin{equation}    \label{eq:linear_system}
            \begin{bmatrix}
                \widetilde{D}_{i\widetilde{\sigma}_1(i)} \\
                \widetilde{D}_{i\widetilde{\sigma}_2(i)} \\
                \widetilde{D}_{\widetilde{\sigma}_1(i)\widetilde{\sigma}_2(i)}
                \end{bmatrix}
                \sim
                \begin{bmatrix}
                r_i + r_{\widetilde{\sigma}_1(i)} \\
                r_i + r_{\widetilde{\sigma}_2(i)} \\
                r_{\widetilde{\sigma}_1(i)} + r_{\widetilde{\sigma}_2(i)}
                \end{bmatrix}
                =
                \begin{bmatrix}
                1 & 1 & 0 \\
                1 & 0 & 1 \\
                0 & 1 & 1
                \end{bmatrix}
                \begin{bmatrix}
                r_i \\
                r_{\widetilde{\sigma}_1(i)} \\
                r_{\widetilde{\sigma}_2(i)}
                \end{bmatrix}
        \end{equation}
        This linear system admits a unique solution for each $i\in [n]$ (solved by treating the asymptotic relationships as exact equalities), yielding the noise magnitude estimates $\{\hat{r}_i\}_{i=1}^n$ in~\eqref{eq:noise_mag_estimate}.

        \begin{algorithm}[!ht]
            \caption{Euclidean Distance Recovery via Optimal Assignments}\label{alg:debias}
            
            \setlength{\abovedisplayskip}{3pt}
            \setlength{\belowdisplayskip}{2pt}
            \setlength{\abovedisplayshortskip}{3pt}
            \setlength{\belowdisplayshortskip}{2pt}

            \begin{algorithmic}[1]
                \State Construct the cost matrix $\widetilde{\mathbf{D}} \in \mathbb{R}^{n\times n}$ according to
                \begin{equation}\label{eq:corrupted_dist_mat}
                \widetilde{D}_{ij} = \begin{cases}
                            \|\mathbf{y}_i - \mathbf{y}_j\|^2_2, & i \neq j, \\
                            +\infty, & i = j,
                        \end{cases} \quad \forall i,j \in [n].
                \end{equation}

                \State Identify two distinct true near neighbors for each $\mathbf{y}_i \in \mathbf{Y}$ by solving the following linear sum assignment problem:

                \begin{enumerate}[label=(\alph*)]
                    \item 1st round of true near neighbor identification:
                    \begin{equation}\label{eq:alg_LSAP_1}
                        \widetilde{\mathbf{P}}^{(1)} = \operatorname*{argmin}_{\mathbf{P} \in \mathcal{P}^n} \operatorname{Tr}(\mathbf{P}^{T}\widetilde{\mathbf{D}}),
                    \end{equation}
                    where $\mathcal{P}^n$ is the set of permutation matrices in $\mathbb{R}^{n\times n}$ and $\operatorname{Tr}(\cdot)$ is the Trace operator. Define the permutation function
                     $\widetilde{\sigma}_1$: $[n] \rightarrow [n]$, where $\widetilde{\sigma}_1(i) = j$ if
                     $\widetilde{P}^{(1)}_{ij} = 1$.

                    \item Construct a modified cost matrix $\widetilde{\mathbf{D}}^\prime \in \mathbb{R}^{n\times n}$ from $\widetilde{\mathbf{D}}$ to avoid identifying the same true near neighbor pairs already identified in~\eqref{eq:alg_LSAP_1}:
                    \begin{equation}\label{eq:modified_D}
                        \widetilde{D}^{\prime}_{ij} = \begin{cases}
                            \widetilde{D}_{ij}, & \widetilde{\sigma}_1(i) \neq j, \\
                            +\infty, & \widetilde{\sigma}_1(i) = j,
                        \end{cases}
                        \quad \forall i,j \in[n].
                    \end{equation}

                    \item 2nd round of true near neighbor identification:
                    \begin{equation}\label{eq:alg_LSAP_2}
                     \widetilde{\mathbf{P}}^{(2)} = \operatorname*{argmin}_{\mathbf{P} \in \mathcal{P}^n} \operatorname{Tr}(\mathbf{P}^T\widetilde{\mathbf{D}}^{\prime}).
                    \end{equation}
                    Define $\widetilde{\sigma}_2$: $[n] \rightarrow [n]$, where $\widetilde{\sigma}_2(i) = j$ if $\widetilde{P}^{(2)}_{ij} = 1$.
                \end{enumerate}

                \State Solve for the noise magnitudes $\hat{\mathbf{r}} = [\hat{r}_1, \hat{r}_2, \ldots, \hat{r}_n]^T$ as
                    \begin{equation}\label{eq:noise_mag_estimate}
                        \hat{r}_i = \frac{1}{2}\left(\widetilde{D}_{i\widetilde{\sigma}_1(i)} + \widetilde{D}_{i\widetilde{\sigma}_2(i)} - \widetilde{D}_{\widetilde{\sigma}_1(i)\widetilde{\sigma}_2(i)}\right), \quad \forall i \in [n].
                    \end{equation}

                \State Recover the corrected distance matrix $\hat{\mathbf{D}} \in \mathbb{R}^{n\times n}$ according to
                \begin{equation}\label{eq:correction}
                    \hat{\mathbf{D}} = \widetilde{\mathbf{D}} - \hat{\mathbf{r}}\mathbf{1}^{T} - \mathbf{1}\hat{\mathbf{r}}^{T}.
                \end{equation}
            \end{algorithmic}
        \end{algorithm}

        We now explain how to identify such true near neighbors using only the corrupted dataset $\mathbf{Y}$ or its pairwise distance matrix $\widetilde{\mathbf{D}}$.
        We note that greedily selecting nearest neighbors
        is inadequate: as shown in~\eqref{eq:squared_distance_cost}, the noise magnitudes $\{r_i\}_{i=1}^n$ distort the distances in $\widetilde{\mathbf{D}}$ in a nontrivial way, which may create misleading proximity relationships where corrupted data points appear geometrically close even when their clean counterparts are distant.

        To identify true near neighbors, we propose solving an optimization problem that satisfies two key criteria: (1) when optimized using the clean distance matrix $\mathbf{D}$, it correctly identifies near neighbors for each $\mathbf{x}_i \in \mathbf{X}$, and (2) when optimized using the corrupted distance matrix $\widetilde{\mathbf{D}}$, its solution remains invariant to the additive bias terms arising from noise magnitudes (specifically, the term $r_i + r_j$ in~\eqref{eq:squared_distance_cost} for any $i\neq j$). To this end, we consider the linear sum assignment problem (LSAP)~\cite{siam_assignment_book}, a classical combinatorial optimization problem that admits efficient solution via the Hungarian algorithm~\cite{hungarian}. The LSAP aims to identify the optimal one-to-one matching between two sets of objects that minimizes the total costs of pairwise matching~\cite{siam_assignment_book}. It has found applications in numerous domains, including personnel assignment
        and graph theory.
        Formally, given a cost matrix $\mathbf{C}\in \mathbb{R}^{n\times n}$ where $C_{ij}$ represents the cost of assigning object $i$ from the first set to object $j$ from the other set, the LSAP is formulated as:\begin{equation}\label{eq:lsap_def}
            \mathbf{P}^\ast
                 = \operatorname*{argmin}_{\mathbf{P} \in \mathcal{P}^n} \operatorname{Tr}(\mathbf{P}^{T}\mathbf{C}) = \operatorname*{argmin}_{\mathbf{P} \in \mathcal{P}^n} \sum_{i=1}^n\sum_{j=1}^n P_{ij}C_{ij},
        \end{equation}
        where $\mathcal{P}^n$ is the set of permutation matrices in $\mathbb{R}^{n\times n}$ and $\operatorname{Tr}(\cdot)$ is the Trace operator. The optimal assignment $\mathbf{P}^\ast$ matches object $i$ to object $j$ when $P^\ast_{ij} = 1$.
        In our context, when using the clean distance matrix $\mathbf{D}$ as the cost matrix, each entry $D_{ij}$ represents the cost of assigning $\mathbf{x}_j$ as the near neighbor of $\mathbf{x}_i$. By setting the diagonal of $\mathbf{D}$ to $+\infty$ in~\eqref{eq:dist_mat}, we prevent any self-assignment (i.e. $P^{\ast}_{ii} \neq 1$, $\forall i \in [n]$), ensuring that the optimal solution identifies one distinct near neighbor for each $\mathbf{x}_i \in \mathbf{X}$, in a globally optimal manner.

        We now explain the aforementioned invariance property of this optimization problem. When using $\widetilde{\mathbf{D}}$ as the cost matrix, given that its diagonal is set to $+\infty$, the LSAP only considers the off-diagonal entries. For any feasible $\mathbf{P}$, observe that when $\epsilon_{ij} \sim 0$ for all $i\neq j$,
        \begin{equation} \label{eq:LSAP_robust}
            \begin{aligned}
            \operatorname{Tr}(\mathbf{P}^T\widetilde{\mathbf{D}})
            &= \sum_{i,j = 1}^n P_{ij} \widetilde{D}_{ij}
            =  \sum_{i,j = 1}^n P_{ij} \left( D_{ij} + r_i + r_j + \epsilon_{ij} \right) \\
            &\sim \sum_{i,j = 1}^n P_{ij} \left( D_{ij} + r_i + r_j \right) \\
            &= \sum_{i,j = 1}^n P_{ij}  D_{ij} +  \sum_{i=1}^n r_i \sum_{j=1}^n P_{ij} + \sum_{j=1}^n r_j \sum_{i=1}^n P_{ij} \\
            &= \operatorname{Tr}(\mathbf{P}^T\mathbf{D}) + 2\sum_{i=1}^n r_i.
            \end{aligned}
        \end{equation}
        Since the term $2\sum_{i=1}^n r_i$ is independent of
        $\mathbf{P}$, it does not affect the optimization.
        Consequently,
        using $\widetilde{\mathbf{D}}$ as the cost matrix in the LSAP optimization yields the same assignment as using $\mathbf{D}$. This enables the identification of true near neighbors using only the corrupted dataset $\mathbf{Y}$.

        To identify two distinct true near neighbors for each $\mathbf{y}_i \in \mathbf{Y}$, our algorithm solves the LSAP twice: first with the cost matrix $\widetilde{\mathbf{D}}$ from~\eqref{eq:corrupted_dist_mat}, then with $\widetilde{\mathbf{D}}^{\prime}$ from~\eqref{eq:modified_D}. In the second LSAP, entries of $\widetilde{\mathbf{D}}^\prime$ that correspond to the neighbor pairs identified in the first round are set to $+\infty$ to ensure that each $\mathbf{y}_i \in \mathbf{Y}$ is assigned a neighbor distinct from the first round. These neighbor pairs are then utilized to compute noise magnitude estimates $\{\hat{r}_i\}_{i=1}^n$ via~\eqref{eq:noise_mag_estimate}, which are subsequently used to derive the corrected distance matrix $\hat{\mathbf{D}}$ according to~\eqref{eq:correction}.
    \section{Theoretical guarantees}\label{sec:theoretical guarantee}

        In this section, we establish theoretical guarantees for our approach. Specifically, we analyze the recovery error of the squared noise magnitudes $\{{\hat{r}}_i\}_{i=1}^n$ from~\eqref{eq:noise_mag_estimate} and the corrected distance matrix $\hat{\mathbf{D}}$ from~\eqref{eq:correction}. We begin by formalizing our problem setting through two key assumptions: the sub-Gaussian properties of the heterogeneous noise relative to the scale of the clean data (Assumption~\ref{assump:noise}) and the geometric characteristics of the underlying clean dataset $\mathbf{X}$ (Assumption~\ref{assump:generalized_cluster}). We establish the broad relevance of Assumption~\ref{assump:generalized_cluster} by proving that it covers many data generative models, including any dataset obtained by a bi-Lipschitz embedding into $\mathbb{R}^d$ of an arbitrary point set drawn from a fixed bounded domain in $\mathbb{R}^m$ with $m \ll d$ (Theorem~\ref{thm:partition}), and, more generally, from a fixed compact $m$-dimensional manifold (Corollary~\ref{col:manifold}) or a union of such manifolds of possibly differing intrinsic dimensions. We then connect the geometric properties of $\mathbf{X}$ in Assumption~\ref{assump:generalized_cluster} to Algorithm~\ref{alg:debias} by deriving upper bounds for the LSAP costs associated with steps~\eqref{eq:alg_LSAP_1} and~\eqref{eq:alg_LSAP_2}; see Lemma~\ref{lem:Geometric LSAP Cost}. Finally, we use Lemma~\ref{lem:Geometric LSAP Cost} to establish our main theoretical result (Theorem~\ref{thm:estimation_error_bound}): probabilistic bounds on the normalized $\ell_1$ recovery errors of $\{{\hat{r}}_i\}_{i=1}^n$ and $\hat{\mathbf{D}}$. We then derive an operator-norm error bound for fixed-bandwidth Gaussian affinity matrices constructed from the corrected distances (Corollary~\ref{cor:kernel_operator_norm}).
        All proofs are deferred to the appendices.

        We recall that a random vector $\boldsymbol{\xi} \in \mathbb{R}^d$ is called sub-Gaussian if for any vector $\mathbf{u} \in \mathbb{R}^d$, the standard inner product $\langle \boldsymbol{\xi} , \mathbf{u}\rangle$ is a sub-Gaussian random variable~\cite{HDP_book}. For each $\mathbf{x}_i \in \mathbf{X}$, we assume that the corresponding noise vector $\boldsymbol{\eta}_i$ is sampled independently from a sub-Gaussian distribution $\boldsymbol{\eta}(\mathbf{x}_i)$ with zero mean (i.e., $\mathbb{E}[\boldsymbol{\eta}(\mathbf{x}_i)] = \mathbf{0}$). The notation $\boldsymbol{\eta}(\mathbf{x}_i)$ indicates that the noise distribution may depend on the data point $\mathbf{x}_i$, allowing for heterogeneity across the dataset. Let $\|\boldsymbol{\eta}(\mathbf{x}_i)\|_{\psi_2}$ be the sub-Gaussian norm of $\boldsymbol{\eta}(\mathbf{x}_i)$, defined as:
        \begin{equation}
            \|\boldsymbol{\eta}(\mathbf{x}_i)\|_{\psi_2} = \sup_{\|\mathbf{u}\|_2 = 1, \boldsymbol{\eta}\sim \boldsymbol{\eta}(\mathbf{x}_i)} \|\langle \boldsymbol{\eta}, \mathbf{u}\rangle\|_{\psi_2},
        \end{equation}
        where $\|\cdot\|_{\psi_2}$ on the right denotes the sub-Gaussian norm of a random variable. We make the following assumption on the scale of the clean data and the sub-Gaussian norm of the noise.
        \begin{assump} \label{assump:noise}
            There exists a deterministic scale $R_d>0$, possibly depending on the feature dimension $d$, such that $\|\mathbf{x}_i\|_2\leq R_d$ for all $\mathbf{x}_i \in \mathbf{X}$ and
            \[
                E := \frac{1}{R_d}\max_{\mathbf{x}_i \in \mathbf{X}} \|\boldsymbol{\eta}(\mathbf{x}_i)\|_{\psi_2} \leq \frac{C}{{d^{1/4}\sqrt{\log d}}}
            \]
            for some absolute constant $C > 0$.
        \end{assump}

        Assumption~\ref{assump:noise} allows the absolute noise scale to vary with the scale $R_d$ of the clean data. For example, a $d$-dimensional Gaussian distribution $\mathcal{N}(0, \mathbf{\Sigma})$ with spectral norm $\|\mathbf{\Sigma}\|_2 \leq C_1R_d^2/(d^{1/2} \log d)$ constitutes a qualified distribution class for a suitable constant $C_1>0$. This includes the special case $\mathbf{\Sigma} = \frac{1}{d} R_d^2 \mathbf{I}_d$, where the coordinates of noise vectors are independent and identically distributed. In particular, in settings where the natural scale of the clean data satisfies $R_d\asymp\sqrt{d}$, Assumption~\ref{assump:noise} includes Gaussian noise $\mathcal{N}(0,\sigma^2\mathbf{I}_d)$ with $\sigma^2$ bounded independently of $d$. Notably, the noise magnitude $\|\boldsymbol{\eta}(\mathbf{x}_i)\|_2$ is still allowed to exceed the clean-data scale $R_d$. For instance, when $\boldsymbol{\eta}(\mathbf{x}_i) = \mathcal{N}\left(\mathbf{0}, \mathbf{\Sigma}\right)$ with $ \mathbf{\Sigma}= R_d^2\mathbf{I}_d / ({d^{1/2} \log d})$, we have $\mathbb{E}\|\boldsymbol{\eta}(\mathbf{x}_i)\|_2^2 = \operatorname{Tr}\left(\mathbf{\Sigma}\right) = R_d^2\sqrt{d}/\log d$, which can significantly exceed $R_d^2$. Furthermore, Assumption~\ref{assump:noise} also accommodates noise vectors whose coordinates are not independent or identically distributed and permits different noise distributions for different $\mathbf{x}_i \in \mathbf{X}$.

         We now turn to describe our requirements on the clean dataset $\mathbf{X}$. Before stating our main assumption (Assumption~\ref{assump:generalized_cluster}), we define the notions of \textit{diameter} and \textit{data partition} below.

        \begin{definition} \label{def:diameter}
            The diameter of a subset $\mathbf{S} \subset \mathbf{X}$ is defined as
            $\operatorname{diam} (\mathbf{S}) := \sup\limits_{\substack{\mathbf{x}_i, \mathbf{x}_j \in \mathbf{S}}}
            \Vert \mathbf{x}_i - \mathbf{x}_j \Vert_2$.
        \end{definition}

        \begin{definition} \label{def:partition}
            A partition of the dataset $\mathbf{X}$ is defined as $\mathcal{P}_{\mathbf{X}}: =\{\mathcal{P}_i\}_{i=1}^{k}$ satisfying
            \[
                \mathbf{X} = \bigcup_{i=1}^k \mathcal{P}_i, \quad   \text{and} \quad \mathcal{P}_i \cap \mathcal{P}_j = \emptyset \text{, } \forall i \neq j.
            \]
        \end{definition}

        \begin{assump} \label{assump:generalized_cluster}
            Given a clean dataset $\mathbf{X} = \{\mathbf{x}_i\}_{i=1}^n$ and the scale $R_d$ introduced above, there exists a partition $\mathcal{P}_{\mathbf{X}} = \{\mathcal{P}_j\}_{j=1}^{k}$ such that each subset $\mathcal{P}_j$ contains at least four points (i.e., $|\mathcal{P}_j|\geq 4$) and the weighted average squared diameter (weighted by the size of subset) satisfies
            \begin{equation}\label{eq:cluster_diam_bound}
            \frac{1}{nR_d^2} \sum_{\mathcal{P}_j \in \mathcal{P}_{\mathbf{X}}} |\mathcal{P}_j| \left(\operatorname{diam}(\mathcal{P}_j)\right)^2 \leq cn^{-\alpha},
            \end{equation}
            for some constant $c \geq 0$ and $\alpha > 0$, where $|\mathcal{P}_j|$ denotes the cardinality of the subset $\mathcal{P}_j$.
        \end{assump}

        Assumption~\ref{assump:generalized_cluster}
        characterizes
        the geometric properties of a sufficiently large and well-behaved dataset $\mathbf{X}$. It states that such datasets can be partitioned into subsets, each satisfying a minimum size constraint, such that on average, the maximum squared distance between any pair of clean data points within the same subset, relative to $R_d^2$,
        decreases polynomially as the dataset size $n$ grows. This property allows us to justify~\eqref{eq:approx_C} after normalization by $R_d^2$: provided the near neighbor for each $\mathbf{x}_i$ is correctly identified, $D_{i\widetilde{\sigma}(i)}/R_d^2$—the normalized squared distance between $\mathbf{x}_i$ and its assigned near neighbor $\mathbf{x}_{\widetilde{\sigma}(i)}$—approaches zero as $n \rightarrow +\infty$.

        Assumption~\ref{assump:generalized_cluster} encompasses a wide range of practical scenarios. The simplest example is when each $\mathbf{x}_i \in \mathbf{X}$ takes a value from a finite set $\{\boldsymbol{\mu}_1, \ldots, \boldsymbol{\mu}_k\}$, with at least four data points sharing the value of each $\boldsymbol{\mu}_j$. Under this setup, we can construct a partition $\mathcal{P}_\mathbf{X}$ by grouping points with identical values in the same subset. Formally, we define each subset as $\mathcal{P}_j := \{\mathbf{x}_i \in \mathbf{X} \mid \mathbf{x}_i = \boldsymbol{\mu}_j\}, \text{ } \forall j \in [k]$. With this construction, $\operatorname{diam}(\mathcal{P}_j) = 0$ for all subsets $\mathcal{P}_j \in \mathcal{P}_\mathbf{X}$. Consequently, $\frac{1}{nR_d^2} \sum_{\mathcal{P}_j \in \mathcal{P}_{\mathbf{X}}} |\mathcal{P}_j| \left(\operatorname{diam}(\mathcal{P}_j)\right)^2 = 0$, representing a special case of Assumption~\ref{assump:generalized_cluster} where $\alpha \rightarrow +\infty$ and $c = 0$. Such a structure of $\mathbf{X}$ naturally arises when the corrupted dataset $\mathbf{Y}$ is generated from a mixture model, such as the Gaussian mixture model. Specifically, each $\mathbf{y}_i \in \mathbf{Y}$ can be expressed as $\mathbf{y}_i = \mathbf{x}_i + \boldsymbol{\eta}_i$, where $\mathbf{x}_i = \boldsymbol{\mu}_l$ represents the centroid of the mixture component $l$ from which $\mathbf{y}_i$ is generated, and $\boldsymbol{\eta}_i$ is a noise vector sampled from a centered, component-specific Gaussian distribution.

        Aside from the simple case of a mixture model, Assumption~\ref{assump:generalized_cluster} holds for a broad range of clean datasets encountered in practice. In particular, it is commonly assumed that clean data points lie on a low-dimensional subspace or manifold embedded in a high-dimensional ambient space, an assumption known as the manifold hypothesis~\cite{manifold2,manifold3}. Theorem~\ref{thm:partition} below shows that Assumption~\ref{assump:generalized_cluster} is satisfied by any dataset $\mathbf{X} \subset \mathbb{R}^d$ obtained by a bi-Lipschitz embedding into $\mathbb{R}^d$ of an arbitrary point set drawn from a fixed bounded domain in $\mathbb{R}^m$ with $m \ll d$. Notably, the result is deterministic and requires no distributional assumptions on the data, with the exponent $\alpha$ depending only on the intrinsic dimension $m$.
        \begin{theorem}\label{thm:partition}
        Let $\Omega \subset \R^m$ be a fixed bounded set and
        $\mathbf{Z} = \{\mathbf{z}_i\}_{i=1}^n \subset \Omega$ be an arbitrary set of
        $n \ge 4$ points. Let $\varphi : \R^m \to \R^d$ be a bi-Lipschitz embedding with
        $d \ge m$, i.e., there exist constants $0 < a \le b$, possibly depending on $d$,
        such that
        \begin{equation}\label{eq:bilipschitz}
            a\,\|\mathbf{z} - \mathbf{z}'\|_2 \;\le\; \|\varphi(\mathbf{z}) - \varphi(\mathbf{z}')\|_2
            \;\le\; b\,\|\mathbf{z} - \mathbf{z}'\|_2
        \end{equation}
        for all $\mathbf{z}, \mathbf{z}' \in \Omega$. Assume that the distortion
        $\kappa := b/a$ satisfies $\kappa \le \kappa_0$ for some constant
        $\kappa_0 \ge 1$ independent of $d$, and let $R_d > 0$ be a deterministic scale, possibly depending on $d$, such that $\|\varphi(\mathbf{z})\|_2 \le R_d$ for all
        $\mathbf{z} \in \Omega$. Then $\mathbf{X} = \{\mathbf{x}_i\}_{i=1}^n \subset \R^d$,
        where $\mathbf{x}_i = \varphi(\mathbf{z}_i)$, satisfies
        Assumption~\ref{assump:generalized_cluster} with
        $\alpha = \min\{\tfrac{2}{m}, 1\}$ and $c = C m \kappa_0^2$, for an absolute
        constant $C > 0$.
    \end{theorem}

        We emphasize that the intrinsic dimension $m$ is distinct from the
        feature dimension $d$. Theorem~\ref{thm:partition} establishes an explicit
        relationship between the convergence rate $\alpha$ in
        Assumption~\ref{assump:generalized_cluster} and the intrinsic dimension $m$. Specifically, let $\mathbf{X}$ be obtained by a bi-Lipschitz embedding into $\mathbb{R}^d$ of an arbitrary point set drawn from a fixed bounded domain in $\mathbb{R}^m$. Then $\mathbf{X}$ admits a partition into subsets of at least four points whose scaled average squared diameter decays polynomially in $n$, at the rate $\alpha = \min\{2/m, 1\}$. Since the constant $c = Cm\kappa_0^2$ depends only on $\kappa_0$, and not on $a$ and $b$ individually, the scale of $\mathbf{X}$ may vary with $d$.

        Theorem~\ref{thm:partition} is stated for datasets that are bi-Lipschitz to a point set in $\mathbb{R}^m$. Corollary~\ref{col:manifold} below extends the same conclusion to point sets drawn from a fixed compact $m$-dimensional manifold and bi-Lipschitz embedded into $\mathbb{R}^d$, with the rate $\alpha$ again determined by the intrinsic dimension $m$ of the manifold rather than the ambient dimension $d$.
        \begin{corollary}\label{col:manifold}
        Let $\mathcal{M}$ be a fixed compact $C^1$ submanifold of some Euclidean space,
        of intrinsic dimension $m \ge 1$ and possibly with boundary, and let
        $\mathbf{Z} = \{\mathbf{z}_i\}_{i=1}^n \subset \mathcal{M}$ be an arbitrary set
        of $n \ge 4$ points. Let $\varphi : \mathcal{M} \to \R^d$ satisfy
    \eqref{eq:bilipschitz} for all $\mathbf{z}, \mathbf{z}' \in \mathcal{M}$, with
    $\kappa$, $\kappa_0$ and $R_d$ as in Theorem~\ref{thm:partition} (with $\Omega$
    replaced by $\mathcal{M}$). Then
        $\mathbf{X} = \{\mathbf{x}_i\}_{i=1}^n$, where $\mathbf{x}_i = \varphi(\mathbf{z}_i)$,
        satisfies Assumption~\ref{assump:generalized_cluster} with
        $\alpha = \min\{\tfrac{2}{m}, 1\}$ and $c = C_{\mathcal{M}}\kappa_0^2$, where
        $C_{\mathcal{M}} < \infty$ depends only on $\mathcal{M}$.
    \end{corollary}

    Corollary~\ref{col:manifold} imposes no assumptions on how the points are distributed over $\mathcal{M}$.
    The same conclusion extends to data residing on a union of manifolds of possibly differing intrinsic dimensions, a setting frequently encountered in practice. Specifically, let $\mathcal{M}_1, \dots, \mathcal{M}_K$ be a fixed number of manifolds, each satisfying the conditions of Corollary~\ref{col:manifold} with intrinsic dimension $m_k$, its own bi-Lipschitz embedding $\varphi_k$ of distortion at most $\kappa_0$, and a common scale $R_d$. Suppose that $n_k \ge 4$ of the $n$ points lie on $\mathcal{M}_k$, where each point is assigned to exactly one manifold. Applying Corollary~\ref{col:manifold} to each manifold and taking the union of the resulting partitions yields a partition $\mathcal{P}_{\mathbf{X}}$ satisfying
    \[
      \frac{1}{nR_d^2} \sum_{\mathcal{P}_j \in \mathcal{P}_{\mathbf{X}}}
      |\mathcal{P}_j| \big(\diam(\mathcal{P}_j)\big)^2
      \le \sum_{k=1}^K \frac{n_k}{n}\, C_{\mathcal{M}_k}\kappa_0^2\, n_k^{-\alpha_k}
      \le \kappa_0^2 \sum_{k=1}^K C_{\mathcal{M}_k}\, n^{-\alpha_k},
    \]
    where \sloppy$\alpha_k = \min\{2/m_k, 1\}$.
    Hence
    $\mathbf{X}$ satisfies Assumption~\ref{assump:generalized_cluster} with
    $\alpha = \min_k \alpha_k = \min\{2/\max_k m_k,\, 1\}$ and
    $c = \kappa_0^2 \sum_{k=1}^K C_{\mathcal{M}_k}$. In particular, the rate is governed by
    the manifold of largest intrinsic dimension, which contributes the slowest-decaying
    term.

        We now connect Assumption~\ref{assump:generalized_cluster} to the recovery accuracy of our approach. Before analyzing the recovery accuracy of $\hat{\mathbf{r}}$ in~\eqref{eq:noise_mag_estimate} and $\hat{\mathbf{D}}$ in~\eqref{eq:correction},
        we first derive bounds for the LSAP costs associated with identifying near neighbors for each clean data $\mathbf{x}_i \in \mathbf{X}$ through two successive rounds of LSAP optimization, as detailed in Lemma~\ref{lem:Geometric LSAP Cost}.

        \begin{lemma}\label{lem:Geometric LSAP Cost}
            Let $\mathbf{X} = \{\mathbf{x}_i\}_{i=1}^n$ be a clean dataset satisfying Assumption~\ref{assump:generalized_cluster} and let $\mathbf{D}$ be the squared pairwise Euclidean distance matrix constructed from $\mathbf{X}$ according to~\eqref{eq:dist_mat}. Given any permutation matrix $\mathbf{P} \in \mathcal{P}^n$, we define a masked cost matrix $\mathbf{D}^{\prime}$ as follows:
            \begin{equation}\label{eq:masked_D_def}
                D^{\prime}_{ij} = \begin{cases}
                    D_{ij}, & P_{ij} = 0, \\
                    +\infty, & P_{ij} = 1.
                \end{cases}
            \end{equation}
            Consider the optimal permutation matrices $\mathbf{P}^{(1)}$ and $\mathbf{P}^{(2)}$ obtained via:
            \begin{equation}\label{eq:lemma_lsap}
                \mathbf{P}^{(1)} = \operatorname*{argmin}_{\mathbf{P} \in \mathcal{P}^n} \operatorname{Tr}(\mathbf{P}^T\mathbf{D}), \quad \text{and} \quad \mathbf{P}^{(2)} = \operatorname*{argmin}_{\mathbf{P} \in \mathcal{P}^n} \operatorname{Tr}(\mathbf{P}^T\mathbf{D}^{\prime}).
            \end{equation}
            Then, the LSAP costs associated with $\mathbf{P}^{(1)}$ and $\mathbf{P}^{(2)}$ satisfy:
            \begin{equation}\label{eq:clean_bounds}
                \frac{1}{nR_d^2}\operatorname{Tr}((\mathbf{P}^{(1)})^T\mathbf{D}) \leq c \cdot n^{-\alpha}, \quad \text{and} \quad \frac{1}{nR_d^2}\operatorname{Tr}((\mathbf{P}^{(2)})^T\mathbf{D}) \leq c \cdot n^{-\alpha},
            \end{equation}
            where $c$ and $\alpha$ are the constants specified in Assumption~\ref{assump:generalized_cluster}.
        \end{lemma}

        We note that, with the masking matrix $\mathbf{P}$ set to $\widetilde{\mathbf{P}}^{(1)}$ from~\eqref{eq:alg_LSAP_1}, Lemma~\ref{lem:Geometric LSAP Cost} and
        \eqref{eq:squared_distance_cost} allow us to bound the LSAP costs of the neighbor assignment steps~\eqref{eq:alg_LSAP_1} and~\eqref{eq:alg_LSAP_2} in Algorithm~\ref{alg:debias}.
        Such analysis of the LSAP cost under a noisy cost matrix underpins
        our main result---Theorem~\ref{thm:estimation_error_bound} below---which establishes recovery guarantees for our approach.

        \begin{theorem} \label{thm:estimation_error_bound}
            Given a corrupted dataset $\mathbf{Y} = \{\mathbf{y}_i\}_{i=1}^n \subset \mathbb{R}^d$ as in~\eqref{eq:noise_model}, under Assumptions~\ref{assump:noise} and~\ref{assump:generalized_cluster}, there exist constants $t_0, d_0, n_0, C^{\prime}, C^{\prime\prime}> 0$ such that for any feature dimension $d > d_0$, any dataset size $n > n_0$, and any $t > t_0$, the estimates $\hat{\mathbf{r}} \in \mathbb{R}^n$ in~\eqref{eq:noise_mag_estimate} and $\hat{\mathbf{D}} \in \mathbb{R}^{n\times n}$ in~\eqref{eq:correction} obtained by executing Algorithm~\ref{alg:debias} on the dataset $\mathbf{Y}$ satisfy:
            \begin{align}
                \frac{1}{nR_d^2} \|\hat{\mathbf{r}} - \mathbf{r}\|_1 &\leq C^{\prime}\left(t\mathcal{E}(d) + n^{-\alpha}\right), \label{eq:estimation_L1_bound_noise} \\
                \frac{1}{n(n-1)R_d^2} \sum_{i=1}^n \sum_{j\neq i}^n |
                \hat{D}_{ij} - D_{ij}| &\leq C^{\prime\prime}\left(t\mathcal{E}(d) + n^{-\alpha}\right), \label{eq:estimation_L1_bound_distance}
            \end{align}
            with probability at least $1-n(n-1)d^{-t}$. Here, $\mathcal{E}(d) := \sqrt{\log d} \cdot \max\{E, E^2\sqrt{d}\}$ with the relative noise scale $E$ defined in Assumption~\ref{assump:noise}, $\mathbf{r} = [r_1, r_2, \ldots, r_n]^T$ is the true noise magnitude in~\eqref{eq:squared_distance_cost}, $\mathbf{D} \in \mathbb{R}^{n\times n}$ is the true distance matrix in~\eqref{eq:dist_mat}, and $\alpha$ is defined in Assumption~\ref{assump:generalized_cluster}.
        \end{theorem}

        In Theorem~\ref{thm:estimation_error_bound}, $t_0$, $d_0$, $n_0$, and $\alpha$ are fixed constants and constants $C^{\prime}$ and $C^{\prime\prime}$ may depend on them, but not on $n$, $d$, and $R_d$. Theorem~\ref{thm:estimation_error_bound} establishes theoretical guarantees for the recovery accuracy of $\hat{\mathbf{r}}$ and $\hat{\mathbf{D}}$ relative to the squared clean-data scale $R_d^2$. Specifically, it states that the normalized $\ell_1$ recovery errors for $\hat{\mathbf{r}}$ and $\hat{\mathbf{D}}$ are bounded by $ \mathcal{O}\left(t\mathcal{E}(d) + n^{-\alpha}\right)$ with probability at least $1-n(n-1)d^{-t}$. Under Assumption~\ref{assump:noise}, we have $\mathcal{E}(d) \leq \max\{C d^{-1/4}, C^2\left(\log d\right)^{-1/2}\}$, which tends to zero as $d$ increases. Consequently, for any fixed $t>t_0$, the error bound converges to zero as $d,n\rightarrow \infty$, and the probability converges to 1 whenever $n(n-1)d^{-t}\rightarrow0$. In particular, if $d\geq n^\gamma$ for some constant $\gamma>0$, then for any $s>0$, choosing a fixed $t>\max\{t_0,(s+2)/\gamma\}$ yields probability at least $1-n^{-s}$ for all sufficiently large $n$.

        \begin{corollary}\label{cor:kernel_operator_norm}
            Under the conditions of Theorem~\ref{thm:estimation_error_bound}, define $\hat D^+_{ij}:=\max\{\hat D_{ij},0\}, \forall i\neq j.$
            For a deterministic bandwidth $h>0$, possibly depending on $d$ and $n$, define $\mathbf K,\hat{\mathbf K}\in\mathbb R^{n\times n}$ by
            \begin{equation}\label{eq:fixed_bandwidth_kernels}
                K_{ii}=\hat K_{ii}=0,\qquad
                K_{ij}=\exp\left(-\frac{D_{ij}}{h}\right),\quad
                \hat K_{ij}=\exp\left(-\frac{\hat D^+_{ij}}{h}\right),\quad i\neq j.
            \end{equation}
            Then, with probability at least $1-n(n-1)d^{-t}$,
            \begin{equation}\label{eq:kernel_operator_bound}
                \left\|\frac{\hat{\mathbf K}-\mathbf K}{n}\right\|_{\mathrm{op}}
                \leq
                \left[\frac{n-1}{n}\frac{C^{\prime\prime}R_d^2}{h}
                \left(t\mathcal E(d)+n^{-\alpha}\right)\right]^{1/2}.
            \end{equation}
        \end{corollary}

        Corollary~\ref{cor:kernel_operator_norm} quantifies how accurately the fixed-bandwidth Gaussian affinity matrix constructed from the corrected distances approximates its clean counterpart in operator norm. In particular, for any fixed $t>t_0$, if $n(n-1)d^{-t}\rightarrow0$ and $\frac{R_d^2}{h}\left(t\mathcal E(d)+n^{-\alpha}\right)\rightarrow0$,
        then $\|\hat{\mathbf K}/n-\mathbf K/n\|_{\mathrm{op}}\rightarrow0$ in probability. The operator-norm bound can be combined directly with Weyl's inequality and the Davis--Kahan theorem to obtain corresponding eigenvalue and eigenspace perturbation bounds, under a suitable eigengap condition for the latter.

    \section{Experiments}\label{sec:experiment}

        \subsection{Numerical validation of main theoretical claims}\label{sec:demonstrating_example}
            In this section, we illustrate the theoretical results in Section~\ref{sec:theoretical guarantee} through numerical simulations. A comparison with the distance estimator based on doubly stochastic normalization~\cite{boris2} is given in Appendix~\ref{sec:sim1}.

            We begin by examining a clean dataset $\mathbf{X} = \{\mathbf{x}_i\}_{i=1}^n \subset \mathbb{R}^d$, whose generative model falls within the framework of Corollary~\ref{col:manifold}, and demonstrate how the geometric properties of $\mathbf{X}$ relate to the LSAP costs, as established in Lemma~\ref{lem:Geometric LSAP Cost}. The data are generated by sampling points from a mixture of a solid ball $\mathcal{B}$ and a concentric planar circle $\mathcal{C}$ surrounding it in $\mathbb{R}^3$, and embedding them into $\mathbb{R}^d$ via a random orthogonal transformation. Figure~\ref{fig:combined_theory}(a) illustrates an example before embedding, and the detailed construction is given in Appendix~\ref{sec:sim1}.

            \begin{figure}[!ht]
                \centering
                \includegraphics[width= 0.95 \textwidth]{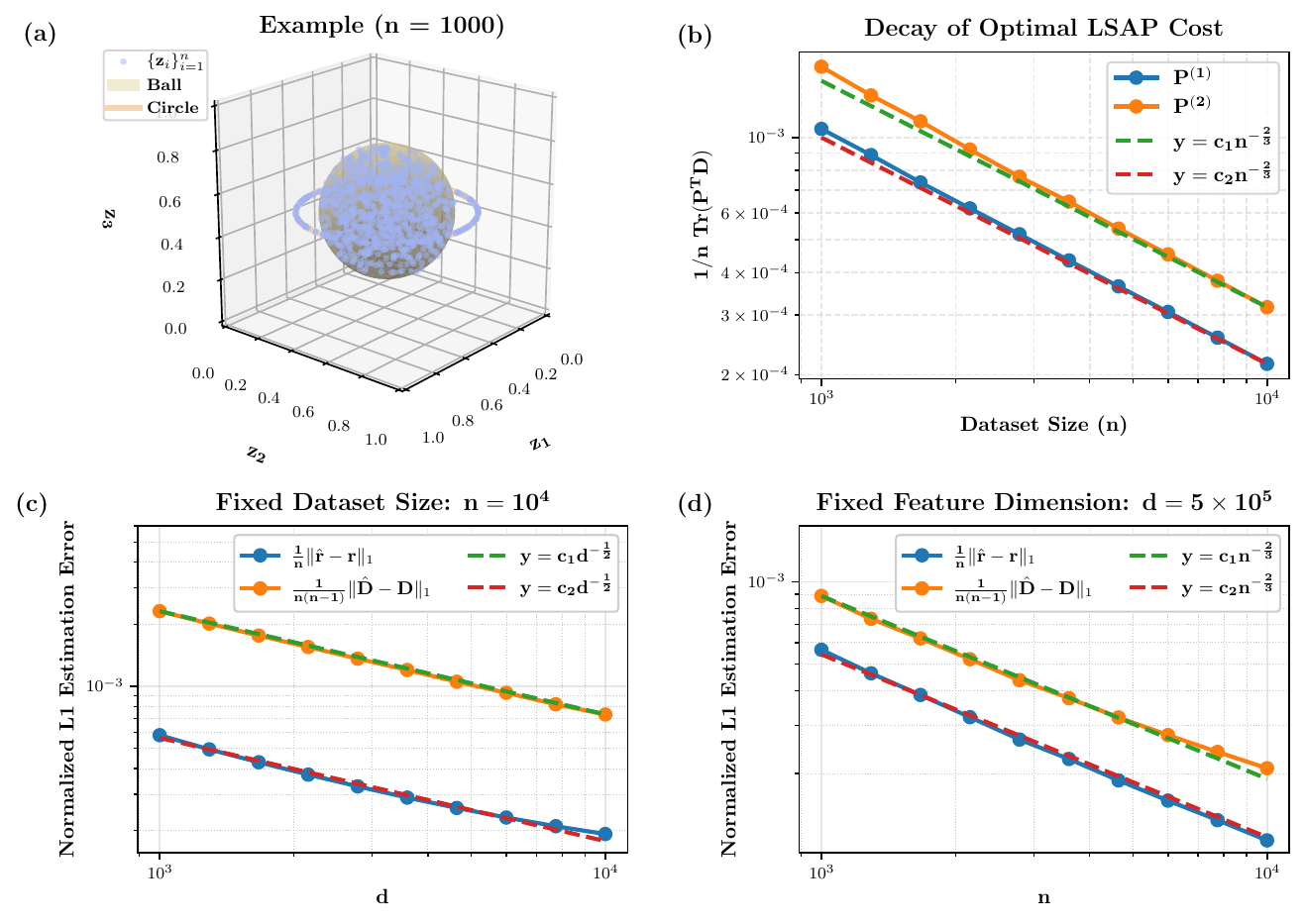}
                \caption{Illustration of the simulated data and empirical validation of the theoretical results. (a) Simulated data $\mathbf{Z}$ with $n = 10^3$. (b) Log-log plot of empirical LSAP costs associated with optimal solutions $\mathbf{P}^{(1)}$ and $\mathbf{P}^{(2)}$ from~\eqref{eq:lemma_lsap} as a function of the data size $n$. (c) Normalized $\ell_1$ recovery errors of Algorithm~\ref{alg:debias} as a function of the feature dimension $d$ with fixed dataset size $n = 10^4$. (d) Normalized $\ell_1$ recovery errors as a function of the dataset size $n$ with fixed feature dimension $d = 5\times10^5$. In (b)--(d), each point represents the average over 10 independent trials.}
                \label{fig:combined_theory}
            \end{figure}
            We first examine the relationship between the geometry of $\mathbf{X}$ and the LSAP costs associated with~\eqref{eq:lemma_lsap}. Applying  Lemma~\ref{lem:Geometric LSAP Cost} to our case, where we generate $\mathbf{D}^{\prime}$ based on $\mathbf{P}^{(1)}$ from~\eqref{eq:lemma_lsap} following~\eqref{eq:masked_D_def}, we expect the LSAP costs associated with $\mathbf{P}^{(1)}$ and $\mathbf{P}^{(2)}$ to decay polynomially as the dataset size $n$ increases. Figure~\ref{fig:combined_theory}(b) displays the LSAP costs as a function of $n$, where the costs decay polynomially as $n$ increases, with empirical decay rates of $n^{-2/3}$. We note that this empirical behavior aligns with Theorem~\ref{thm:partition} and Corollary~\ref{col:manifold}. Specifically, $\mathcal{B}$ and $\mathcal{C}$ have intrinsic dimensions $m_1 = 3$ and $m_2 = 1$, respectively. Applying the predicted rate $\alpha = \min\{2/m,\,1\}$ to each component separately yields decay rates of $n^{-2/3}$ for $\mathcal{B}$ and $n^{-1}$ for $\mathcal{C}$. By the discussion following Corollary~\ref{col:manifold}, the slower-decaying term $n^{-2/3}$ dominates the overall behavior, matching the empirical rate in Figure~\ref{fig:combined_theory}(b).

            We next examine the corrupted dataset $\mathbf{Y} = \{\mathbf{y}_i\}_{i=1}^n \subset \mathbb{R}^d$ in~\eqref{eq:noise_model} and demonstrate numerically how the recovery errors of the noise magnitudes $\{\hat{r}_i\}_{i=1}^n$ from~\eqref{eq:noise_mag_estimate} and the corrected distance matrix $\hat{\mathbf{D}}$ from~\eqref{eq:correction} scale with the dataset size $n$ and the feature dimension $d$. Specifically, we consider Gaussian noise $\boldsymbol{\eta}_i \sim \mathcal{N}(\mathbf{0}, \boldsymbol{\Sigma}_i)$ with $\boldsymbol{\Sigma}_i = \operatorname{diag}(\tau_i \delta_1, \ldots, \tau_i \delta_d)/d$, where $\tau_i, \delta_k \sim \operatorname{Uniform}[0.01, 0.15]$ encode sample- and feature-specific noise levels. This noise satisfies Assumption~\ref{assump:noise}, and its magnitude $\mathbb{E}\|\boldsymbol{\eta}_i\|_2^2 \in [10^{-4}, 2.25\times 10^{-2}]$ varies across samples through $\boldsymbol{\Sigma}_i$. Applying Theorem~\ref{thm:estimation_error_bound}, we expect $\frac{1}{n} \|\hat{\mathbf{r}} - \mathbf{r}\|_1$ in~\eqref{eq:estimation_L1_bound_noise} and $\frac{1}{n(n-1)} \sum_{i=1}^n \sum_{j\neq i}^n |\hat{D}_{ij} - D_{ij}|$ in~\eqref{eq:estimation_L1_bound_distance} to be bounded by $\mathcal{O}\left(\sqrt{\log d} \cdot d^{-1/2} + n^{-2/3}\right)$ with high probability, where the term $\sqrt{\log d} \cdot d^{-1/2}$ corresponds to $\mathcal{E}(d)$ and $n^{-2/3}$ is observed in Figure~\ref{fig:combined_theory}(b).

            Figure~\ref{fig:combined_theory}(c) and (d) illustrate how the recovery errors $\frac{1}{n(n-1)} \sum_{i=1}^n \sum_{j\neq i}^n |\hat{D}_{ij} - D_{ij}|$ and $\frac{1}{n} \|\hat{\mathbf{r}} - \mathbf{r}\|_1$ scale with $d$ and $n$. To isolate the influence of $d$, Figure~\ref{fig:combined_theory}(c) examines the regime where $ \mathcal{E}(d) \gg n^{-2/3}$ by fixing $n$ to be sufficiently large ($n = 10^4$) while varying $d$. The observed decay rates approximate $d^{-1/2}$, closely matching the theoretical rate of $\mathcal{E}(d)$. Figure~\ref{fig:combined_theory}(d) examines the dependency of the error bounds on $n$ by considering the regime where $ n^{-2/3} \gg \mathcal{E}(d)$. Fixing $d$ to be large ($d = 5\times 10^5$) while varying $n$ reveals the expected rate of $n^{-2/3}$.

        \subsection{Robust distance recovery and graph construction under data-dependent heterogeneous noise}\label{sec:uniform_density}

            In this example, we apply our approach to a simulated dataset $\mathbf{Y}$ corrupted by heterogeneous noise whose magnitudes depend on the clean dataset $\mathbf{X}$. We show that our approach accurately estimates the varying noise magnitudes and pairwise distances, enabling assessment of local signal-to-noise ratios and significantly enhancing the noise robustness of Gaussian kernel matrices and KNN graphs constructed from the corrected distances.

            \begin{figure}[!ht]
                \centering
                \includegraphics[width=\textwidth]{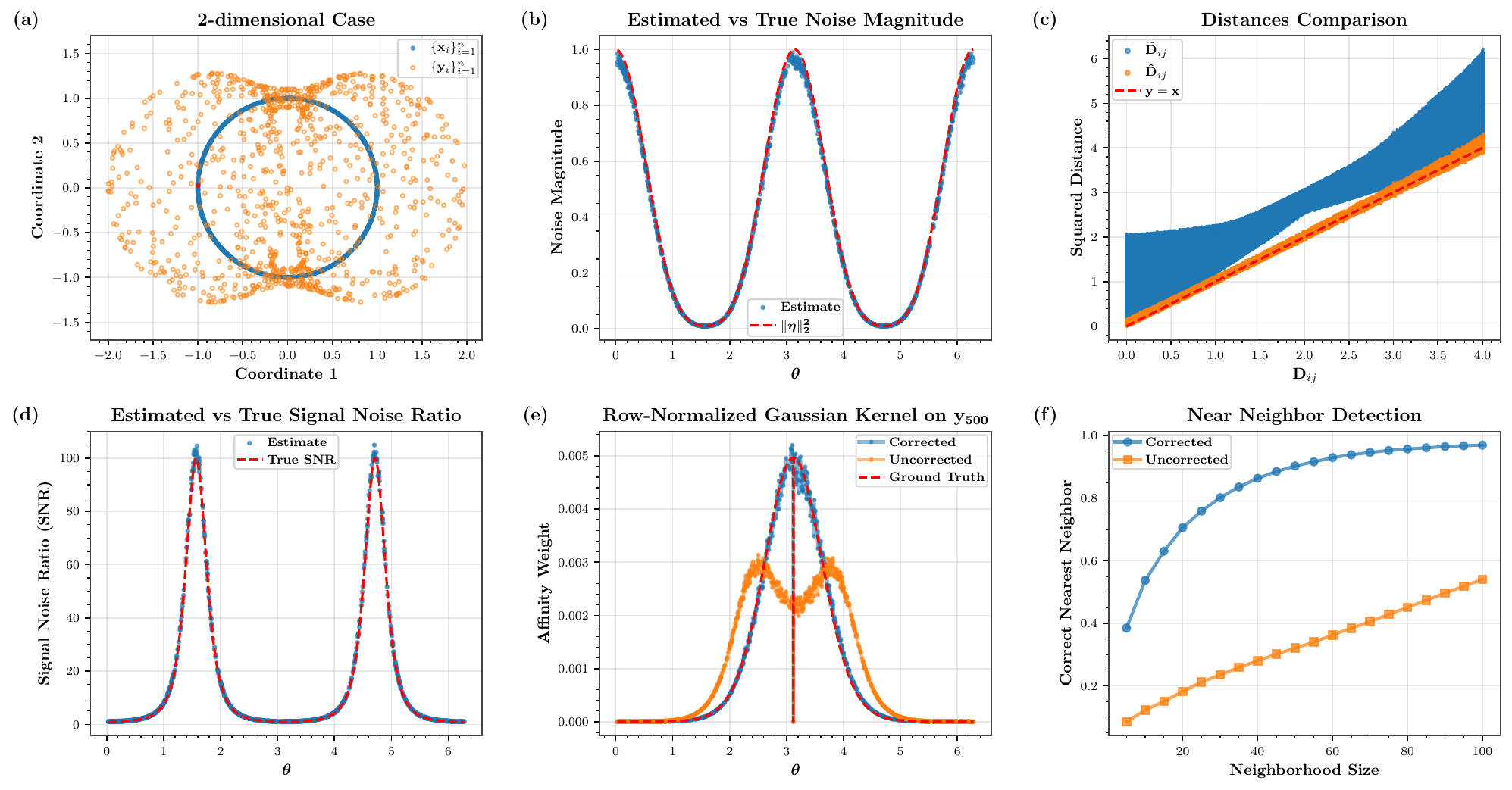}
                \caption{Results of applying Algorithm~\ref{alg:debias} to simulated data corrupted by varying noise. (a) Illustration of simulated data for $n=10^3$. (b) Estimated noise magnitudes $\hat{\mathbf{r}}$ from~\eqref{eq:noise_mag_estimate} and the ground truth $\{{r}_i\}_{i=1}^n$ from~\eqref{eq:squared_distance_cost} as functions of the angles $\theta$ on the circle. (c) Comparison of the clean distances $\mathbf{D}$ from~\eqref{eq:dist_mat}, the corrupted distances $\widetilde{\mathbf{D}}$ from~\eqref{eq:corrupted_dist_mat}, and the corrected distances $\hat{\mathbf{D}}$ from~\eqref{eq:correction}. The diagonal elements (self-distances) are excluded. (d) Comparison of $\widehat{\operatorname{SNR}}$ estimates from~\eqref{eq:snr} with ground truth. (e) Similarity weights for $\mathbf{y}_{500}$: comparison of the $500$th row of row-stochastic Gaussian kernel matrices $\mathbf{W}$, $\widetilde{\mathbf{W}}$, and $\hat{\mathbf{W}}$ constructed from $\mathbf{D}$, $\widetilde{\mathbf{D}}$, and $\hat{\mathbf{D}}$ following~\eqref{eq:rs_Gaussian}. (f) Accuracy of nearest neighbor identification as a function of the neighborhood size $k$ using $\widetilde{\mathbf{D}}$ and $\hat{\mathbf{D}}$.}
                \label{fig:unit_circle}
            \end{figure}

            The experiment involves a clean dataset $\mathbf{X}$
            of sample size $n = 10^3$ and feature dimension $d = 10^4$, where each $\mathbf{x}_i \in \mathbf{X}$ is sampled uniformly from a unit circle in $\mathbb{R}^2$ and embedded into $\mathbb{R}^{d}$ via a random orthogonal transformation.
            The heterogeneous noise vectors $\{\boldsymbol{\eta}_i\}_{i=1}^n$ are generated by sampling random vectors $\{\mathbf{u}_i\}_{i=1}^n$ uniformly from the unit sphere $\mathbb{S}^{d-1} \subset \mathbb{R}^d$, and scaling their magnitudes by a function of $\mathbf{X}$. Concretely, $\{\boldsymbol{\eta}_i\}_{i=1}^n$ are generated as:
            \begin{equation}\label{eq:geometry_noise}
            \mathbf{u}_i \sim \mathbb{S}^{d-1}, \quad g(\theta_i, \varphi)= 0.1 + 0.9 \cdot \frac{1+\cos(2\theta_i + \varphi)}{2}, \quad \boldsymbol{\eta}_i = g(\theta_i, \varphi) \cdot \mathbf{u}_i.
            \end{equation}
            The process in~\eqref{eq:geometry_noise} creates a heterogeneous pattern where the noise magnitudes vary smoothly from $0.1$ to $1$ along the underlying circle. We set $\varphi = 0$ in~\eqref{eq:geometry_noise} and create the corrupted dataset $\mathbf{Y}$ following~\eqref{eq:noise_model}. See Figure~\ref{fig:unit_circle}(a) for examples of $\mathbf{X}$ and $\mathbf{Y}$ for $d = 2$.

            We first evaluate the performance of our approach on noise magnitude estimation and distance correction. Figure~\ref{fig:unit_circle}(b) compares
            $\hat{\mathbf{r}}$ from~\eqref{eq:noise_mag_estimate} against the ground truth calculated from~\eqref{eq:geometry_noise}, where we plot both quantities against the parameters $\{\theta_i\}_{i=1}^n$ used in data generation. The estimated noise magnitudes closely match the ground truth across the entire range of $\{\theta_i\}_{i=1}^n$. This accurate estimation of noise magnitudes enables precise distance correction. In Figure~\ref{fig:unit_circle}(c), we compare three distance matrices: the clean distance matrix $\mathbf{D}$ from~\eqref{eq:dist_mat}, the corrupted distance matrix $\widetilde{\mathbf{D}}$ from~\eqref{eq:corrupted_dist_mat}, and the corrected distance matrix $\hat{\mathbf{D}}$ from~\eqref{eq:correction}. As evident in Figure~\ref{fig:unit_circle}(c), the corrupted distances in $\widetilde{\mathbf{D}}$ exhibit substantial non-linear deviations from their clean counterparts in $\mathbf{D}$, illustrating the distortion introduced by the heterogeneous noise (see~\eqref{eq:squared_distance_cost}). In contrast, $\hat{\mathbf{D}}$ aligns closely around $\mathbf{D}$ with small variation. This pronounced improvement validates our approach's capacity to accurately reconstruct the underlying geometric relationship between data points, even in the presence of heterogeneous noise whose magnitude depends on the clean data.

            Next, we demonstrate that the noise magnitude estimates $\hat{\mathbf{r}}$ can be utilized to assess the signal-to-noise ratio (SNR) for each data observation, providing valuable insights into data quality and reliability. Specifically, under the noise model in~\eqref{eq:noise_model}, for each $\mathbf{x}_i \in \mathbf{X}$, we define an estimator $\hat{s}_i$ for  $\|\mathbf{x}_i\|_2^2$ as $\hat{s}_i: = \|\mathbf{y}_i\|_2^2 - \hat{r}_i$, and define the SNR estimator as:
            \begin{equation}\label{eq:snr}
            \widehat{\operatorname{SNR}}(\mathbf{y}_i)
            : = \frac{\hat{s}_i}{\hat{r}_i}
            = \frac{\|\mathbf{y}_i\|_2^2 - \hat{r}_i}{\hat{r}_i}.
            \end{equation}
            Figure~\ref{fig:unit_circle}(d) shows that our SNR estimates closely match the ground truth, clearly revealing regions of degraded data quality around $\theta = 0,\pi$, and $2\pi$ as designed in~\eqref{eq:geometry_noise}.

            We proceed to show that the corrected distance matrix $\hat{\mathbf{D}}$ from~\eqref{eq:correction} can improve the fidelity of distance-based computations, focusing on Gaussian kernel-based similarity measurements and KNN graphs. For pairwise data similarity assessment, we employ a row-stochastic Gaussian kernel. Formally, given a distance matrix $\mathbf{D}$ from~\eqref{eq:dist_mat} and a bandwidth parameter $\sigma$, we construct the similarity matrix $\mathbf{W}$ based on the row-stochastic Gaussian kernel as:
            \begin{equation}\label{eq:rs_Gaussian}
                \mathbf{W} = \operatorname{diag}(\mathbf{d})^{-1}\mathbf{K}, \quad \mathbf{d} = \mathbf{K} \mathbf{1}_n,
                \quad
                \mathbf{K} = \exp{\left(-\frac{\mathbf{D}}{\sigma^2}\right)}.
            \end{equation}
            Here, $\mathbf{1}_n \in \mathbb{R}^n$ is the all-one vector, $\mathbf{d} \in \mathbb{R}^n$ is the degree vector containing the row sums of $\mathbf{K}$, $\operatorname{diag}(\cdot)$ transforms a vector into a diagonal matrix, and $\exp(\cdot)$ is applied entry-wise.

            In Figure~\ref{fig:unit_circle}(e), we compare three similarity matrices: $\mathbf{W}$, $\widetilde{\mathbf{W}}$, and $\hat{\mathbf{W}}$ derived from $\mathbf{D}$, $\widetilde{\mathbf{D}}$ and $\hat{\mathbf{D}}$ respectively, following~\eqref{eq:rs_Gaussian}, with the bandwidth set to $\sigma^2 = 0.5$. Specifically, we focus on the similarity relationships between a representative data point $\mathbf{y}_{500}$ and all other data points. We select $\mathbf{y}_{500}$ because $\theta_{500} \approx \pi$, which, by the noise design in~\eqref{eq:geometry_noise}, makes it the most severely corrupted point and consequently the most challenging case for accurate similarity relationship reconstruction.
            Moreover, since points with $\theta_j$ closer to $\theta_{500}$ are closer to $\mathbf{x}_{500}$ and should receive higher similarity, sorting $\{\theta_i\}_{i=1}^n$ in ascending order should reveal a Gaussian-shaped similarity curve centered at $\theta_{500}$. As shown in Figure~\ref{fig:unit_circle}(e), the similarities between $\mathbf{y}_{500}$ and other points captured by $\hat{\mathbf{W}}$ closely track the ground truth in $\mathbf{W}$, whereas those from $\widetilde{\mathbf{W}}$ deviate significantly, exhibiting an unusual bimodal shape that is
            influenced by the heterogeneous noise. This experiment confirms that the corrected distances from our approach can effectively render Gaussian-kernel matrices robust to heterogeneous noise.

            For a more direct assessment of how effectively $\hat{\mathbf{D}}$ preserves local neighborhood structure, we compare KNN graphs constructed from both $\widetilde{\mathbf{D}}$ and $\hat{\mathbf{D}}$, assessing the accuracy of nearest neighbor identification by computing the percentage of correctly identified neighbors (i.e., overlap with the ground truth nearest neighbors identified using $\mathbf{D}$). As illustrated in Figure~\ref{fig:unit_circle}(f), KNN graphs constructed using our corrected distance matrix $\hat{\mathbf{D}}$ demonstrate substantially higher accuracy across all examined neighborhood sizes, achieving more than three-fold improvement for $k \leq 40$ and an approximately two-fold improvement for $k \in (40,100]$, in comparison with those constructed using the uncorrected matrix $\widetilde{\mathbf{D}}$. This improvement in preserving local neighborhood structure demonstrates our approach's advantage in promoting robustness to heterogeneous noise in distance-based computations.

        \subsection{Self-tuning kernels with corrected distances are robust to heterogeneous  noise}\label{sec:self_tune_eig}

            In this example, we apply our approach to a dataset $\mathbf{Y}$ with non-uniform sampling density and demonstrate that the corrected distances from~\eqref{eq:correction} can enhance the robustness of density-adaptive kernels, such as the self-tuning kernel~\cite{self_tune}. Specifically, we show that when constructing a Laplacian matrix using the self-tuning kernel applied to our corrected distances, its leading eigenvectors (corresponding to eigenvalues sorted in ascending order) remain robust even in the presence of heterogeneous noise whose noise magnitudes depend on the clean data. The robustness of leading eigenvectors is crucial for spectral clustering~\cite{spectral_clustering} and dimensionality reduction methods~\cite{laplacian_eigenmap}, as the performance of these techniques depends on the fidelity of the eigenvector representation of the underlying data geometry.

            For this example, we consider a clean dataset $\mathbf{X}$ with $n = 4 \times 10^3$ and $d = 10^4$, obtained by sampling points from two circles of different sizes and embedding them into $\mathbb{R}^d$ via a random orthogonal transformation. Half of the points lie on a large circle $\mathcal{C}_l$ with non-uniform angular density, and the rest on an adjacent small circle $\mathcal{C}_s$ with uniform density. The noise vectors $\{\boldsymbol{\eta}_i\}_{i=1}^n$ follow the pattern of~\eqref{eq:geometry_noise}, with magnitudes varying smoothly along each circle and scaled in proportion to its radius. The corrupted dataset $\mathbf{Y}$ is formed following~\eqref{eq:noise_model}. See Figure~\ref{fig:two_circle}(a) for an example of $\mathbf{X}$ and $\mathbf{Y}$ with $d = 2$ and Appendix~\ref{sec:sim3} for simulation details.

            \begin{figure}[!ht]
                \centering
                \includegraphics[width= 0.95\textwidth]{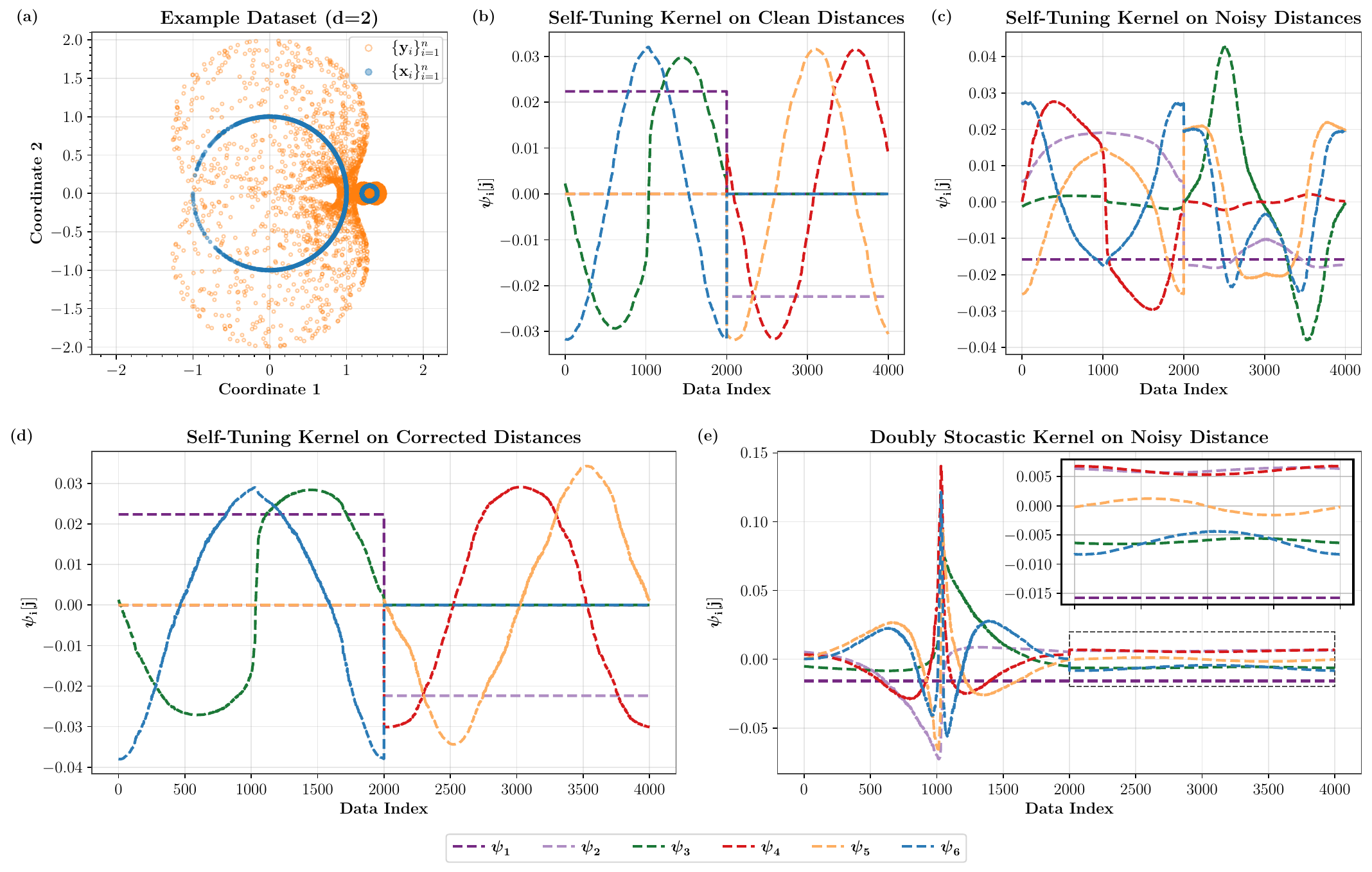}
                \caption{Self-tuning kernels with the corrected distances reveal the clean geometry of a heterogeneous noise-corrupted dataset. (a) Illustration of simulated data for $n=4\times 10^3$. (b)–(d) Leading eigenvectors of self-tuning Laplacian matrices $\mathbf{L}$, $\widetilde{\mathbf{L}}$, and $\hat{\mathbf{L}}$ constructed from $\mathbf{D}$, $\widetilde{\mathbf{D}}$, and $\hat{\mathbf{D}}$ respectively, according to~\eqref{eq:lap}, with bandwidth set as the distance to the 150th nearest neighbor. (e) Leading eigenvectors of a doubly stochastic Laplacian matrix constructed from the Gaussian kernel on the corrupted distances $\widetilde{\mathbf{D}}$ with bandwidth $\sigma^2 = 0.5$. }
                \label{fig:two_circle}
            \end{figure}

            The self-tuning kernel~\cite{self_tune} is a widely used approach for handling varying sampling density
            which replaces the single global bandwidth of the Gaussian kernel with data-specific bandwidths. Given the squared distance matrix $\mathbf{D}$ in~\eqref{eq:dist_mat}, the self-tuning affinity matrix $\mathbf{K}$ is\begin{equation}\label{eq:self-tune}
                K_{ij} = \exp{\left(-\frac{D_{ij}}{\sigma_i\sigma_j}\right)}, \quad \forall i,j \in [n],
            \end{equation}
            where $\sigma_i$ is an adaptive bandwidth reflecting the local density around $\mathbf{x}_i$, typically set to the distance from $\mathbf{x}_i$ to its $k$th nearest neighbor.
            A popular normalization for the self-tuning kernel matrix $\mathbf{K}$ is symmetric normalization, where the normalized matrix $\mathbf{W}$ is defined as:
            \begin{equation}\label{eq:symmetric_normalization}
                \mathbf{W} = \operatorname{diag}(\mathbf{d})^{-1/2}\mathbf{K}\operatorname{diag}(\mathbf{d})^{-1/2}, \quad \mathbf{d} = \mathbf{K}\mathbf{1}_n.
            \end{equation}

            In this experiment, we compare three symmetrically-normalized similarity matrices: $\mathbf{W}$, $\widetilde{\mathbf{W}}$, and $\hat{\mathbf{W}}$ constructed from $\mathbf{D}$ in~\eqref{eq:dist_mat}, $\widetilde{\mathbf{D}}$ in~\eqref{eq:corrupted_dist_mat}, and $\hat{\mathbf{D}}$ in~\eqref{eq:correction}, respectively. All three matrices are computed following~\eqref{eq:self-tune} and~\eqref{eq:symmetric_normalization}. For each data point, we set the bandwidth as its distance to its 150th nearest neighbor in the corresponding distance matrix.

            To evaluate how effectively these similarity matrices capture the underlying data geometry, we examine the leading eigenvectors of the Laplacian matrix. The spectrum of the Laplacian matrix is particularly illuminating as the leading eigenvectors encode the intrinsic data geometry and provide meaningful low-dimensional embeddings~\cite{laplacian_eigenmap}. Given a similarity matrix $\mathbf{W}$ defined in~\eqref{eq:symmetric_normalization}, the Laplacian matrix is constructed as:
            \begin{equation}\label{eq:lap}
                \mathbf{L} = \operatorname{diag}(\mathbf{d}) - \mathbf{W}, \quad \mathbf{d} = \mathbf{W} \mathbf{1}_n.
            \end{equation}
            We construct $\mathbf{L}$, $\widetilde{\mathbf{L}}$, and $\hat{\mathbf{L}}$ from $\mathbf{W}$, $\widetilde{\mathbf{W}}$, and $\hat{\mathbf{W}}$ respectively, all following~\eqref{eq:lap}.

            In Figure~\ref{fig:two_circle}(b), the first two eigenvectors $\boldsymbol{\psi}_1$ and $\boldsymbol{\psi}_2$ of $\mathbf{L}$ exhibit distinct piecewise constant behaviors.
            Specifically, $\boldsymbol{\psi}_1$ is a non-zero constant on $\mathcal{C}_l$ and zero on $\mathcal{C}_s$, while $\boldsymbol{\psi}_2$ is a non-zero constant on $\mathcal{C}_s$ and zero on $\mathcal{C}_l$, indicating that $\mathbf{W}$ correctly identifies $\mathcal{C}_l$ and $\mathcal{C}_s$ as disconnected. The subsequent eigenvectors ($\boldsymbol{\psi}_3$ through $\boldsymbol{\psi}_6$) exhibit oscillatory patterns resembling sine and cosine waves on
            $\mathcal{C}_l$ and $\mathcal{C}_s$ respectively. These patterns are expected as the eigenvectors of the graph Laplacian approximate the eigenfunctions of the Laplace-Beltrami operator, which are sine and cosine functions with different frequencies for the circular geometry of $\mathcal{C}_l$ and $\mathcal{C}_s$.

            In contrast, as depicted in Figure~\ref{fig:two_circle}(c), the eigenvectors of $\widetilde{\mathbf{L}}$ fail to capture the correct geometry. Specifically, the second eigenvector $\widetilde{\boldsymbol{\psi}}_2$ exhibits irregular patterns instead of maintaining the expected piecewise-constant property, erroneously suggesting connectivity between $\mathcal{C}_l$ and $\mathcal{C}_s$. This spurious connectivity arises because heterogeneous noise inflates distances non-uniformly (see~\eqref{eq:squared_distance_cost}), so that some true neighbors on the same circle appear farther apart than points on the other circle,
            creating non-negligible weights in $\widetilde{\mathbf{W}}$ that incorrectly connect $\mathcal{C}_l$ and $\mathcal{C}_s$. The subsequent eigenvectors also deviate significantly from the expected oscillatory patterns of sine and cosine waves, failing to reflect the circular geometry inherent in $\mathbf{X}$.

            Figure~\ref{fig:two_circle}(d) demonstrates that the leading eigenvectors from $\hat{\mathbf{L}}$ closely resemble those obtained from the clean dataset (Figure~\ref{fig:two_circle}(b)). The first two eigenvectors maintain the piecewise-constant property that distinguishes $\mathcal{C}_l$ and $\mathcal{C}_s$. Additionally, the subsequent four eigenvectors exhibit the expected oscillatory patterns resembling sine and cosine waves, accurately reflecting the circular geometry. This result confirms that the distance correction performed by Algorithm~\ref{alg:debias} successfully recovers the true geometric relationships from the corrupted data, thereby conferring robustness to the self-tuning kernel against heterogeneous noise.

            For additional comparison, we compute the Gaussian kernel matrix from the corrupted distances $\widetilde{\mathbf{D}}$ using a fixed bandwidth $\sigma^2 = 0.5$ and apply the Sinkhorn--Knopp algorithm~\cite{sinkhorn} for doubly-stochastic normalization. As discussed in Section~\ref{sec:related_work}, the Gaussian kernel with doubly stochastic normalizations is robust to heterogeneous noise (pending a correct choice of the bandwidth parameter)~\cite{boris1}, but its performance degrades on datasets with non-uniform density. This limitation is illustrated in Figure~\ref{fig:two_circle}(e): while the leading eigenvectors of the resulting Laplacian exhibit correct oscillatory patterns on $\mathcal{C}_s$ (which has uniform density), unexpected localized peaks emerge on $\mathcal{C}_l$ (where the data density varies). Comparing with Figure~\ref{fig:two_circle}(d), we observe that self-tuning kernels combined with our corrected distances more accurately reveal the clean geometry of the data, particularly in regions of varying density.

        \subsection{Application to single-cell RNA sequencing data (scRNA-seq)}\label{sec:scRNAseq}

            scRNA-seq revolutionizes biological research by enabling genome-wide profiling of gene expression in individual cells~\cite{scRNAseq_intro}.
            It has transformed the discovery of novel cell types and the reconstruction of developmental trajectories~\cite{best_sc_practice}, both of which rely on accurately quantifying cell-cell similarity.

            A typical scRNA-seq dataset is a count matrix $\mathbf{Y} \in \mathbb{Z}_+^{n \times d}$, where $y_{ij}$ is the expression level of gene $j$ in cell $i$, with $n$ cells (typically $10^3$--$10^4$) and $d$ genes (typically on the order of $10^4$). A fundamental challenge in scRNA-seq analysis stems from heterogeneous noise~\cite{scrnaseq_poisson_model},
            which distorts the geometric relationships between cells and thereby compromises analyses that require accurate similarity measurements. We demonstrate that applying our approach as a preprocessing step mitigates these distortions, yielding more reliable quantification of cell-cell relationships. Additionally, our method provides accurate estimates of the varying noise levels across the cells, which can be beneficial for data quality assessment and control.

            In this experiment, we apply Algorithm~\ref{alg:debias} to a scRNA-seq dataset of peripheral blood mononuclear cells (PBMC) from~\cite{zheng2017pbmc}, which includes cell type annotations derived from flow cytometry and thus provides ground truth labels for evaluating computationally derived cell-cell relationships. We randomly sample $500$ cells from six cell types with distinct surface markers, yielding a downsampled dataset $\mathbf{Y} \in \mathbb{Z}_+^{n\times d}$ with $n = 3\times 10^3$ cells and $d = 32{,}738$ genes. See Appendix~\ref{sec:real} for details of the data and the rationale for downsampling, and Appendix~\ref{sec:runtime} for a runtime discussion. To account for variations in sequencing depths,
         for each cell $i$, we normalize its expression profile to obtain
            $\widetilde{\mathbf{y}}_i \in \mathbb{R}_+^d$ with
            $\widetilde{y}_{ij} = y_{ij}/\sum_{j'} y_{ij'}$, $\forall j \in [d]$. We define the noise as the deviation from the expected expression: $\boldsymbol{\eta}_i = \widetilde{\mathbf{y}}_i - \mathbb{E}\left[\widetilde{\mathbf{y}}_i\right]$.

            \begin{figure}[!ht]
                \centering
                \includegraphics[width=0.95\textwidth]{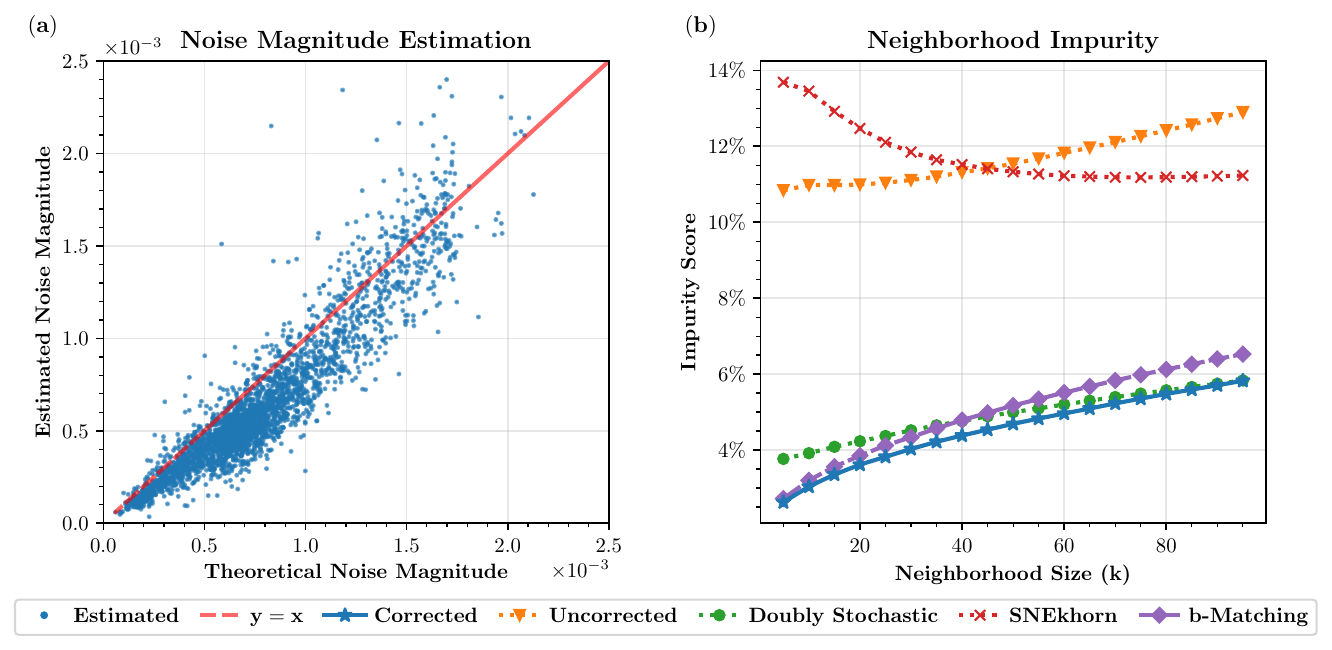}
                 \caption{Application to a scRNA-seq dataset of PBMC. (a) Comparison of estimated noise magnitudes $\hat{\mathbf{r}}$ from~\eqref{eq:noise_mag_estimate} and the theoretical predictions from the prototypical Poisson model. (b) Neighborhood impurity score (defined in~\eqref{eq:impurity}) for KNN graphs constructed using the corrected distances $\hat{\mathbf{D}}$, the uncorrected distances $\widetilde{\mathbf{D}}$, and the Gaussian kernel with doubly-stochastic normalization, b-matching, and SNEkhorn as a function of neighborhood size $k$. Each point represents the average over 10 independent downsampled datasets.}
                \label{fig:pbmc_neighborhood}
            \end{figure}

            We first demonstrate that our approach can accurately estimate noise magnitudes from the dataset $\widetilde{\mathbf{Y}} = \{\widetilde{\mathbf{y}}_i\}_{i=1}^n$ without requiring any prior information. Specifically, under a prototypical Poisson model~\cite{scrnaseq_poisson_model}, the noise magnitude satisfies  $\|\boldsymbol{\eta}_i\|_2^2 \approx 1/{\|\mathbf{y}_i\|_1}$~\cite{boris2}.  Figure~\ref{fig:pbmc_neighborhood}(a) shows that our estimates $\hat{\mathbf{r}}$ from~\eqref{eq:noise_mag_estimate} agree well with the Poisson model.

            Many scRNA-seq analyses rely on the similarity relationships captured by KNN graphs. Here, we demonstrate that KNN graphs constructed from $\hat{\mathbf{D}}$ defined in~\eqref{eq:correction} more accurately reflect the ground truth cell similarities. Given a KNN Graph $\mathcal{G}$ with vertex set $\mathbf{V}$ and node label $\mathbf{c}$, we quantify graph quality using the neighborhood impurity score:
            \begin{equation}\label{eq:impurity}
                \operatorname{\text{Impurity Score}}(\mathcal{G}, k) = \frac{\sum_{i\in \mathbf{V}} \sum_{j \in \operatorname{nbr}(i)} \mathbf{1}(c_i \neq c_j)}{k |\mathbf{V}|},
            \end{equation}
            where $\operatorname{nbr}(\cdot)$ gives the indices of a node's $k$-nearest neighbors and $c_i$ denotes the ground truth label for node $i$. This impurity score quantifies the fraction of nearest neighbors belonging to different cell types than the reference cell. Since cells of the same type usually cluster together, a lower impurity score indicates better preservation of the underlying biological structure.

            To assess graph quality, we compare KNN graphs constructed from five sources: the corrected distances $\hat{\mathbf{D}}$, the uncorrected distances $\widetilde{\mathbf{D}}$, the Gaussian kernel on $\widetilde{\mathbf{D}}$
            with doubly stochastic normalization~\cite{boris1}, b-matching~\cite{b_matching}, and SNEkhorn~\cite{snekhorn}; see Appendix~\ref{sec:real} for experimental details. Figure~\ref{fig:pbmc_neighborhood}(b) reports the impurity score of each method as a function of $k$. Across all neighborhood sizes, the corrected distances yield the lowest impurity, reducing the impurity of the uncorrected distances by roughly 75\% at $k = 5$ and 55\% at $k = 95$. Among the competing methods, b-matching nearly matches the corrected distances for the small neighborhood sizes, but its impurity grows faster with $k$. Moreover, it requires solving a separate matching problem for each value of $k$, each taking roughly four times the runtime of our approach, whereas our approach produces a single corrected distance matrix that serves all $k$; see Figure~\ref{fig:pbmc_runtime} in Appendix~\ref{sec:real}. The doubly stochastic Gaussian kernel performs well for large $k$, in line with its known robustness to heterogeneous noise~\cite{boris1}. But it requires choosing a kernel bandwidth, which  usually requires tuning, and it becomes numerically unstable when the bandwidth is small. SNEkhorn also demonstrates robustness to noise for $k \geq 45$, where its impurity falls below that of the uncorrected distances.

\section{Discussion}

    This work suggests several possible future research directions.
    It is of interest to derive
    recovery error bounds for $\hat{\mathbf{D}}$ in the $\ell_\infty$ norm, which could enable convergence analysis of density estimators and self-tuning kernels based on $\hat{\mathbf{D}}$. Furthermore, the guarantees in Theorem~\ref{thm:estimation_error_bound} could be extended beyond the sub-Gaussian noise class (Assumption~\ref{assump:noise}) to broader classes, such as sub-exponential noise. We leave such extensions to future work.

\section*{Acknowledgments}
The authors acknowledge the use of AI models (Claude and ChatGPT) for language polishing, proofreading, technical review of proofs and refinement of technical arguments. This work was supported by NIH grants UM1PA051410, U54AG076043, U54AG079759, U01DA053628, P50CA121974, and R33DA047037.

\begin{appendices}

\section{Details of the Motivating Example in Figure~\ref{fig:intro_ex}}\label{sec:intro_ex}
    The clean dataset $\mathbf{X}=\{\mathbf{x}_i\}_{i=1}^n\subset\mathbb{R}^d$  was generated according to
     \begin{equation}\label{eq:clean_circle}
                \theta_i \sim \operatorname{Uniform}[0, 2\pi],
                \quad \mathbf{x}_i = \mathbf{R}_d \cdot \begin{bmatrix} \cos(\theta_i) \\ \sin(\theta_i) \end{bmatrix},
            \end{equation}
    with $n = 10^3$ and $d = 10^4$. The corresponding corrupted dataset $\mathbf{Y}$ was obtained by adding noise drawn from a mixture of two regimes: independently for each $\mathbf{x}_i$, the noise is of \emph{outlier type} with probability $p = 0.2$ and \emph{geometry dependent} with probability $1 - p = 0.8$.

    The geometry-dependent noise is generated with a direction uniform on $\mathbb{S}^{d-1}$ and a magnitude varying smoothly with the underlying angular parameter $\theta_i$:
\begin{equation}\label{eq:mixed_structural_noise}
\begin{gathered}
    \mathbf{u}_i \sim \mathbb{S}^{d-1}, \quad \varphi_i = \begin{cases} 0, & \sin\theta_i \geq 0, \\ \pi, & \sin\theta_i < 0, \end{cases} \\
    \boldsymbol{\eta}_i = s_{\max}\left(0.1 + 0.9 \cdot \frac{1 + \cos(2\theta_i + \varphi_i)}{2}\right) \mathbf{u}_i,
\end{gathered}
\end{equation}
where $s_{\max} = 1$. The outlier-type noise instead has a direction uniform on $\mathbb{S}^{d-1}$ and a magnitude drawn independently of $\theta_i$ from a range exceeding $s_{\max}$:
\begin{equation}\label{eq:mixed_outlier_noise}
    \mathbf{u}_i \sim \mathbb{S}^{d-1}, \quad g_i \sim \operatorname{Uniform}(s_{\max}, \, 2 s_{\max}), \quad \boldsymbol{\eta}_i = g_i \cdot \mathbf{u}_i.
\end{equation}
The corrupted dataset $\mathbf{Y}$ is then formed following~\eqref{eq:noise_model}, using $\boldsymbol{\eta}_i$ from~\eqref{eq:mixed_structural_noise} or~\eqref{eq:mixed_outlier_noise} according to each point's assigned regime. The t-SNE affinity matrices and embeddings are then computed using the default settings and perplexity $=50$ on the corresponding distance matrices.

\section{Simulation Details of Section~\ref{sec:demonstrating_example}}\label{sec:sim1}

The clean dataset $\mathbf{X}$ is generated by first sampling $\mathbf{Z} = \{\mathbf{z}_i\}_{i=1}^n \subset \mathbb{R}^{3}$ and then embedding $\mathbf{Z}$ into $\mathbb{R}^d$ via a random orthogonal transformation. Each $\mathbf{z}_i \in \mathbf{Z}$ is sampled independently from a mixture model: with probability $0.8$, we sample uniformly from the ball
\begin{equation}\label{eq:clean_ball}
    \mathcal{B} = \{\mathbf{z} \in \mathbb{R}^3 : \|\mathbf{z} - [0.5,0.5,0.5]^{T}\|_2 \leq 0.3\},
\end{equation}
and with probability $0.2$, we sample uniformly from the circle
\begin{equation}\label{eq:clean_circle3d}
    \mathcal{C} = \{\mathbf{z} \in \mathbb{R}^3 : \|\mathbf{z} - [0.5,0.5,0.5]^{T}\|_2 = 0.4,\ z_3 = 0.5\}.
\end{equation}
We then embed each $\mathbf{z}_i$ into $\mathbb{R}^d$ to obtain $\mathbf{x}_i = \mathbf{R}_d \mathbf{z}_i$, where $\mathbf{R}_d \in \mathbb{R}^{d\times 3}$ has random orthonormal columns (i.e., $\mathbf{R}_d^T\mathbf{R}_d = \mathbf{I}_3$).

The corrupted dataset $\mathbf{Y}$ is constructed following~\eqref{eq:noise_model}, where the noise vectors $\{\boldsymbol{\eta}_i\}_{i=1}^n$ are generated as follows. We first sample two sets of heterogeneity parameters $\{\tau_i\}_{i=1}^{n}$ and $\{\delta_j\}_{j=1}^{d}$, representing sample-specific and feature-specific noise levels, respectively, each drawn independently and uniformly from $[0.01, 0.15]$. For each $\mathbf{x}_i \in \mathbf{X}$, the corresponding noise vector is sampled as
\begin{equation}\label{eq:hetero_gaussian_noise}
    \boldsymbol{\eta}_i \sim \mathcal{N}(\mathbf{0},\boldsymbol{\Sigma}_i), \qquad \boldsymbol{\Sigma}_i = \operatorname{diag}\left(\frac{\tau_i \delta_1}{d}, \ldots, \frac{\tau_i \delta_d}{d}\right) \in \mathbb{R}^{d \times d}.
\end{equation}

\subsection{Comparison with Doubly Stochastic Normalization}\label{sec:compare_l1}
We next compare our proposed approach with the distance estimator based on doubly stochastic normalization~\cite{boris2}, using the exact setup of Figure~\ref{fig:combined_theory}(c) and (d). The estimator of~\cite{boris2} is implemented following equations (3.1) and (3.6) therein, with a fixed bandwidth $\epsilon = 0.001$ and $s = 2$.

\begin{figure}[!ht]
    \centering
    \includegraphics[width=0.95\textwidth]{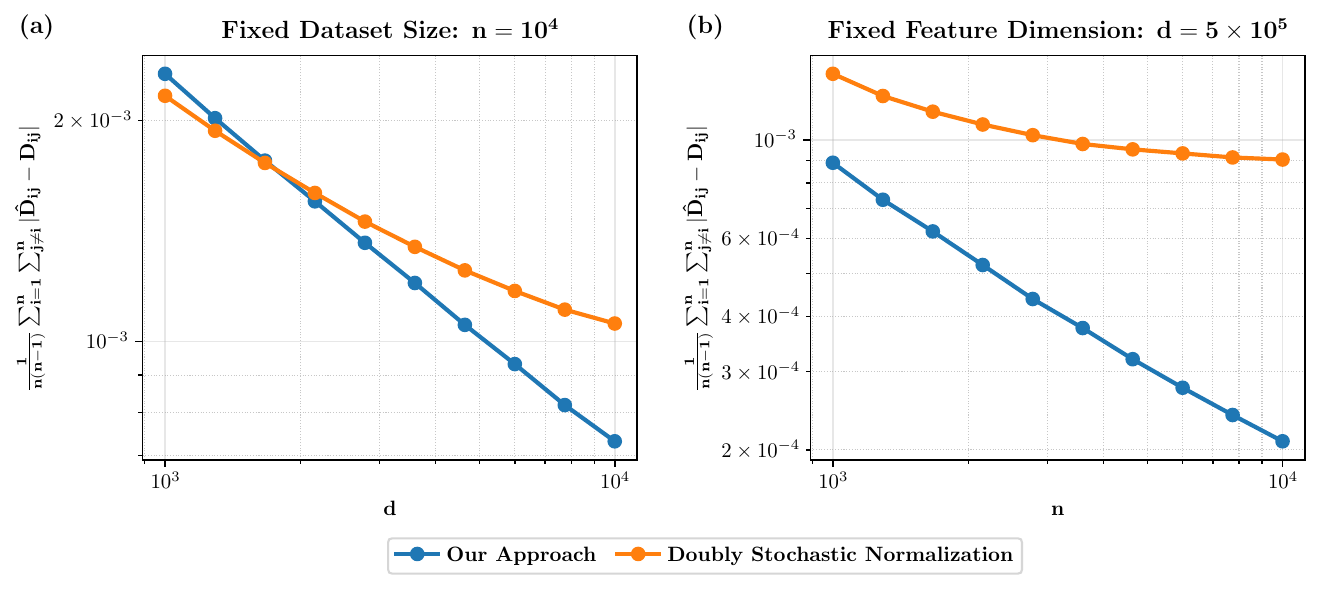}
    \caption{Comparison of our approach with the distance estimator based on doubly stochastic normalization~\cite{boris2} under the setup of Figure~\ref{fig:combined_theory}(c) and (d). (a) Normalized $\ell_1$ recovery errors of the distance matrix as a function of the feature dimension $d$ with fixed dataset size $n = 10^4$. (b) Normalized $\ell_1$ recovery errors as a function of the dataset size $n$ with fixed feature dimension $d = 5 \times 10^5$. Each point represents the average over 10 independent trials.}
    \label{fig:dsn_comparison}
\end{figure}

As shown in Figure~\ref{fig:dsn_comparison}(a), where $n = 10^4$ is fixed and $d$ varies from $10^3$ to $10^4$, the two methods perform comparably for $d \lesssim 2 \times 10^3$, while our approach becomes increasingly advantageous as $d$ grows. The gap is more pronounced in Figure~\ref{fig:dsn_comparison}(b), where $d = 5 \times 10^5$ is fixed and $n$ varies from $10^3$ to $10^4$. Here, the error of the estimator from~\cite{boris2} decreases only marginally with $n$ and quickly plateaus, whereas the error of our approach continues to decay. This behavior is consistent with the error analysis in~\cite{boris2}: for a fixed bandwidth $\epsilon$, the distance estimates carry a bias of order $\mathcal{O}(\epsilon)$ that depends on the intrinsic dimension of the data and does not vanish as $n$ or $d$ increases. Achieving higher accuracy therefore requires reducing the bandwidth, which in turn leads to the numerical instability and convergence issues of doubly stochastic normalization discussed in Section~\ref{sec:related_work}.

\section{Simulation Details of Section~\ref{sec:self_tune_eig}}\label{sec:sim3}
    For this example, we consider a clean dataset $\mathbf{X} = \{\mathbf{x}_i\}_{i=1}^n \subset \mathbb{R}^d$ with sample size $n = 4 \times 10^3$, where each $\mathbf{x}_i \in \mathbf{X}$ is sampled with equal probability from two circles: a large circle $\mathcal{C}_l \subset \mathbb{R}^2$ and a small circle $\mathcal{C}_s \subset \mathbb{R}^2$, both embedded in $\mathbb{R}^d$ via a random orthogonal transformation. Specifically, $\mathcal{C}_l$ is centered at the origin with unit radius and features a non-uniform angular density where $\theta \sim \mathcal{N}(0, (0.17\cdot 2\pi)^2)$, while $\mathcal{C}_s$ is centered at $(1.3,0)$ with a radius of 0.1 and possesses uniform angular density. Formally, each $\mathbf{x}_i \in \mathbf{X}$ is generated as:
            \begin{equation}\label{eq:two_circle}
                \begin{aligned}
                    \theta_i &\sim
                    \begin{cases}
                        \mathcal{N}\left(0, (0.17\cdot 2\pi)^2\right), &i \leq 2000 \\
                        \operatorname{Uniform}[0, 2\pi],&i > 2000
                                    \end{cases},
                \\
                    \mathbf{x}_i &= \begin{cases}
                        \mathbf{R}_d \cdot \begin{bmatrix}
                        \cos(\theta_i), \sin(\theta_i)
                        \end{bmatrix}^{T}, &i \leq 2000 \\
                        \mathbf{R}_d \cdot
                        \begin{bmatrix}
                             1.3 + \frac{\cos(\theta_i)}{10}, \frac{\sin(\theta_i)}{10}
                        \end{bmatrix}^{T}, &i > 2000
                        \end{cases},
                \end{aligned}
            \end{equation}
            where $\mathbf{R}_d \in \mathbb{R}^{d\times2}$ has random orthonormal columns. The random noise vectors $\{\boldsymbol{\eta}_i\}_{i=1}^n$ are generated similarly to~\eqref{eq:geometry_noise}, where for $\mathbf{x}_i \in \mathbf{X}$ that is generated from $\mathcal{C}_l$, the corresponding $\boldsymbol{\eta}_i$ has a magnitude following $g(\theta_i, \pi)$ in~\eqref{eq:geometry_noise}, while for $\mathbf{x}_i$ generated from $\mathcal{C}_s$, its noise magnitude follows $0.1 \cdot g(\theta_i, 0)$, with the scaling factor accounting for the smaller radius of $\mathcal{C}_s$.

\section{Experiment Details of Section~\ref{sec:scRNAseq}}\label{sec:real}

The PBMC dataset of~\cite{zheng2017pbmc} measures the expression of $32{,}738$ genes across $94{,}655$ cells, encompassing $11$ distinct cell types. The cell type annotations are derived from flow cytometry, a technology that classifies cells based on surface protein markers, which makes the dataset a valuable benchmark with ground truth labels. For our experiment, we randomly sample $500$ cells from each of the following six cell types: CD19+ B cells, CD14+CLEC9A- monocytes, CD34+ cells, CD56+ natural killer cells, CD4+ T cells, and CD8+CD45RA+ naive cytotoxic T cells.

For the $k$NN graph quality comparison, we repeat the downsampling procedure 10 times; for each downsampled dataset, we construct the $k$NN graphs from all five sources and compute their impurity scores according to~\eqref{eq:impurity}. The Gaussian kernel with doubly stochastic normalization uses a fixed bandwidth $\sigma^2 = 0.0005$. SNEkhorn is run with its default settings, with a perplexity of 50, a learning rate of 0.1, a maximum of 1{,}000 iterations, and the squared parameterization for numerical stability. For b-matching, the $b$-regular graph is constructed separately for each $k$ of interest with $b = k$. The runtime of each method, reported as mean and standard deviation across repetitions, is summarized in Figure~\ref{fig:pbmc_runtime}.

\begin{figure}[!ht]
    \centering
    \includegraphics[width=0.95\textwidth]{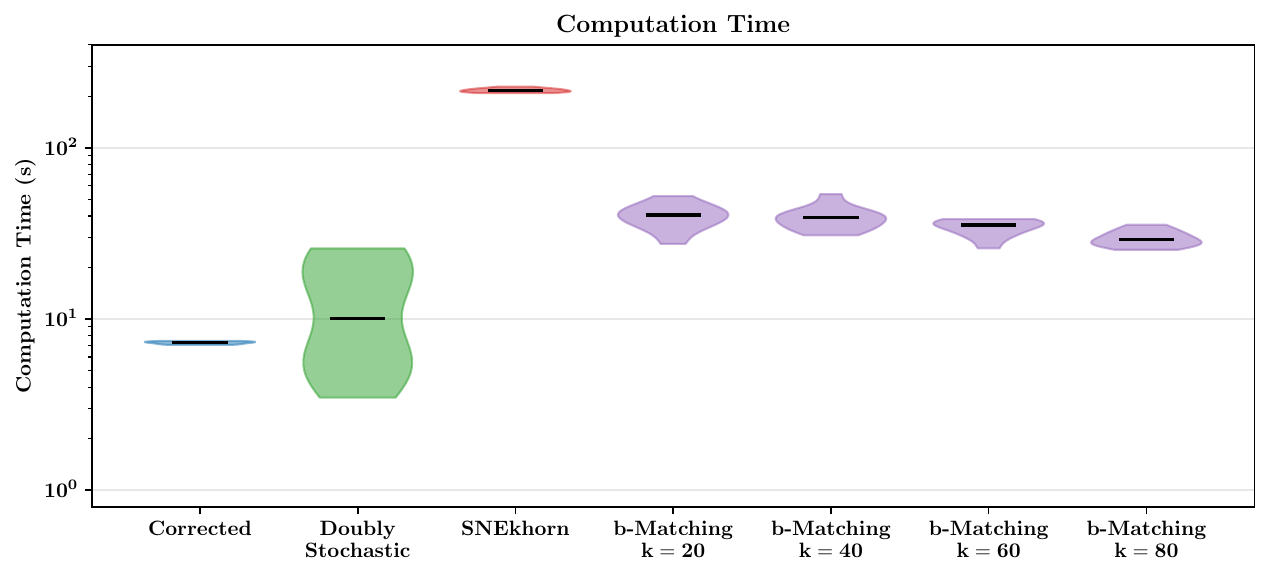}
    \caption{Computation time of each method on the downsampled PBMC dataset ($n = 3\times 10^3$, $d = 32{,}738$) over the 10 repetitions. The corrected distances $\hat{\mathbf{D}}$, the doubly stochastic Gaussian kernel, and SNEkhorn are each computed once and reused for all neighborhood sizes $k$, whereas b-matching must be solved separately for each $k$; its runtimes are reported for $k \in \{20, 40, 60, 80\}$.}
    \label{fig:pbmc_runtime}
\end{figure}

\section{Additional Real-Data Example: Clustering Performance on Raman Spectra of Bacteria}\label{sec:real2}

In this experiment, we consider the Raman spectra dataset of bacterial and yeast cells from~\cite{ho2019raman}. The dataset is a matrix in which each row is the Raman spectrum of a single cell and each column is the measured intensity at a specific wavenumber. As noted in~\cite{ho2019raman}, the spectra are acquired with short integration times and consequently exhibit a low signal-to-noise ratio, with noise levels that vary across spectra. We show that applying our approach as a preprocessing step substantially improves the quality of unsupervised clustering for species identification.

We use the reference dataset of~\cite{ho2019raman}, which consists of $60{,}000$ single-cell Raman spectra from 30 bacterial and yeast isolates ($2{,}000$ spectra per isolate), each measured over $d = 1{,}000$ wavenumbers in the range 382--1792 cm$^{-1}$ with a one-second integration time. The spectra were background-subtracted and min--max scaled to $[0,1]$ by the authors. Since each spectrum is acquired from a culture of a known isolate, the isolate identity provides a ground truth label for evaluating clustering performance. We select one isolate from each of eight species: \emph{C.\ albicans}, \emph{C.\ glabrata}, \emph{P.\ aeruginosa}, \emph{S.\ aureus}, \emph{S.\ enterica}, \emph{S.\ epidermidis}, \emph{S.\ lugdunensis}, and \emph{S.\ pneumoniae}. Following the downsampling idea in Section~\ref{sec:scRNAseq}, we randomly draw 500 spectra without replacement from each selected isolate, yielding a dataset with $n = 4{,}000$ and $d = 1{,}000$ labeled by species. As preprocessing, we subtract the mean of each spectrum and scale it to unit $\ell_2$ norm.

\begin{wrapfigure}{r}{0.42\textwidth}
    \centering
    \vspace{-\intextsep}
    \includegraphics[width=0.40\textwidth]{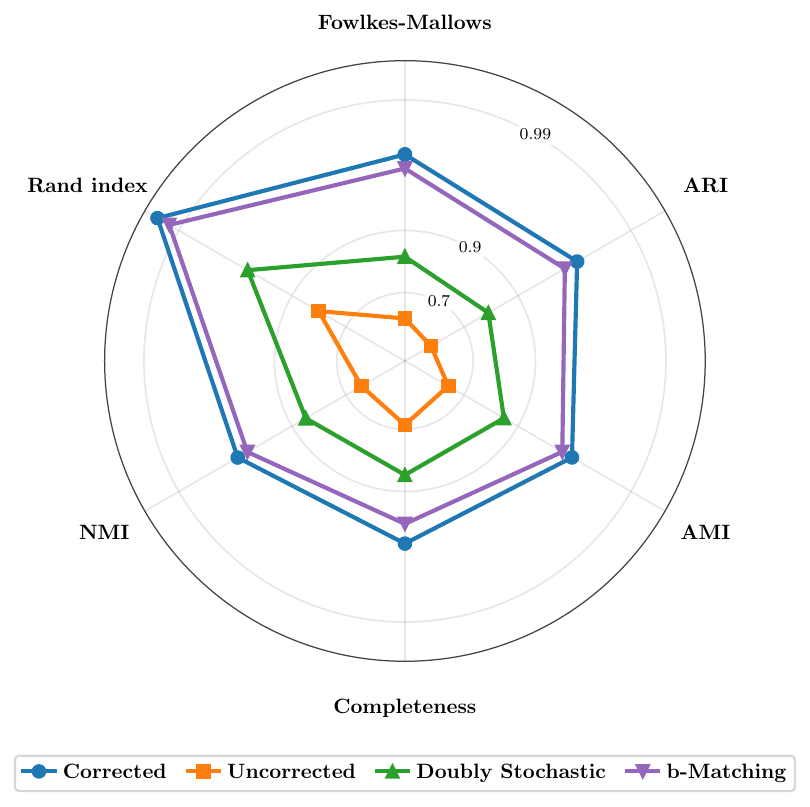}
    \caption{Clustering performance on the Raman spectra of bacterial and yeast cells from~\cite{ho2019raman}. Each axis shows one external clustering score of the Louvain partition of the $k$NN graph ($k = 40$) constructed from the corrected distances $\hat{\mathbf{D}}$, the uncorrected distances $\widetilde{\mathbf{D}}$, the Gaussian kernel on $\widetilde{\mathbf{D}}$ with doubly stochastic normalization, and b-matching, averaged over 5 independently downsampled datasets.}
    \label{fig:raman_radar}
\end{wrapfigure}

Graph-based clustering algorithms, such as Louvain~\cite{louvain}, operate on a $k$NN graph of the data, and their performance depends directly on the quality of the graph. To demonstrate the benefit of our approach for downstream clustering, we perform Louvain clustering on $k$NN graphs constructed from our approach, the corrected distances $\hat{\mathbf{D}}$, the uncorrected distances $\widetilde{\mathbf{D}}$, the Gaussian kernel on $\widetilde{\mathbf{D}}$ with doubly stochastic normalization~\cite{boris1} (The bandwidth is set to the median squared distance to the 15th nearest neighbor under $\widetilde{\mathbf{D}}$), and b-matching~\cite{b_matching}, and compare the quality of the resulting clusters. Specifically, for each source we construct the $k$NN graph with $k = 40$ (for b-matching, the $b$-regular graph with $b = 40$), and partition it into clusters using the Louvain algorithm~\cite{louvain} with resolution parameter 1. The resulting partition is compared against the ground truth species labels using six standard external clustering scores: the adjusted Rand index (ARI), the Rand index, the adjusted mutual information (AMI), the normalized mutual information (NMI), completeness, and the Fowlkes--Mallows index, whose definitions can be found in~\cite{sklearn_api}. Higher values indicate better agreement with the ground truth.

Figure~\ref{fig:raman_radar} summarizes the results, averaged over 5 independently downsampled datasets. The corrected distances achieve the highest value on every score, with all six exceeding $0.9$. In contrast, clustering based on the uncorrected distances performs poorly, with all scores except the Rand index falling below $0.7$. The doubly stochastic Gaussian kernel substantially improves upon the uncorrected distances but remains well below the corrected distances on every score. b-matching comes closest to the corrected distances, but at a higher computational cost.

\section{Computational Considerations and Guidelines for Large Datasets}\label{sec:runtime}

In this section, we discuss the computational complexity of our approach, compare its empirical runtime with that of related methods, and provide practical guidelines for applying it to large datasets.

Our method begins by constructing the distance matrix in~\eqref{eq:corrupted_dist_mat}, a step with complexity $\mathcal{O}(n^2 d)$ that is common to all related methods discussed in Section~\ref{sec:related_work}. The main bottleneck of Algorithm~\ref{alg:debias} is the true near-neighbor identification step in~\eqref{eq:alg_LSAP_1} and~\eqref{eq:alg_LSAP_2}, which employs the Hungarian algorithm with complexity $\mathcal{O}(n^3)$. For comparison, the belief-propagation solver for $b$-matching~\cite{b_matching} requires $\mathcal{O}(bn^3)$ operations. \sloppy SNEkhorn~\cite{snekhorn} computes its affinity matrix by dual ascent, where each iteration costs $\mathcal{O}(n^2)$, but no bound on the number of iterations is provided. The Sinkhorn--Knopp algorithm used for doubly stochastic normalization~\cite{boris1} also costs $\mathcal{O}(n^2)$ per iteration, but its number of iterations depends on the bandwidth and typically grows as the bandwidth decreases. Hence, the per-iteration complexity of the latter two methods does not fully reflect their total computational cost.

We now compare the empirical runtime of our approach against the related methods above. Since construction of the distance matrix in~\eqref{eq:corrupted_dist_mat} is common to all methods, we exclude it from the comparison. We consider a simple two-dimensional setting in which data are generated from a mixture model: half of the samples are drawn from a Gaussian distribution centered at $\boldsymbol{\mu}_1 = (-4,-8)^\top$ with covariance $\mathbf{I}_2$, and the other half from a Gaussian distribution centered at $\boldsymbol{\mu}_2 = (8,-8)^\top$ with covariance $9\,\mathbf{I}_2$. We test dataset sizes $n$ logarithmically spaced from $10^2$ to $10^4$. Our approach requires no hyperparameters. For the Sinkhorn--Knopp algorithm, we set the bandwidth to $\sigma^2 = 0.5$ and terminate when a tolerance of $10^{-4}$ is reached or after a maximum of $5 \times 10^4$ iterations. For $b$-matching, we set $b = 0.05n$ and a maximum of $500$ iterations. For SNEkhorn, we use the default settings except for a larger learning rate of $0.05$ and an increased maximum of $1{,}000$ iterations. All experiments were conducted on a shared Linux server running Ubuntu 20.04.5 LTS, equipped with two AMD EPYC 7643 processors (96 physical cores in total) and 2\,TB of RAM. All algorithms were executed single-threaded on a single CPU core without GPU acceleration.

\begin{wrapfigure}{r}{0.42\textwidth}
    \centering
    \vspace{-\intextsep}
    \includegraphics[width=0.40\textwidth]{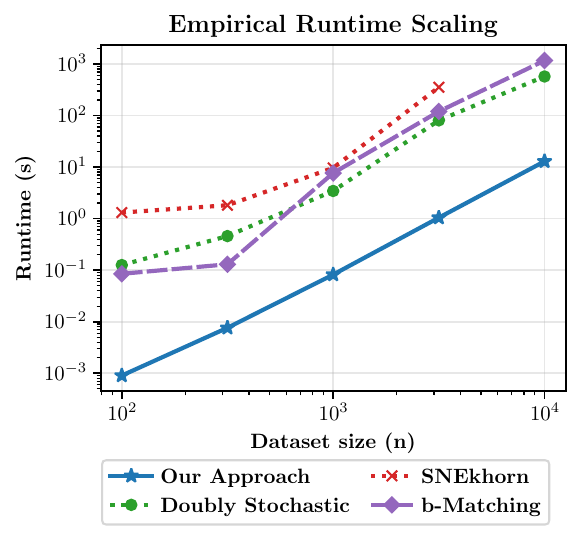}
    \caption{Empirical runtime as a function of dataset size $n$. Our
    method exhibits lower runtime than the other methods across all dataset sizes tested.}
    \label{fig:runtime}
\end{wrapfigure}

As shown in Figure~\ref{fig:runtime}, our approach is fast in practice despite its $\mathcal{O}(n^3)$ complexity, whereas doubly stochastic normalization via the Sinkhorn--Knopp algorithm, despite its $\mathcal{O}(n^2)$ per-iteration complexity, is considerably slower due to slow convergence. $b$-matching and SNEkhorn are slower still for large $n$: $b$-matching exceeds $10^3$ seconds at $n = 10^4$, and SNEkhorn is omitted at $n = 10^4$ due to numerical issues and prohibitive runtime on CPU.

Although our approach is competitive at the scales considered above, its cubic complexity eventually limits scalability to very large datasets. Two strategies could be considered in this regime. First, the Hungarian algorithm can be replaced by more efficient alternatives, such as cost-scaling~\cite{cost_scaling}, auction~\cite{auction}, or pseudoflow algorithms~\cite{pseudoflow}, which achieve complexity $\mathcal{O}(\sqrt{n}\, m \log(nC))$ in the general bipartite matching formulation, where $n$ denotes the number of nodes, $m$ the number of edges, and $C$ the largest edge cost. In our setting, $m = n(n-1)$ and $C$ corresponds to the largest squared distance, yielding an overall complexity of $\mathcal{O}(n^{5/2} \log(nC))$. We refer the reader to~\cite{siam_assignment_book} and the references therein for algorithmic details, analysis, and comparisons. Second, since the distance correction step in~\eqref{eq:correction} requires only the noise magnitude of each point, a very large dataset can be randomly partitioned into subsets, each still large enough for accurate noise magnitude estimation. Applying our method to each subset independently yields a noise magnitude estimate for every point, and these estimates can then be merged to correct the full distance matrix.

\section{Supporting Lemmas for the Proof of Theorem~\ref{thm:partition} and Corollary~\ref{col:manifold}}

\subsection{Background}

For every $m \ge 2$, fix an $m$-dimensional Hilbert space-filling curve $
  H \colon [0,1] \longrightarrow [0,1]^m$,
for whose construction we refer to~\cite{sagan}. We use only two of its
properties.
\begin{enumerate}
    \item $H$ is surjective, so every $u \in [0,1]^m$ has at least one preimage $t \in [0,1]$; preimages need not be unique, and any choice will serve.
    \item $H$ is $\frac{1}{m}$-H\"older in the Euclidean norm: there is a constant $A_m < \infty$, depending only on $m$, with
    \begin{equation}\label{eq:holder}
      \| H(s) - H(t) \|_2 \le A_m\, |s-t|^{\frac{1}{m}},
      \qquad s,t \in [0,1] .
    \end{equation}
    One may take $A_m = 2\sqrt{m+3}$; see~\cite{heowen} and~\cite[Ch.~4]{zumbusch}.
\end{enumerate}

\begin{lemma}\label{lem:ordering}
Let $Q_L := \prod_{r=1}^{m} [b_r, b_r + L] \subset \R^m$ be any axis-aligned cube of side length $L > 0$. For every $n \ge 2$, and every
$z_1,\dots,z_n \in Q_L$, there is an ordering $z_{(1)},\dots,z_{(n)}$ of these
points such that
\begin{equation}\label{eq:ordering}
  \sum_{j=1}^{n-1} \big\| z_{(j+1)} - z_{(j)} \big\|_2^2
  \le A_m^2\, L^2\, n^{\,1-\alpha},
\end{equation}
where $\alpha = \min\{\frac{2}{m},\,1\}$ and $A_m := 2\sqrt{m+3}$.
\end{lemma}

\begin{proof}
We treat the two cases $m = 1$ and $m \ge 2$ separately.

\textbf{The case $m = 1$.} Here $\alpha = 1$, so $n^{1-\alpha} = 1$, and since
$A_1 = 4$, the claim~\eqref{eq:ordering} reads
$\sum_{j} \| z_{(j+1)} - z_{(j)} \|_2^2 \le A_1^2 L^2 = 16L^2$. Take $z_{(1)} \le \cdots \le z_{(n)}$ to be the points sorted along the
interval $Q_L$, so that
\begin{equation}\label{eq:sortline}
  g_j := z_{(j+1)} - z_{(j)} \ge 0 ,
  \qquad
  \sum_{j=1}^{n-1} g_j = z_{(n)} - z_{(1)} \le L .
\end{equation}
By the non-negativity of the $g_j$, we have:
\begin{equation}\label{eq:l2l1}
  \sum_{j=1}^{n-1} g_j^{2} \;\le\; \Big( \sum_{j=1}^{n-1} g_j \Big)^{2}
  \;\le\; L^2 ,
\end{equation}
which satisfies~\eqref{eq:ordering}.

\textbf{The case $m \ge 2$.} Here $\alpha = \frac{2}{m}$.
Let $\Psi(u) := b + L u$, where $b := (b_1,\dots,b_m)$, be the affine bijection
from $[0,1]^m$ onto $Q_L$; it scales all distances by $L$. Hence
$H_L := \Psi \circ H \colon [0,1] \to Q_L$ is a continuous surjection which, by
\eqref{eq:holder}, is $\frac{1}{m}$-H\"older with constant $A_m L$:
\begin{equation}\label{eq:holderL}
  \| H_L(s) - H_L(t) \|_2 \le A_m L\, |s-t|^{\frac{1}{m}},
  \qquad s,t \in [0,1].
\end{equation}
For each $i \in [n]$, choose $t_i \in [0,1]$ with
$H_L(t_i) = {z}_i$, and sort these numbers as $t_{(1)} \le \cdots \le t_{(n)}$
(ties broken arbitrarily).
Setting $z_{(j)} := H_L(t_{(j)})$ defines an ordering of the given points. Write
$a_j := t_{(j+1)} - t_{(j)} \ge 0$ for $j \in [n-1]$, and apply \eqref{eq:holderL}:
\begin{equation}\label{eq:perstep}
  \big\| z_{(j+1)} - z_{(j)} \big\|_2^2 \le A_m^2\, L^2\, a_j^{\,\frac{2}{m}},
  \qquad j \in [n-1].
\end{equation}

We note that
\begin{equation}\label{eq:telescope}
  \sum_{j=1}^{n-1} a_j \;=\; t_{(n)} - t_{(1)} \le 1 .
\end{equation}
Applying H\"older's inequality with the conjugate exponents $(p,q)$, where $p = \frac{m}{2} \ge 1$, we have:
\begin{equation}\label{eq:holderstep}
\begin{aligned}
\sum_{j=1}^{n-1} a_j^{\frac{2}{m}} &\le \left(\sum_{j=1}^{n-1} |a_j^{\frac{2}{m}}|^{p}\right)^{1/p}\left(\sum_{j=1}^{n-1} 1^{q}\right)^{1/q} \\
&= \left(\sum_{j=1}^{n-1} a_j\right)^{\frac{2}{m}}(n-1)^{1-\frac{2}{m}} \le (n-1)^{1-\frac{2}{m}} \le n^{1-\frac{2}{m}}
\end{aligned}
\end{equation}
Summing \eqref{eq:perstep} over $j$ and combining
with \eqref{eq:holderstep} yields \eqref{eq:ordering}.
\end{proof}

Lemma~\ref{lem:manifold} below is the analogue of Lemma~\ref{lem:ordering} on
$\mathcal{M}$.

\begin{lemma}\label{lem:manifold}
Let $\mathcal{M} \subset \R^{d_0}$ be a compact $C^1$ submanifold of dimension $m$,
possibly with boundary, for some $d_0 \ge m$. There exists a positive constant
$A_{\mathcal{M}} < \infty$, depending only on $\mathcal{M}$, such that for every
$n \ge 2$ and every $\mathbf{z}_1,\dots,\mathbf{z}_n \in \mathcal{M}$, there is an
ordering $\mathbf{z}_{(1)},\dots,\mathbf{z}_{(n)}$ of these points satisfying
\begin{equation}\label{eq:manifoldorder}
  \sum_{j=1}^{n-1} \big\| \mathbf{z}_{(j+1)} - \mathbf{z}_{(j)} \big\|_2^2
  \le A_{\mathcal{M}}^2\, n^{\,1-\alpha},
\end{equation}
where $\alpha = \min\{\frac{2}{m},\,1\}$.
\end{lemma}

\begin{proof}
For each $p \in \mathcal{M}$, fix a $C^1$ chart $(U_p, \varphi_p)$ with
$p \in U_p$ and $\varphi_p(p) = 0$, where $\varphi_p$ maps $U_p$
homeomorphically onto a relatively open subset of $\R^m$, or of the
half-space $\R^{m-1} \times [0,\infty)$ when $p$ lies on the boundary.
Set
\[
  C_p := \begin{cases}
    [-r_p, r_p]^m, & p \text{ an interior point},\\[2pt]
    [-r_p, r_p]^{m-1} \times [0, r_p], & p \in \partial \mathcal{M},
  \end{cases}
\]
where $r_p > 0$ is chosen small enough that $C_p \subset \varphi_p(U_p)$;
such an $r_p$ exists because $\varphi_p(U_p)$ is relatively open and
contains the origin. Each $C_p$ is compact and convex, and is contained
in an axis-aligned cube of side length $\ell_p := 2 r_p$. Write also
\[
  V_p := \begin{cases}
    (-r_p, r_p)^m, & p \text{ an interior point},\\[2pt]
    (-r_p, r_p)^{m-1} \times [0, r_p), & p \in \partial \mathcal{M},
  \end{cases}
\]
so that $V_p \subset C_p$ and, in the boundary case, $V_p$ retains the face
$\{u_m = 0\}$ on which $\varphi_p$ meets $\partial\mathcal{M}$.

Each $V_p$ is relatively open in $\R^m$, respectively in the half-space
$\R^{m-1} \times [0,\infty)$, so each $\varphi_p^{-1}(V_p)$ is relatively open
in $\mathcal{M}$ and contains $p$. These sets therefore cover $\mathcal{M}$,
and by compactness, finitely many of them suffice. Index these by $k \in [K]$
and write $(U_k, \varphi_k, C_k, V_k, \ell_k)$ for the corresponding data.
Disjointify the cover by setting
\[
  E_k := \varphi_k^{-1}(V_k) \;\setminus\;
  \bigcup_{k' < k} \varphi_{k'}^{-1}(V_{k'}),
  \qquad k \in [K],
\]
so that $\mathcal{M} = E_1 \sqcup \cdots \sqcup E_K$ and
$\varphi_k(E_k) \subset C_k$ for every $k$.

Fix $k \in [K]$. By definition of a $C^1$ submanifold, $\varphi_k^{-1}$ is $C^1$
on an open neighbourhood of $\varphi_k(U_k)$ in $\R^m$, hence in particular on a
neighbourhood of the compact convex set $C_k$. It is therefore Lipschitz on
$C_k$ with constant
$\Lambda_k := \sup_{u \in C_k} \big\| D\varphi_k^{-1}(u) \big\|_{\mathrm{op}}
< \infty$, where $D\varphi_k^{-1}(u)$ denotes the differential of
$\varphi_k^{-1}$ at $u$ and $\|\cdot\|_{\mathrm{op}}$ the operator norm. Hence,
\begin{equation}\label{eq:chartlip}
\begin{aligned}
  \| x - y \|_2
  &= \big\| \varphi_k^{-1}(\varphi_k(x)) - \varphi_k^{-1}(\varphi_k(y)) \big\|_2 \\
  &\le \Lambda_k \big\| \varphi_k(x) - \varphi_k(y) \big\|_2 ,
  \qquad \forall x, y \in E_k .
\end{aligned}
\end{equation}

Let $n_k$ be the number of the ${z}_i$ lying in $E_k$, and set
\[
  Z_k := \{ {z}_i : {z}_i \in E_k \}, \qquad
  \mathcal{K} := \{ k \in [K] : n_k \ge 1 \}, \qquad
  \sum_{k \in \mathcal{K}} n_k = n .
\]

Fix $k \in \mathcal{K}$. If $n_k = 1$, the ordering of $Z_k$ is trivial, there
are no internal gaps, and the sum on the left of \eqref{eq:withinpatch} below is
empty; so we may assume $n_k \ge 2$. Then $\varphi_k(Z_k)$ is a set of $n_k$
points contained in $C_k$, which in turn lies in an axis-aligned cube of side
length $\ell_k$ (in the boundary case
$[-r_k, r_k]^{m-1} \times [0, r_k] \subset [-r_k, r_k]^m$), so applying
Lemma~\ref{lem:ordering} yields an ordering
$u^{(k)}_{(1)},\dots,u^{(k)}_{(n_k)}$ of $\varphi_k(Z_k)$ with
\begin{equation}\label{eq:patchorder}
  \sum_{j=1}^{n_k-1} \big\| u^{(k)}_{(j+1)} - u^{(k)}_{(j)} \big\|_2^2
  \;\le\; A_m^2\, \ell_k^2\, n_k^{\,1-\alpha}.
\end{equation}
Order points in $Z_k$ as $z^{(k)}_{(1)}, \dots, z^{(k)}_{(n_k)}$, where
$z^{(k)}_{(j)} := \varphi_k^{-1}\big(u^{(k)}_{(j)}\big)$. Combining
\eqref{eq:chartlip} and \eqref{eq:patchorder} gives:
\begin{equation}\label{eq:withinpatch}
\begin{aligned}
  \sum_{j=1}^{n_k-1}
  \big\| z^{(k)}_{(j+1)} - z^{(k)}_{(j)} \big\|_2^2
  &= \sum_{j=1}^{n_k-1}
  \big\| \varphi_k^{-1}\big(u^{(k)}_{(j+1)}\big)
       - \varphi_k^{-1}\big(u^{(k)}_{(j)}\big) \big\|_2^2 \\
  &\le \Lambda_k^2\, \ell_k^2\, A_m^2\, n_k^{\,1-\alpha}.
\end{aligned}
\end{equation}

Define $B := \max_{k \in [K]} \Lambda_k \ell_k\, A_m < \infty$. Concatenate the $|\mathcal{K}|$ within-patch orderings. This creates
$|\mathcal{K}| - 1 \le K-1$ transition gaps between different patches, each of
squared length at most $\diam(\mathcal{M})^2$. Hence,
\begin{equation}\label{eq:concat}
  \sum_{j=1}^{n-1} \big\| z_{(j+1)} - z_{(j)} \big\|_2^2
  \;\le\; (K-1)\diam(\mathcal{M})^2 \;+\; B^2
  \sum_{k \in \mathcal{K}} n_k^{\,1-\alpha},
\end{equation}
where the first term bounds the at most $K-1$ gaps between consecutive patches, and the second sums \eqref{eq:withinpatch} over
$k \in \mathcal{K}$.

We analyze \eqref{eq:concat} in cases of $m \le 2$ and $m \ge 3$, separately. If $m \le 2$, then $\alpha = 1$, we have:
\begin{equation}
\begin{aligned}
  \sum_{j=1}^{n-1} \big\| z_{(j+1)} - z_{(j)} \big\|_2^2
  &\;\le\; (K-1)\diam(\mathcal{M})^2 \;+\; B^2 K \\
  &= (K-1)\diam(\mathcal{M})^2 \;+\; B^2 K n^{1-\alpha}
\end{aligned}
\end{equation}
If $m \ge 3$, then $1-\alpha \in (0,1)$ and
H\"older's inequality with the conjugate exponents
$\big(\tfrac{1}{1-\alpha},\tfrac{1}{\alpha}\big)$ gives
\begin{equation}\label{eq:patchsum}
  \sum_{k \in \mathcal{K}} n_k^{\,1-\alpha}
  \le \Big( \sum_{k \in \mathcal{K}} n_k \Big)^{1-\alpha}
      |\mathcal{K}|^{\alpha}
  \le K^{\alpha}\, n^{\,1-\alpha} \le K n^{\,1-\alpha}.
\end{equation}
Hence,
\begin{equation}
  \sum_{j=1}^{n-1} \big\| z_{(j+1)} - z_{(j)} \big\|_2^2
  \;\le\; (K-1)\diam(\mathcal{M})^2 \;+\; B^2 K n^{1-\alpha}
\end{equation}
In both cases $n^{1-\alpha} \ge 1$, so the transition-gap term may be
absorbed:
\[
  (K-1)\diam(\mathcal{M})^2 + B^2 K n^{1-\alpha}
  \;\le\; \Big( (K-1)\diam(\mathcal{M})^2 + K B^2 \Big) n^{1-\alpha}.
\]
Choosing $A_{\mathcal{M}}^2 := (K-1)\diam(\mathcal{M})^2 + K B^2$, which is
finite and depends only on $\mathcal{M}$, concludes the proof.
\end{proof}

\section{Proof for Theorem~\ref{thm:partition}}\label{sec:proof_partition}

\begin{proof}

\textbf{Step 1: enclose the data in a cube.} The norm bound
$\|\mathbf{x}_i\|_2 \le R_d$ gives $\diam(\mathbf{X}) \le 2R_d$. Since
$\mathbf{z}_i \in \Omega$ for all $i \in [n]$, the lower bound
in~\eqref{eq:bilipschitz} applies to every pair of points in $\mathbf{Z}$, so
$\diam(\mathbf{Z}) \le \diam(\mathbf{X})/a \le 2R_d/a$. Consequently, for each
coordinate $r \in [m]$,
$\max_i(\mathbf{z}_i)_r - \min_i (\mathbf{z}_i)_r \le \diam(\mathbf{Z}) \le 2R_d/a$, so
$\mathbf{Z}$ is contained in an axis-aligned cube
$Q_L \subset \R^m$ of side length $L := 2R_d/a$.

\textbf{Step 2: order the points.} Apply Lemma~\ref{lem:ordering} to
$\mathbf{Z} \subset Q_L$ and write $\mathbf{z}_{(1)},\dots,\mathbf{z}_{(n)}$ for the resulting
ordering. Setting
\[
  g_j := \big\| \mathbf{z}_{(j+1)} - \mathbf{z}_{(j)} \big\|_2, \qquad j \in [n-1],
\]
we have:
\begin{equation}\label{eq:gapsum}
  \sum_{j=1}^{n-1} g_j^{2} \le A_m^{2}\, L^2\, n^{\,1-\alpha}.
\end{equation}

\textbf{Step 3: turn the ordering into partitions.} Write $n = 4k + s$ with
$k \ge 1$ and $s \in \{0,1,2,3\}$. Define
blocks of consecutive points in the above ordering by
\[
  \mathcal{P}_l := \big\{ \mathbf{z}_{(4l-3)}, \dots, \mathbf{z}_{(4l)} \big\}, \quad l \in [k-1],
  \qquad
  \mathcal{P}_k := \big\{ \mathbf{z}_{(4k-3)}, \dots, \mathbf{z}_{(n)} \big\}.
\]
These sets are pairwise disjoint and cover $\mathbf{Z}$, so
$\PZ := \{\mathcal{P}_l\}_{l=1}^{k}$ is a partition. Moreover $|\mathcal{P}_l| = 4$ for
$l \in [k-1]$ and $|\mathcal{P}_k| = 4 + s \in \{4,\dots,7\}$, so every block contains at
least four points.

\textbf{Step 4: derive the diameter bound of $\mathcal{P}_{\mathbf{Z}}$.} Consider a subset
$\mathcal{P}_l$, spanning the index range $a_l, a_l+1, \dots, a_l+q_l-1$ with $q_l := |\mathcal{P}_l| \le 7$,
and let $
  J_l := \{a_l, a_l+1, \dots, a_l+q_l-2\}$
be the set of indices of the $q_l-1$ gaps internal to $\mathcal{P}_l$.  For any $s'$ and $t'$ with
$a_l \le s' < t' \le a_l+q_l-1$, we have
\[
  \mathbf{z}_{(t')} - \mathbf{z}_{(s')} \;=\; \sum_{j=s'}^{t'-1} \big( \mathbf{z}_{(j+1)} - \mathbf{z}_{(j)} \big),
\]
and applying the triangle inequality, we obtain
\begin{equation}\label{eq:triangle}
  \big\| \mathbf{z}_{(t')} - \mathbf{z}_{(s')} \big\|_2
  \le \sum_{j=s'}^{t'-1} \big\| \mathbf{z}_{(j+1)} - \mathbf{z}_{(j)} \big\|_2
  \;=\; \sum_{j=s'}^{t'-1} g_j .
\end{equation}
Squaring \eqref{eq:triangle} and applying the Cauchy--Schwarz inequality to the
$t'-s'$ summands gives
\begin{equation}\label{eq:cs}
  \big\| \mathbf{z}_{(t')} - \mathbf{z}_{(s')} \big\|_2^{2}
  \le \Bigg( \sum_{j=s'}^{t'-1} g_j \Bigg)^{2}
  \le (t'-s') \sum_{j=s'}^{t'-1} g_j^{2}
  \le (q_l-1) \sum_{j \in J_l} g_j^{2}.
\end{equation}
Taking the supremum of \eqref{eq:cs} over
$s'$ and $t'$, we have:
\begin{equation}
  \big( \diam(\mathcal{P}_l) \big)^{2} \le (q_l-1) \sum_{j \in J_l} g_j^{2}
\end{equation}
Hence,
\begin{equation}\label{eq:blockdiam}
  |\mathcal{P}_l| \big( \diam(\mathcal{P}_l) \big)^{2} \le q_l(q_l-1) \sum_{j \in J_l} g_j^{2} \leq 42 \sum_{j \in J_l} g_j^{2}.
\end{equation}
Using \eqref{eq:blockdiam}, we have:
\begin{equation}\label{eq:blocksum}
\begin{aligned}
  \frac{1}{nR_d^2}\sum_{\mathcal{P}_l \in \PZ} |\mathcal{P}_l| \big( \diam(\mathcal{P}_l) \big)^{2}
  &\le \frac{42}{nR_d^2} \sum_{l=1}^{k} \sum_{j \in J_l} g_j^{2}
  \le \frac{42}{nR_d^2} \sum_{j=1}^{n-1} g_j^{2} \\
  &\le \frac{42\, A_m^{2}\, L^2}{R_d^2}\, n^{-\alpha} = \frac{672\,(m+3)}{a^2}\, n^{-\alpha},
\end{aligned}
\end{equation}
where the last equality uses $L = 2R_d/a$ from Step~1 and $A_m = 2\sqrt{m+3}$.

\textbf{Step 5: transfer to $\mathbf{X}$.} Let $\PX$ be the partition of
$\mathbf{X}$ induced by $\PZ$, that is, $\mathbf{x}_i$ and $\mathbf{x}_{i'}$
lie in the same subset of $\PX$ if and only if $\mathbf{z}_i$ and $\mathbf{z}_{i'}$ lie in the same
subset of $\PZ$. Corresponding subsets have equal cardinalities,
and by the upper bound in~\eqref{eq:bilipschitz}, the diameter of each subset of
$\PX$ is at most $b$ times the diameter of the corresponding subset of $\PZ$.
Combining this with~\eqref{eq:blocksum},
\[
\begin{aligned}
  \frac{1}{nR_d^2}\sum_{\mathcal{P}_l \in \PX} |\mathcal{P}_l| \big(\diam(\mathcal{P}_l)\big)^2
  &\le \frac{672\,(m+3)\, b^2}{a^2}\, n^{-\alpha}
  = 672\,(m+3)\,\kappa^2\, n^{-\alpha} \\
  &\le C m \kappa^2\, n^{-\alpha} \le C m \kappa_0^2\, n^{-\alpha},
\end{aligned}
\]
where $C$ is chosen such that $Cm \ge 672(m+3)$ holds for all $m \ge 1$.
Thus $\mathbf{X}$ satisfies Assumption~\ref{assump:generalized_cluster} with
$\alpha = \min\{\tfrac{2}{m},1\}$ and $c = Cm\kappa_0^2$, completing the proof.
\end{proof}

\section{Proof for Corollary~\ref{col:manifold}}\label{sec:proof_corollary}
\begin{proof}
\textbf{Step 1: order the points.} Apply Lemma~\ref{lem:manifold} to $\mathbf{Z} \subset \mathcal{M}$, which yields
an ordering $\mathbf{z}_{(1)},\dots,\mathbf{z}_{(n)}$ whose gaps $g_j := \| \mathbf{z}_{(j+1)} - \mathbf{z}_{(j)} \|_2$      satisfy
\begin{equation}\label{eq:manifoldgapsum}
\sum_{j=1}^{n-1} g_j^{2} \le A_{\mathcal{M}}^{2}\, n^{\,1-\alpha}.
\end{equation}

\textbf{Step 2: turn the ordering into partitions and bound the diameters.}
Steps~3 and~4 of the proof of Theorem~\ref{thm:partition} apply verbatim, with
\eqref{eq:manifoldgapsum} in place of \eqref{eq:gapsum}, and give a partition $\PZ$ of
$\mathbf{Z}$ into subsets of at least four points satisfying:
\begin{equation}\label{eq:manifoldblocksum}
  \frac{1}{nR_d^2}\sum_{\mathcal{P}_l \in \PZ}
  |\mathcal{P}_l| \big( \diam(\mathcal{P}_l) \big)^{2}
  \le \frac{42\, A_{\mathcal{M}}^{2}}{R_d^2}\, n^{-\alpha}.
\end{equation}

\textbf{Step 3: transfer to $\mathbf{X}$.} As in Step~5 of the proof of Theorem~\ref{thm:partition}, the induced partition $\PX$ has
the same cardinalities, and by the upper bound in~\eqref{eq:bilipschitz},
its diameters are at most $b$ times those of $\PZ$. Moreover, by the triangle inequality, the norm bound $\|\varphi(\mathbf{z})\|_2 \le R_d$, and~\eqref{eq:bilipschitz}, we have:
\begin{equation}\label{eq:R_d_a}
    a\|\mathbf{z} - \mathbf{z}'\|_2 \le \|\varphi(\mathbf{z}) - \varphi(\mathbf{z}')\|_2
\le \|\varphi(\mathbf{z})\|_2 +\| \varphi(\mathbf{z}')\|_2 \leq
2R_d, \quad \forall \mathbf{z}, \mathbf{z}' \in \mathcal{M}.
\end{equation}
Taking the supremum over $\mathbf{z}, \mathbf{z}' \in \mathcal{M}$ of~\eqref{eq:R_d_a}, we have $\diam(\mathcal{M}) \le 2R_d/a$.

Hence, from~\eqref{eq:manifoldblocksum}, we have:
\[
\begin{aligned}
  \frac{1}{nR_d^2}\sum_{\mathcal{P}_l \in \PX}
  |\mathcal{P}_l| \big( \diam(\mathcal{P}_l) \big)^{2}
  &\le \frac{42\, b^2 A_{\mathcal{M}}^{2}}{R_d^2}\, n^{-\alpha} \\
  &\le \frac{168\, A_{\mathcal{M}}^{2}}{\diam(\mathcal{M})^2}\,\kappa^2\, n^{-\alpha} \le \frac{168\, A_{\mathcal{M}}^{2}}{\diam(\mathcal{M})^2}\,\kappa_0^2\, n^{-\alpha}.
\end{aligned}
\]
Setting $C_{\mathcal{M}} := 168\, A_{\mathcal{M}}^{2}/\diam(\mathcal{M})^2$,
which depends only on $\mathcal{M}$, completes the proof.
\end{proof}

        \section{Proof for Lemma~\ref{lem:Geometric LSAP Cost}}\label{sec:Geometric LSAP Cost}

        The proof uses a graph-theoretic approach based on the following result (see, e.g.,~\cite{graph_theory}):

     \begin{lemma}\label{lem:bipartite-matching}
        Let $G = (U, V; E)$ be a bipartite graph with $|U| = |V| = n$. Let $\deg(v)$ denote the degree of any vertex $v \in U \cup V$. If $\deg(v) \geq \frac{n}{2}$ for every vertex $v \in U \cup V$, then $G$ contains a perfect matching.
    \end{lemma}

        \begin{proof}
            We reindex the data points in $\mathbf{X}$ according to the partition $\mathcal{P}_\mathbf{X} = \{\mathcal{P}_i\}_{i=1}^{k}$ in Assumption~\ref{assump:generalized_cluster}, such that data points belonging to the same subset are grouped together.
            We define a permutation matrix $\mathbf{P}_1$ as:
            \begin{equation}\label{eq:P1}
                \mathbf{P}_1 =
                \left[
                \begin{array}{cccc}
                \mathbf{B}_1 & & & \\
                & \mathbf{B}_2 & & \\
                & & \ddots & \\
                & & & \mathbf{B}_k
                \end{array}
                \right], \quad
                \mathbf{B}_i = \begin{bmatrix}
                0 & 1 & 0 & \cdots & 0 \\
                0 & 0 & 1 & \cdots & 0 \\
                0 & 0 & \ddots & \ddots & \vdots \\
                \vdots & \vdots & \ddots & 0 & 1 \\
                1 & 0 & \cdots & 0 & 0
                \end{bmatrix}, \quad \forall i \in [k],
            \end{equation}
            where each block $\mathbf{B}_i$ on the diagonal of $\mathbf{P}_1$ is of size $|\mathcal{P}_i| \times |\mathcal{P}_i|$ and is a cyclic permutation matrix with 1s on the superdiagonal. We note that $\mathbf{P}_1$ is feasible for the LSAP optimization in~\eqref{eq:lemma_lsap} with the cost matrix $\mathbf{D}$ as $\operatorname{diag}(\mathbf{P}_1) = \mathbf{0}$.

            By construction of $\mathbf{P}_1$ in~\eqref{eq:P1}, we establish the following bound for $\mathbf{P}^{(1)}$ from~\eqref{eq:lemma_lsap}:\begin{equation} \label{eq:clean_iter1}
            \begin{aligned}
              \frac{1}{nR_d^2}\operatorname{Tr}((\mathbf{P}^{(1)})^T{\mathbf{D}})
              &= \frac{1}{nR_d^2} \min_{\mathbf{P} \in \mathcal{P}^n} \operatorname{Tr}(\mathbf{P}^T\mathbf{D})
              \leq \frac{1}{nR_d^2}\operatorname{Tr}(\mathbf{P}_1^T{\mathbf{D}}) \\
               &\leq \frac{1}{nR_d^2} \sum_{\mathcal{P}_i \in \mathcal{P}_\mathbf{X}} |\mathcal{P}_i| \left(\operatorname{diam}(\mathcal{P}_i)\right)^2
               \leq  c n^{-\alpha}.
            \end{aligned}
            \end{equation}
            The first inequality follows from the optimality of the minimization. For the second inequality, each diagonal element of $\mathbf{P}_1^T\mathbf{D}$ represents the squared distance between two data points from the same subset, hence is bounded by the squared diameter of the corresponding subset (see Definition~\ref{def:diameter}). Applying this upper bound for each diagonal element and summing yields the second inequality. The last inequality follows from Assumption~\ref{assump:generalized_cluster}.

            Next, we show that for the LSAP optimization with the masked cost matrix $\mathbf{D}^\prime$ in~\eqref{eq:masked_D_def}, there exists a feasible permutation matrix $\mathbf{P}_2$ with the same block structure as $\mathbf{P}_1$ in~\eqref{eq:P1},
            where the diagonal block $\mathbf{B}_i^{\prime}$ is a
            permutation matrix of size $|\mathcal{P}_i| \times |\mathcal{P}_i|$, for any $i \in [k]$.
            To prove the existence of $\mathbf{P}_2$, it suffices to prove the existence of each permutation matrix $\mathbf{B}_l^{\prime}$ for all $l \in [k]$. We note that a permutation matrix can be interpreted as a bijection from the set of rows to the set of columns. Thus, the existence of each permutation matrix $\mathbf{B}_l^{\prime}$ is equivalent to the existence of a perfect matching in a bipartite graph constructed from the rows and columns of $\mathbf{B}_l^\prime$. Specifically, for each $\mathbf{B}_l^\prime$ with $n_l= |\mathcal{P}_l|$ rows and columns, we construct a bipartite graph $G_l = (U_l, V_l; E_l)$, where:
            \begin{enumerate}[label=(\roman*), labelindent=\parindent, leftmargin=*]
                \item $U_l= \{u_1, \ldots, u_{n_l}\}$ and $V_l= \{v_1, \ldots, v_{n_l}\}$, with $u_j$ and $v_j$ representing the $j$th row and column of $\mathbf{B}_l^\prime$, respectively;
                \item $E_l = \{(u_i,v_j) \mid \text{the corresponding entry in } \mathbf{D}^\prime \text{ is not} +\infty\}$.
            \end{enumerate}
            Note that (ii) removes edges corresponding to $+\infty$ in the cost matrix $\mathbf{D}^\prime$. This construction ensures that if a perfect matching exists in each $G_l$ for all $l \in [k]$, the resulting $\mathbf{P}_2$ will be feasible for the LSAP optimization with $\mathbf{D}^\prime$.

            We now prove the existence of each $\mathbf{B}_l^\prime$ for all $l \in [k]$. By the definition of the edge set $E_l$, for each vertex $s \in U_l \cup V_l$, we have $\deg(s) \geq |\mathcal{P}_l| - 2$. This is because $\mathbf{D}^\prime$ from~\eqref{eq:masked_D_def} contains at most two $+\infty$ entries in each row and each column. As a result, when we focus on the block in $\mathbf{D}^\prime$ that corresponds to $\mathbf{B}_l^{\prime}$, each row and column contains at most two $+\infty$ entries. Thus, by the construction of $G_l$, we exclude at most two edges per vertex. Under the minimum size condition (i.e., $\forall l \in [k]$, $|\mathcal{P}_l| \geq 4$), we have $|\mathcal{P}_l| - 2 \geq \frac{|\mathcal{P}_l|}{2}$. By Lemma~\ref{lem:bipartite-matching}, a perfect matching exists for each $G_l$, thus establishing the existence of each $\mathbf{B}_l^\prime$, and consequently $\mathbf{P}_2$.

            Since $\mathbf{P}_2$ shares the same block structure as $\mathbf{P}_1$ from~\eqref{eq:P1}, following the same derivation in~\eqref{eq:clean_iter1}, we have:
            \begin{equation}
            \begin{aligned}
                \frac{1}{nR_d^2}\operatorname{Tr}((\mathbf{P}^{(2)})^T{\mathbf{D}^{\prime}})
              &= \frac{1}{nR_d^2} \min_{\mathbf{P} \in \mathcal{P}^n} \operatorname{Tr}(\mathbf{P}^T\mathbf{D}^{\prime})
              \leq \frac{1}{nR_d^2}\operatorname{Tr}(\mathbf{P}_2^T{\mathbf{D}}^{\prime}) \\
              &= \frac{1}{nR_d^2}\operatorname{Tr}(\mathbf{P}_2^T{\mathbf{D}})
               \leq \frac{1}{nR_d^2} \sum_{\mathcal{P}_i \in \mathcal{P}_\mathbf{X}} |\mathcal{P}_i| \left(\operatorname{diam}(\mathcal{P}_i)\right)^2 \\
               &\leq  c n^{-\alpha},
            \end{aligned}
            \end{equation}
            where $\frac{1}{nR_d^2}\operatorname{Tr}(\mathbf{P}_2^T{\mathbf{D}}^{\prime})
              = \frac{1}{nR_d^2}\operatorname{Tr}(\mathbf{P}_2^T{\mathbf{D}})$ follows from the feasibility of $\mathbf{P}_2$ and~\eqref{eq:masked_D_def}.
        \end{proof}

        \section{Auxiliary Definition and Supporting Lemma for the Proof of Theorem~\ref{thm:estimation_error_bound}}
            For notational simplicity, we introduce the notation of \textit{order in probability}, denoted as $\widetilde{\mathcal{O}}_{d}$~\cite{boris2}.
            \begin{definition} \label{def:order_with_high_probability}
                    For a random variable $z$, we say $z = \widetilde{\mathcal{O}}_{d}\left(f(d)\right)$ if there exist $t_0, d_0, C >0$ such that for all $t>t_0$ and $d>d_0$, with probability at least $1-d^{-t}$,
                    \begin{equation}
                        |z| \leq {t} C f(d).
                    \end{equation}
            \end{definition}

        \begin{lemma} \label{lem:epsilon_bound_in_prob}
            Given a dataset $\mathbf{X} = \{\mathbf{x}_i\}_{i=1}^n \subset \mathbb{R}^{d}$ with $\|\mathbf{x}_i\|_2 \leq R_d$ for all $i \in [n]$, under Assumption~\ref{assump:noise}, there exist constants $c^{\prime}, d_0, t_0 > 0$ such that for any $t > t_0$ and $d > d_0$,
                \begin{equation} \label{eq:epsilon_bound_in_prob}
                    \Pr\left\{|\epsilon_{ij}| > c^{\prime} t R_d^2 \cdot \mathcal{E}(d)\right\} \leq d^{-t},
                \end{equation}
                for any $i, j \in [n]$ with $i \neq j$, where $\epsilon_{ij}$ is defined in~\eqref{eq:squared_distance_cost}, and $\mathcal{E}(d) := \sqrt{\log d} \cdot \max\{E, E^2\sqrt{d}\}$ for the relative noise scale $E$ in Assumption~\ref{assump:noise}.
        \end{lemma}

        \begin{proof}
            Define the normalized vectors $\overline{\mathbf{x}}_i=\mathbf{x}_i/R_d$ and $\overline{\boldsymbol{\eta}}_i=\boldsymbol{\eta}_i/R_d$. Then $\|\overline{\mathbf{x}}_i\|_2\leq1$, and Assumption~\ref{assump:noise} gives
            \[
                \max_{\mathbf{x}_i\in\mathbf X}\|\overline{\boldsymbol{\eta}}(\mathbf{x}_i)\|_{\psi_2}=E.
            \]
            Moreover, if $\overline{\epsilon}_{ij}$ denotes the term in~\eqref{eq:squared_distance_cost} associated with $\overline{\mathbf{x}}_i$ and $\overline{\boldsymbol{\eta}}_i$, then $\overline{\epsilon}_{ij}=\epsilon_{ij}/R_d^2$. By Definition~\ref{def:order_with_high_probability}, it therefore suffices to show that $|\overline{\epsilon}_{ij}| = \widetilde{\mathcal{O}}_{d}(\mathcal{E}(d))$ for any $i,j \in [n]$ with $i \neq j$.

            By the definition of $\overline{\epsilon}_{ij}$ and the triangle inequality,
            \begin{equation}\label{eq:epsilon_decompose}
            \begin{aligned}
                |\overline{\epsilon}_{ij}| &\leq 2\left(|\langle \overline{\mathbf{x}}_i, \overline{\boldsymbol{\eta}}_i\rangle| + |\langle \overline{\mathbf{x}}_i, \overline{\boldsymbol{\eta}}_j\rangle| + |\langle \overline{\mathbf{x}}_j, \overline{\boldsymbol{\eta}}_j\rangle| + |\langle \overline{\mathbf{x}}_j, \overline{\boldsymbol{\eta}}_i\rangle| + |\langle \overline{\boldsymbol{\eta}}_i, \overline{\boldsymbol{\eta}}_j\rangle|\right).
            \end{aligned}
            \end{equation}

            We bound each term in~\eqref{eq:epsilon_decompose} using results from SM1.1 of~\cite{boris2}. Specifically, under Assumption~\ref{assump:noise},
            \begin{equation}\label{eq:x_eta_eta}
            \begin{aligned}
                &|\langle \overline{\mathbf{x}}_i, \overline{\boldsymbol{\eta}}_j\rangle| = \widetilde{\mathcal{O}}_{d}(\mathcal{E}(d)), \quad \forall i,j\in[n],
                \qquad\text{and}\\
                &|\langle \overline{\boldsymbol{\eta}}_i, \overline{\boldsymbol{\eta}}_j\rangle| = \widetilde{\mathcal{O}}_{d}(\mathcal{E}(d)), \quad \forall i,j\in[n]\text{ with }i\neq j.
            \end{aligned}
            \end{equation}
            Combining~\eqref{eq:x_eta_eta} with~\eqref{eq:epsilon_decompose} and applying the union bound over the five terms in~\eqref{eq:epsilon_decompose} gives $|\overline{\epsilon}_{ij}| = \widetilde{\mathcal{O}}_{d}(\mathcal{E}(d))$ for any $i,j \in [n]$ with $i \neq j$, which proves~\eqref{eq:epsilon_bound_in_prob}.
        \end{proof}
        Lemma~\ref{lem:epsilon_bound_in_prob} formalizes the concentration of $\epsilon_{ij}$ relative to the squared clean-data scale. Under Assumption~\ref{assump:noise}, $\mathcal{E}(d) \leq \max\{C d^{-1/4}, C^2\left(\log d\right)^{-1/2}\} \rightarrow 0$ as $d \rightarrow +\infty$. Thus, $|\epsilon_{ij}|/R_d^2$ concentrates around zero as $d$ grows, justifying~\eqref{eq:approx_ep} in the normalized setting $R_d=1$ used in Section~\ref{sec:method derivation}.

        \section{Proof of Theorem~\ref{thm:estimation_error_bound}}\label{sec:estimation_error_bound}

        \begin{proof}
            Let $c_{\epsilon}>0$ be the constant $c^{\prime}$ from Lemma~\ref{lem:epsilon_bound_in_prob}. For any $t>t_0$, define
            \begin{equation}\label{eq:uniform_epsilon_event}
                \Omega_t:=\bigcap_{\substack{i,j\in[n]\\i\neq j}}\left\{|\epsilon_{ij}|\leq c_{\epsilon}tR_d^2\mathcal E(d)\right\}.
            \end{equation}
            By Lemma~\ref{lem:epsilon_bound_in_prob} and the union bound,
            \begin{equation}\label{eq:uniform_epsilon_probability}
                \Pr(\Omega_t)\geq1-n(n-1)d^{-t}.
            \end{equation}
            We derive the two error bounds deterministically on $\Omega_t$. Set
            \begin{equation}\label{eq:B_t_definition}
                B_t:=c_{\epsilon}tR_d^2\mathcal E(d),
            \end{equation}
            so that $|\epsilon_{ij}|\leq B_t$ for all $i\neq j$ on $\Omega_t$.

        \textbf{Notation}: Given a cost matrix $\mathbf{C}$ and a permutation matrix $\mathbf{P}$, we define the corresponding assignment cost as $\mathcal{L}_\mathbf{C}(\mathbf{P}): = \operatorname{Tr}(\mathbf{P}^T\mathbf{C})$. Let $\widetilde{\mathbf{D}}'$, $\widetilde{\mathbf{P}}^{(1)}, \widetilde{\mathbf{P}}^{(2)}, \widetilde{\sigma}_1, \widetilde{\sigma}_2$ be obtained by executing Algorithm~\ref{alg:debias} on the corrupted dataset $\mathbf{Y}$. Let $\mathbf{D}'$, $\mathbf{P}^{(1)}$, and $\mathbf{P}^{(2)}$ be obtained from~\eqref{eq:masked_D_def} and~\eqref{eq:lemma_lsap}, with the masking permutation matrix in Lemma~\ref{lem:Geometric LSAP Cost} taken as $\widetilde{\mathbf{P}}^{(1)}$. We define $\sigma_1, \sigma_2 : [n] \rightarrow [n]$ for $\mathbf{P}^{(1)}$ and $\mathbf{P}^{(2)}$ respectively, where $\sigma_1(i) = j$ if $P^{(1)}_{ij} = 1$ and $\sigma_2(i) = j$ if $P^{(2)}_{ij} = 1$.

            We first establish bounds for the following three quantities: $\sum_{i=1}^n D_{i\widetilde{\sigma}_1(i)}$, $\sum_{i=1}^n D_{i\widetilde{\sigma}_2(i)}$, and $\sum_{i=1}^n D_{\widetilde{\sigma}_1(i)\widetilde{\sigma}_2(i)}$.

            \paragraph{The bound for $\sum_{i=1}^n D_{i\widetilde{\sigma}_1(i)}$} We bound $\sum_{i=1}^n D_{i\widetilde{\sigma}_1(i)}$ by comparing $\mathcal{L}_{\widetilde{\mathbf{D}}}(\widetilde{\mathbf{P}}^{(1)})$ and $\mathcal{L}_{\widetilde{\mathbf{D}}} (\mathbf{P}^{(1)})$.

            First, by the definition of $\widetilde{\mathbf{P}}^{(1)}$ in~\eqref{eq:alg_LSAP_1} and the optimality of minimization, we have:
            \begin{equation}\label{eq:LSAP_iter1}
                    \mathcal{L}_{\widetilde{\mathbf{D}}}(\widetilde{\mathbf{P}}^{(1)})
                    \leq \mathcal{L}_{\widetilde{\mathbf{D}}}
                    (\mathbf{P}^{(1)}).
            \end{equation}
            Next, we rewrite $\mathcal{L}_{\widetilde{\mathbf{D}}}(\widetilde{\mathbf{P}}^{(1)})$ and $\mathcal{L}_{\widetilde{\mathbf{D}}}(\mathbf{P}^{(1)})$ using~\eqref{eq:squared_distance_cost} and the permutation functions $\widetilde{\sigma}_1$ and $\sigma_1$. For $\mathcal{L}_{\widetilde{\mathbf{D}}}(\widetilde{\mathbf{P}}^{(1)})$, we have:
            \begin{equation} \label{eq:LC_rewrite1}
                \begin{aligned}
                    \mathcal{L}_{\widetilde{\mathbf{D}}}(\widetilde{\mathbf{P}}^{(1)})
                    &= \sum_{i,j =1}^n \widetilde{P}^{(1)}_{ij} \widetilde{D}_{ij} =
                    \sum_{i=1}^n \widetilde{D}_{i\widetilde{\sigma}_1(i)} \\
                    &=\sum_{i=1}^n \left(D_{i\widetilde{\sigma}_1(i)} + r_i + r_{\widetilde{\sigma}_1(i)} + \epsilon_{i\widetilde{\sigma}_1(i)}\right) \\
                    &= \sum_{i=1}^n \left(D_{i\widetilde{\sigma}_1(i)} + \epsilon_{i\widetilde{\sigma}_1(i)}\right) + 2 \sum_{i=1}^n r_i.
                    \end{aligned}
            \end{equation}
            Following the analogous derivation in~\eqref{eq:LC_rewrite1}, for $\mathcal{L}_{\widetilde{\mathbf{D}}}(\mathbf{P}^{(1)})$, we have:
            \begin{equation} \label{eq:LC_rewrite2}
                \mathcal{L}_{\widetilde{\mathbf{D}}}(\mathbf{P}^{(1)}) = \sum_{i=1}^n \left(D_{i\sigma_1(i)} + \epsilon_{i\sigma_1(i)}\right) + 2 \sum_{i=1}^n r_i.
            \end{equation}

            Substituting~\eqref{eq:LC_rewrite1} and~\eqref{eq:LC_rewrite2} into~\eqref{eq:LSAP_iter1}, and applying the triangle inequality yields, on $\Omega_t$,
            \begin{align}
                \sum_{i=1}^nD_{i\widetilde{\sigma}_1(i)}
                &\leq\sum_{i=1}^nD_{i\sigma_1(i)}
                +\sum_{i=1}^n|\epsilon_{i\sigma_1(i)}|
                +\sum_{i=1}^n|\epsilon_{i\widetilde\sigma_1(i)}|\nonumber\\
                &\leq cR_d^2n^{1-\alpha}+2nB_t.
                \label{eq:iter1_bound}
            \end{align}
            Here, the last inequality follows from Lemma~\ref{lem:Geometric LSAP Cost} and~\eqref{eq:uniform_epsilon_event}.

            \paragraph{The bound for $\sum_{i=1}^n D_{i\widetilde{\sigma}_2(i)}$} We bound $\sum_{i=1}^n D_{i\widetilde{\sigma}_2(i)}$ by comparing $\mathcal{L}_{\widetilde{\mathbf{D}}}(\widetilde{\mathbf{P}}^{(2)})$ and $\mathcal{L}_{\widetilde{\mathbf{D}}} (\mathbf{P}^{(2)})$.

            First, we observe that any permutation matrix $\mathbf{P}$ feasible for the LSAP optimization in~\eqref{eq:alg_LSAP_2} satisfies $\mathcal{L}_{\widetilde{\mathbf{D}}}(\mathbf{P}) = \mathcal{L}_{\widetilde{\mathbf{D}}'}(\mathbf{P})$, since every feasible matrix avoids the $+\infty$ entries of $\widetilde{\mathbf{D}}'$. Hence, by the feasibility of $\widetilde{\mathbf{P}}^{(2)}$, we have:
            \begin{equation}\label{eq:cost_mod_unmod1}
                \mathcal{L}_{\widetilde{\mathbf{D}}}(\widetilde{\mathbf{P}}^{(2)}) = \mathcal{L}_{\widetilde{\mathbf{D}}'}(\widetilde{\mathbf{P}}^{(2)}).
            \end{equation}
            Note that since $\widetilde{\mathbf{D}}'$ and $\mathbf{D}'$ share the same masking matrix $\widetilde{\mathbf{P}}^{(1)}$, any matrix feasible for the LSAP optimization with $\mathbf{D}'$ is also feasible for the optimization with $\widetilde{\mathbf{D}}'$. Thus, $\mathbf{P}^{(2)}$ is feasible for the LSAP optimization in~\eqref{eq:alg_LSAP_2}, leading to:
            \begin{equation}\label{eq:cost_mod_unmod2}
                \mathcal{L}_{\widetilde{\mathbf{D}}}(\mathbf{P}^{(2)}) = \mathcal{L}_{\widetilde{\mathbf{D}}'}(\mathbf{P}^{(2)}).
            \end{equation}

            Next, we note that by the definition of $\widetilde{\mathbf{P}}^{(2)}$ in~\eqref{eq:alg_LSAP_2} and the optimality of minimization:
            \begin{equation}\label{eq:optimal_second_round}
                \mathcal{L}_{\widetilde{\mathbf{D}}^{'}}(\widetilde{\mathbf{P}}^{(2)})
                \leq \mathcal{L}_{\widetilde{\mathbf{D}}^{'}}(\mathbf{P}^{(2)}).
            \end{equation}
            Combining~\eqref{eq:cost_mod_unmod1},~\eqref{eq:cost_mod_unmod2}, and~\eqref{eq:optimal_second_round}, we obtain:
            \begin{equation}
                \mathcal{L}_{\widetilde{\mathbf{D}}}(\widetilde{\mathbf{P}}^{(2)})
                \leq \mathcal{L}_{\widetilde{\mathbf{D}}}(\mathbf{P}^{(2)}).
            \end{equation}

            We now follow the analogous derivation steps in~\eqref{eq:LC_rewrite1},~\eqref{eq:LC_rewrite2}, and~\eqref{eq:iter1_bound} to obtain the following bound on $\Omega_t$:
            \begin{equation} \label{eq:iter2_bound}
                \sum_{i=1}^n D_{i\widetilde{\sigma}_2(i)}
                \leq cR_d^2n^{1-\alpha}+2nB_t.
            \end{equation}

            \paragraph{The bound for $\sum_{i=1}^n D_{\widetilde{\sigma}_1(i)\widetilde{\sigma}_2(i)}$}
            We derive a bound for $\sum_{i=1}^n D_{\widetilde{\sigma}_1(i)\widetilde{\sigma}_2(i)}$ using~\eqref{eq:iter1_bound} and~\eqref{eq:iter2_bound}.

            First, since pairwise Euclidean distances satisfy the triangle inequality, for any $i \in [n]$, we have:
            \begin{equation}\label{eq:Euclidean_triangle}
                \sqrt{D_{\widetilde{\sigma}_1(i)\widetilde{\sigma}_2(i)}} \leq \sqrt{D_{i\widetilde{\sigma}_1(i)}} + \sqrt{D_{i\widetilde{\sigma}_2(i)}}.
            \end{equation}

            By~\eqref{eq:Euclidean_triangle}, the arithmetic mean-geometric mean inequality, \eqref{eq:iter1_bound}, and~\eqref{eq:iter2_bound}, we obtain:
            \begin{align}
                \sum_{i=1}^n D_{\widetilde{\sigma}_1(i)\widetilde{\sigma}_2(i)}
                &\leq \sum_{i=1}^n \left( D_{i\widetilde{\sigma}_1(i)} + D_{i\widetilde{\sigma}_2(i)} + 2\sqrt{D_{i\widetilde{\sigma}_1(i)}D_{i\widetilde{\sigma}_2(i)}}\right) \nonumber\\
                &\leq \sum_{i=1}^n 2\left( D_{i\widetilde{\sigma}_1(i)} + D_{i\widetilde{\sigma}_2(i)}\right) \leq4cR_d^2n^{1-\alpha}+8nB_t. \label{eq:iter3_bound}
            \end{align}

            We now analyze the estimation error of noise magnitudes, namely $ \frac{1}{nR_d^2} \|\hat{\mathbf{r}} - \mathbf{r}\|_1$. Substituting~\eqref{eq:squared_distance_cost} into~\eqref{eq:noise_mag_estimate} and applying the triangle inequality gives
            \begin{align}
                \frac{1}{nR_d^2} \| \mathbf{\hat{r}} - \mathbf{r} \|_1
                &\leq \frac{1}{2nR_d^2} \sum_{i=1}^n  \Big( D_{i\widetilde{\sigma}_1(i)} + D_{i\widetilde{\sigma}_2(i)} + D_{\widetilde{\sigma}_1(i)\widetilde{\sigma}_2(i)} \nonumber\\
                &\qquad\qquad\qquad + \left| \epsilon_{i\widetilde{\sigma}_1(i)} \right| + \left| \epsilon_{i\widetilde{\sigma}_2(i)} \right| + \left| \epsilon_{\widetilde{\sigma}_1(i)\widetilde{\sigma}_2(i)} \right| \Big) \nonumber\\
                &\leq 3cn^{-\alpha}+\frac{15}{2}\frac{B_t}{R_d^2}
                \leq C^{\prime}\left(t\mathcal E(d)+n^{-\alpha}\right),
                \label{eq:proof_noise_bound}
            \end{align}
            for a constant $C^{\prime}>0$ independent of $n$, $d$, $R_d$, and $t$.

            Finally, we bound $\frac{1}{n(n-1)R_d^2} \sum_{i=1}^n \sum_{j\neq i}^n |\hat{D}_{ij} - D_{ij}|$ following a similar approach:
            \begin{align}
                \frac{1}{n(n-1)R_d^2} \sum_{i=1}^n \sum_{j\neq i}^n |\hat{D}_{ij} - D_{ij}| &\leq \frac{2}{nR_d^2} \|\mathbf{r} - \hat{\mathbf{r}}\|_1 + \frac{1}{n(n-1)R_d^2} \sum_{i = 1}^n \sum_{j\neq i}^n |\epsilon_{ij}| \nonumber\\
                &\leq2C^{\prime}\left(t\mathcal E(d)+n^{-\alpha}\right)+\frac{B_t}{R_d^2} \leq C^{\prime\prime}\left(t\mathcal E(d)+n^{-\alpha}\right),
                \label{eq:proof_distance_bound}
            \end{align}
            for a constant $C^{\prime\prime}>0$ independent of $n$, $d$, $R_d$, and $t$. Combining~\eqref{eq:uniform_epsilon_probability},~\eqref{eq:proof_noise_bound}, and~\eqref{eq:proof_distance_bound} completes the proof.
        \end{proof}

        \section{Proof of Corollary~\ref{cor:kernel_operator_norm}}\label{sec:kernel_operator_norm_proof}
        \begin{proof}
            Since $D_{ij}\geq0$ for every $i\neq j$, positive-part truncation satisfies
            \begin{equation}\label{eq:positive_part_contraction}
                |\hat D^+_{ij}-D_{ij}|\leq|\hat D_{ij}-D_{ij}|.
            \end{equation}
            Indeed, the claim is immediate when $\hat D_{ij}\geq0$, while if $\hat D_{ij}<0$, then
            $|\hat D^+_{ij}-D_{ij}|=D_{ij}\leq D_{ij}-\hat D_{ij}=|\hat D_{ij}-D_{ij}|$.

            For any $a,b\geq0$ and $h>0$, the map $u\mapsto e^{-u/h}$ takes values in $[0,1]$ and is $1/h$-Lipschitz. Hence,
            \begin{equation}\label{eq:kernel_entry_bound}
                |e^{-a/h}-e^{-b/h}|^2
                \leq\min\left\{1,\frac{|a-b|}{h}\right\}^2
                \leq\frac{|a-b|}{h}.
            \end{equation}
            Applying~\eqref{eq:kernel_entry_bound} entrywise, using~\eqref{eq:positive_part_contraction}, and then using the inequality between the operator and Frobenius norms yields
            \begin{align}
                \left\|\frac{\hat{\mathbf K}-\mathbf K}{n}\right\|_{\mathrm{op}}^2
                &\leq\frac{1}{n^2}\|\hat{\mathbf K}-\mathbf K\|_{\mathrm F}^2 \leq\frac{1}{n^2h}\sum_{i=1}^n\sum_{j\neq i}|\hat D^+_{ij}-D_{ij}| \nonumber\\
                &\leq\frac{n-1}{n}\frac{R_d^2}{h}
                \left[\frac{1}{n(n-1)R_d^2}\sum_{i=1}^n\sum_{j\neq i}|\hat D_{ij}-D_{ij}|\right].
                \label{eq:kernel_frobenius_derivation}
            \end{align}
            Combining~\eqref{eq:kernel_frobenius_derivation} with~\eqref{eq:estimation_L1_bound_distance} proves~\eqref{eq:kernel_operator_bound}.
        \end{proof}

\end{appendices}

\bibliographystyle{plain}
\bibliography{references}

@book{sagan,
  title={Space-filling curves},
  author={Sagan, Hans},
  year={2012},
  publisher={Springer Science \& Business Media}
}

@article{heowen,
  title={Extensible grids: uniform sampling on a space filling curve},
  author={He, Zhijian and Owen, Art B},
  journal={Journal of the Royal Statistical Society: Series B (Statistical Methodology)},
  volume={78},
  number={4},
  pages={917--931},
  year={2016},
  publisher={Wiley Online Library}
}

@book{zumbusch,
  title={Parallel multilevel methods: adaptive mesh refinement and loadbalancing},
  author={Zumbusch, Gerhard},
  year={2012},
  publisher={Springer Science \& Business Media}
}

@Article{MDS,
author={Torgerson, Warren S.},
title={Multidimensional scaling: I. Theory and method},
journal={Psychometrika},
year={1952},
month={Dec},
day={01},
volume={17},
number={4},
pages={401-419},
issn={1860-0980},
}

@article{self_tune,
  title={Self-tuning spectral clustering},
  author={Zelnik-Manor, Lihi and Perona, Pietro},
  journal={Advances in neural information processing systems},
  volume={17},
  year={2004}
}

@article{tsne,
  author  = {Laurens van der Maaten and Geoffrey Hinton},
  title   = {Visualizing Data using t-SNE},
  journal = {Journal of Machine Learning Research},
  year    = {2008},
  volume  = {9},
  number  = {86},
  pages   = {2579--2605},
}

@article{diffusion_map,
title = {Diffusion maps},
journal = {Applied and Computational Harmonic Analysis},
volume = {21},
number = {1},
pages = {5-30},
year = {2006},
issn = {1063-5203},
author = {Ronald R. Coifman and Stéphane Lafon},
}

@article{spectral_clustering,
  title={On spectral clustering: Analysis and an algorithm},
  author={Ng, Andrew and Jordan, Michael and Weiss, Yair},
  journal={Advances in neural information processing systems},
  volume={14},
  year={2001}
}

@misc{spectral_clustering3,
      title={SpectralNet: Spectral Clustering using Deep Neural Networks}, 
      author={Uri Shaham and Kelly Stanton and Henry Li and Boaz Nadler and Ronen Basri and Yuval Kluger},
      year={2018},
      archivePrefix={arXiv},
      primaryClass={stat.ML},
}

@article{spectral_clustering4,
  title={Spectral biclustering of microarray data: coclustering genes and conditions},
  author={Kluger, Yuval and Basri, Ronen and Chang, Joseph T and Gerstein, Mark},
  journal={Genome research},
  volume={13},
  number={4},
  pages={703--716},
  year={2003},
  publisher={Cold Spring Harbor Lab}
}

@article{community_detection,
title = {Community detection in graphs},
journal = {Physics Reports},
volume = {486},
number = {3},
pages = {75-174},
year = {2010},
issn = {0370-1573},
author = {Santo Fortunato},
}

@INPROCEEDINGS{image_denoising,
  author={Buades, A. and Coll, B. and Morel, J.-M.},
  booktitle={2005 IEEE Computer Society Conference on Computer Vision and Pattern Recognition (CVPR'05)}, 
  title={A non-local algorithm for image denoising}, 
  year={2005},
  volume={2},
  number={},
  pages={60-65 vol. 2}}

@ARTICLE{image_denoising2,
  author={Pang, Jiahao and Cheung, Gene},
  journal={IEEE Transactions on Image Processing}, 
  title={Graph Laplacian Regularization for Image Denoising: Analysis in the Continuous Domain}, 
  year={2017},
  volume={26},
  number={4},
  pages={1770-1785}}

@article{image_denoising4,
author = {Landa, Boris and Shkolnisky, Yoel},
title = {The Steerable Graph Laplacian and its Application to Filtering Image Datasets},
journal = {SIAM Journal on Imaging Sciences},
volume = {11},
number = {4},
pages = {2254-2304},
year = {2018},
}

@article{graph_learning,
title = {Diffusion wavelets},
journal = {Applied and Computational Harmonic Analysis},
volume = {21},
number = {1},
pages = {53-94},
year = {2006},
note = {Special Issue: Diffusion Maps and Wavelets},
issn = {1063-5203},
author = {Ronald R. Coifman and Mauro Maggioni},
}

@article{graph_learning2,
  title={Convolutional neural networks on graphs with fast localized spectral filtering},
  author={Defferrard, Micha{\"e}l and Bresson, Xavier and Vandergheynst, Pierre},
  journal={Advances in neural information processing systems},
  volume={29},
  year={2016}
}

@article{graph_learning3,
title = {Wavelets on graphs via spectral graph theory},
journal = {Applied and Computational Harmonic Analysis},
volume = {30},
number = {2},
pages = {129-150},
year = {2011},
issn = {1063-5203},
author = {David K. Hammond and Pierre Vandergheynst and Rémi Gribonval},
}

@ARTICLE{remote_sensing,
  author={Chi, Mingmin and Plaza, Antonio and Benediktsson, Jón Atli and Sun, Zhongyi and Shen, Jinsheng and Zhu, Yangyong},
  journal={Proceedings of the IEEE}, 
  title={Big Data for Remote Sensing: Challenges and Opportunities}, 
  year={2016},
  volume={104},
  number={11},
  pages={2207-2219}}

@article{photon_imaging,
author = {Salmon, Joseph and Harmany, Zachary and Deledalle, Charles-Alban and Willett, Rebecca},
title = {Poisson Noise Reduction with Non-local PCA},
year = {2014},
issue_date = {February  2014},
publisher = {Kluwer Academic Publishers},
address = {USA},
volume = {48},
number = {2},
issn = {0924-9907},
journal = {J. Math. Imaging Vis.},
month = feb,
pages = {279–294},
numpages = {16}
}

@article{network,
  title={Analysis of call centre arrival data using singular value decomposition},
  author={Shen, Haipeng and Huang, Jianhua Z},
  journal={Applied Stochastic Models in Business and Industry},
  volume={21},
  number={3},
  pages={251--263},
  year={2005},
  publisher={Wiley Online Library}
}

@article{astro,
    author = {Tamuz, O. and Mazeh, T. and Zucker, S.},
    title = {Correcting systematic effects in a large set of photometric light curves},
    journal = {Monthly Notices of the Royal Astronomical Society},
    volume = {356},
    number = {4},
    pages = {1466-1470},
    year = {2005},
    month = {02},
}

@article{human_atlas,
  title={The human cell atlas},
  author={Regev, Aviv and Teichmann, Sarah A and Lander, Eric S and Amit, Ido and Benoist, Christophe and Birney, Ewan and Bodenmiller, Bernd and Campbell, Peter and Carninci, Piero and Clatworthy, Menna and others},
  journal={elife},
  volume={6},
  pages={e27041},
  year={2017},
  publisher={eLife Sciences Publications, Ltd}
}

@article{homoskedastic_theory1,
author = {Noureddine El Karoui},
title = {{On information plus noise kernel random matrices}},
volume = {38},
journal = {The Annals of Statistics},
number = {5},
publisher = {Institute of Mathematical Statistics},
pages = {3191 -- 3216},
year = {2010}
}

@article{homoskedastic_theory,
author = {Noureddine El Karoui and Hau-Tieng Wu},
title = {{Graph connection Laplacian methods can be made robust to noise}},
volume = {44},
journal = {The Annals of Statistics},
number = {1},
publisher = {Institute of Mathematical Statistics},
pages = {346 -- 372},
year = {2016}
}

@article{boris1,
  title={Doubly stochastic normalization of the gaussian kernel is robust to heteroskedastic noise},
  author={Landa, Boris and Coifman, Ronald R and Kluger, Yuval},
  journal={SIAM journal on mathematics of data science},
  volume={3},
  number={1},
  pages={388--413},
  year={2021},
  publisher={SIAM}
}

@article{boris2,
author = {Landa, Boris and Cheng, Xiuyuan},
title = {Robust Inference of Manifold Density and Geometry by Doubly Stochastic Scaling},
journal = {SIAM Journal on Mathematics of Data Science},
volume = {5},
number = {3},
pages = {589-614},
year = {2023},
}

@article{boris3,
  title={Bi-stochastically normalized graph Laplacian: convergence to manifold Laplacian and robustness to outlier noise},
  author={Cheng, Xiuyuan and Landa, Boris},
  journal={Information and Inference: A Journal of the IMA},
  volume={13},
  number={4},
  pages={iaae026},
  year={2024},
  publisher={Oxford University Press}
}

@article{operator1,
  title={Manifold learning with bi-stochastic kernels},
  author={Marshall, Nicholas F and Coifman, Ronald R},
  journal={IMA Journal of Applied Mathematics},
  volume={84},
  number={3},
  pages={455--482},
  year={2019},
  publisher={Oxford University Press}
}

@InProceedings{loopy_belief_prop,
  title = 	 {Loopy Belief Propagation for Bipartite Maximum Weight b-Matching},
  author = 	 {Huang, Bert and Jebara, Tony},
  booktitle = 	 {Proceedings of the Eleventh International Conference on Artificial Intelligence and Statistics},
  pages = 	 {195--202},
  year = 	 {2007},
  editor = 	 {Meila, Marina and Shen, Xiaotong},
  volume = 	 {2},
  series = 	 {Proceedings of Machine Learning Research},
  address = 	 {San Juan, Puerto Rico},
  month = 	 {21--24 Mar},
  publisher =    {PMLR},
}

@article{metric_repair_review,
  title={The metric nearness problem},
  author={Brickell, Justin and Dhillon, Inderjit S and Sra, Suvrit and Tropp, Joel A},
  journal={SIAM Journal on Matrix Analysis and Applications},
  volume={30},
  number={1},
  pages={375--396},
  year={2008},
  publisher={SIAM}
}

@article{metric_repair_triangle,
  title={Triangle fixing algorithms for the metric nearness problem},
  author={Sra, Suvrit and Tropp, Joel and Dhillon, Inderjit},
  journal={Advances in Neural Information Processing Systems},
  volume={17},
  year={2004}
}

@article{metric_repair_norm,
  title={An efficient algorithm for the $\ell_p$ norm based metric nearness problem},
  author={Tang, Peipei and Jiang, Bo and Wang, Chengjing},
  journal={Mathematics of Computation},
  year={2024}
}

@article{wavelet_denoise,
  title={Review of wavelet denoising algorithms},
  author={Halidou, Aminou and Mohamadou, Youssoufa and Ari, Ado Adamou Abba and Zacko, Edinio Jocelyn Gbadoubissa},
  journal={Multimedia Tools and Applications},
  volume={82},
  number={27},
  pages={41539--41569},
  year={2023},
  publisher={Springer}
}

@article{pca,
  title={Principal component analysis: a review and recent developments},
  author={Jolliffe, Ian T and Cadima, Jorge},
  journal={Philosophical transactions of the royal society A: Mathematical, Physical and Engineering Sciences},
  volume={374},
  number={2065},
  pages={20150202},
  year={2016},
  publisher={the Royal Society publishing}
}

@article{autoencoder,
  title={Stacked denoising autoencoders: Learning useful representations in a deep network with a local denoising criterion.},
  author={Vincent, Pascal and Larochelle, Hugo and Lajoie, Isabelle and Bengio, Yoshua and Manzagol, Pierre-Antoine and Bottou, L{\'e}on},
  journal={Journal of machine learning research},
  volume={11},
  number={12},
  year={2010}
}

@inproceedings{b_matching,
  title={Graph construction and b-matching for semi-supervised learning},
  author={Jebara, Tony and Wang, Jun and Chang, Shih-Fu},
  booktitle={Proceedings of the 26th annual international conference on machine learning},
  pages={441--448},
  year={2009}
}

@article{snekhorn,
  title={Snekhorn: Dimension reduction with symmetric entropic affinities},
  author={Van Assel, Hugues and Vayer, Titouan and Flamary, R{\'e}mi and Courty, Nicolas},
  journal={Advances in Neural Information Processing Systems},
  volume={36},
  pages={44470--44487},
  year={2023}
}

@book{siam_assignment_book,
author = {Burkard, Rainer and Dell'Amico, Mauro and Martello, Silvano},
title = {Assignment Problems},
publisher = {Society for Industrial and Applied Mathematics},
year = {2012}
}

@article{hungarian,
author = {Kuhn, H. W.},
title = {The Hungarian method for the assignment problem},
journal = {Naval Research Logistics Quarterly},
volume = {2},
number = {1-2},
pages = {83-97},
year = {1955}
}

@book{HDP_book,
  title={High-dimensional probability: An introduction with applications in data science},
  author={Vershynin, Roman},
  volume={47},
  year={2018},
  publisher={Cambridge university press}
}

@ARTICLE{laplacian_eigenmap,
  author={Belkin, Mikhail and Niyogi, Partha},
  journal={Neural Computation}, 
  title={Laplacian Eigenmaps for Dimensionality Reduction and Data Representation}, 
  year={2003},
  volume={15},
  number={6},
  pages={1373-1396},}

@article{sinkhorn,
  title={Concerning nonnegative matrices and doubly stochastic matrices},
  author={Sinkhorn, Richard and Knopp, Paul},
  journal={Pacific Journal of Mathematics},
  volume={21},
  number={2},
  pages={343--348},
  year={1967},
  publisher={Mathematical Sciences Publishers}
}

@article{zheng2017pbmc,
  title={Massively parallel digital transcriptional profiling of single cells},
  author={Zheng, Grace X. Y. and Terry, Jessica M. and Belgrader, Phillip and Ryvkin, Paul and Bent, Zachary W. and Wilson, Ryan and Ziraldo, Solongo B. and Wheeler, Tobias D. and McDermott, Geoff P. and Zhu, Junjie and Gregory, Mark T. and Shuga, Joe and Montesclaros, Luz and Underwood, Jason G. and Masquelier, Donald A. and Nishimura, Stefanie Y. and Schnall-Levin, Michael and Wyatt, Paul W. and Hindson, Christopher M. and Bharadwaj, Rajiv and Wong, Alexander and Ness, Kevin D. and Beppu, Lan W. and Deeg, H. Joachim and McFarland, Christopher and Loeb, Keith R. and Valente, William J. and Ericson, Nolan G. and Stevens, Emily A. and Radich, Jerald P. and Mikkelsen, Tarjei S. and Hindson, Benjamin J. and Bielas, Jason H.},
  journal={Nature Communications},
  volume={8},
  number={1},
  pages={14049},
  year={2017},
  publisher={Nature Publishing Group}
}

@ARTICLE{scrnaseq_poisson_model,
  title    = "Separating measurement and expression models clarifies confusion
              in single-cell {RNA} sequencing analysis",
  author   = "Sarkar, Abhishek and Stephens, Matthew",
  journal  = "Nature Genetics",
  volume   =  53,
  number   =  6,
  pages    = "770--777",
  month    =  jun,
  year     =  2021
}

@book{graph_theory,
author = {Lovasz, L.},
title = {Matching Theory (North-Holland mathematics studies)},
year = {1986},
isbn = {0444879161},
publisher = {Elsevier Science Ltd.},
address = {GBR}
}

@article{best_sc_practice,
author = {Luecken, Malte D and Theis, Fabian J},
title = {Current best practices in single‐cell RNA‐seq analysis: a tutorial},
journal = {Molecular Systems Biology},
volume = {15},
number = {6},
pages = {e8746},
unuseddoi = {https://doi.org/10.15252/msb.20188746},
year = {2019}
}

@article{scRNAseq_intro,
  title={Highly parallel genome-wide expression profiling of individual cells using nanoliter droplets},
  author={Macosko, Evan Z and Basu, Anindita and Satija, Rahul and Nemesh, James and Shekhar, Karthik and Goldman, Melissa and Tirosh, Itay and Bialas, Allison R and Kamitaki, Nolan and Martersteck, Emily M and others},
  journal={Cell},
  volume={161},
  number={5},
  pages={1202--1214},
  year={2015},
  publisher={Elsevier}
}

@article{quantum_measurement,
  title={Intrinsic randomness as a measure of quantum coherence},
  author={Yuan, Xiao and Zhou, Hongyi and Cao, Zhu and Ma, Xiongfeng},
  journal={Physical Review A},
  volume={92},
  number={2},
  pages={022124},
  year={2015},
  publisher={APS}
}

@article{manifold2,
  title={Algorithms for manifold learning},
  author={Cayton, Lawrence and others},
  journal={Univ. of California at San Diego Tech. Rep},
  volume={12},
  number={1-17},
  pages={1},
  year={2005}
}

@article{manifold3,
  title={Testing the manifold hypothesis},
  author={Fefferman, Charles and Mitter, Sanjoy and Narayanan, Hariharan},
  journal={Journal of the American Mathematical Society},
  volume={29},
  number={4},
  pages={983--1049},
  year={2016}
}

@article{hubness,
  title={Hubs in space: Popular nearest neighbors in high-dimensional data},
  author={Radovanovic, Milos and Nanopoulos, Alexandros and Ivanovic, Mirjana},
  journal={Journal of Machine Learning Research},
  volume={11},
  number={sept},
  pages={2487--2531},
  year={2010}
}

@article{cost_scaling,
  title={Faster scaling algorithms for network problems},
  author={Gabow, Harold N and Tarjan, Robert E},
  journal={SIAM Journal on Computing},
  volume={18},
  number={5},
  pages={1013--1036},
  year={1989},
  publisher={SIAM}
}

@article{pseudoflow,
  title={An efficient cost scaling algorithm for the assignment problem},
  author={Goldberg, Andrew V and Kennedy, Robert},
  journal={Mathematical Programming},
  volume={71},
  number={2},
  pages={153--177},
  year={1995},
  publisher={Springer}
}

@article{auction,
  title={New scaling algorithms for the assignment and minimum cycle mean problem},
  author={Orlin, JB and Ahuja, RK},
  journal={Math. Program. v54},
  pages={41--56},
  year={1988}
}

@article{ho2019raman,
  title={Rapid identification of pathogenic bacteria using Raman spectroscopy and deep learning},
  author={Ho, Chi-Sing and Jean, Neal and Hogan, Catherine A and Blackmon, Lena and Jeffrey, Stefanie S and Holodniy, Mark and Banaei, Niaz and Saleh, Amr AE and Ermon, Stefano and Dionne, Jennifer},
  journal={Nature communications},
  volume={10},
  number={1},
  pages={4927},
  year={2019},
  publisher={Nature Publishing Group UK, London}
}

@article{louvain,
  title={Fast unfolding of communities in large networks},
  author={Blondel, Vincent D and Guillaume, Jean-Loup and Lambiotte, Renaud and Lefebvre, Etienne},
  journal={Journal of statistical mechanics: theory and experiment},
  volume={2008},
  number={10},
  pages={P10008},
  year={2008}
}

@inproceedings{sklearn_api,
  author    = {Lars Buitinck and Gilles Louppe and Mathieu Blondel and
                Fabian Pedregosa and Andreas Mueller and Olivier Grisel and
                Vlad Niculae and Peter Prettenhofer and Alexandre Gramfort
                and Jaques Grobler and Robert Layton and Jake VanderPlas and
                Arnaud Joly and Brian Holt and Ga{\"{e}}l Varoquaux},
  title     = {{API} design for machine learning software: experiences from the scikit-learn
                project},
  booktitle = {ECML PKDD Workshop: Languages for Data Mining and Machine Learning},
  year      = {2013},
  pages = {108--122},
}
\end{document}